\documentclass[twoside,11pt]{article}

\usepackage{blindtext}

\usepackage{amsmath,amsfonts,mathtools,amsthm}
\usepackage[preprint]{jmlr2e}
\usepackage{moresize}

\usepackage[most]{tcolorbox}
\usepackage{graphicx}
\usepackage{hyperref}
\usepackage{cleveref}

\usepackage{algorithm}
\usepackage{algpseudocode}
\usepackage{todonotes}
\usepackage{lscape}
\usepackage{booktabs}
\usepackage{geometry}
\usepackage{pdflscape}
\usepackage{makecell}
\usepackage{utfsym}

\newcommand{\newcheckmark}{\usym{2713}}
\newcommand{\newcrossmark}{\usym{2717}}
\usepackage{caption}
\usepackage{subcaption}
\newcommand{\de}{\,\text{d}}
\newcommand{\R}{\mathbb{R}}

\newcommand{\Ex}{\mathbb{E}}

\newcommand{\edits}[1]{\textcolor{black}{#1}}

\newcommand{\MK}[1]{\textcolor{orange}{MK: #1}}

\usepackage{lastpage}
\jmlrheading{xx}{2023}{1-\pageref{LastPage}}{xx/xx; Revised xx/xx}{xx/xx}{xx-xxxx}{Zhang and Katsoulakis}

\ShortHeadings{Mean-field games laboratory for generative modeling}{Zhang and Katsoulakis}
\firstpageno{1}

\begin{document}

\title{A Mean-Field Games Laboratory for Generative Modeling}

\author{\name Benjamin J. Zhang \email bjzhang@umass.edu \\
       \addr Department of Mathematics and Statistics\\
       University of Massachusetts Amherst\\
       Amherst, MA 01003-9305, USA
       \AND
       \name Markos A. Katsoulakis \email markos@umass.edu \\
       \addr Department of Mathematics and Statistics\\
       University of Massachusetts Amherst\\
       Amherst, MA 01003-9305, USA}

\editor{My editor}

\maketitle

\begin{abstract}%   <- trailing '%' for backward compatibility of .sty file
We demonstrate the versatility of mean-field games (MFGs) as a mathematical framework for explaining, enhancing, and designing generative models. In generative flows, a Lagrangian formulation is used where each particle (generated sample) aims to minimize a loss function over its simulated path. The loss, however, is dependent on the paths of other particles, which leads to a competition among the population of particles. The asymptotic behavior of this competition yields a mean-field game. We establish connections between MFGs and major classes of generative flows and diffusions including continuous-time normalizing flows, score-based generative models (SGM), and Wasserstein gradient flows. Furthermore, we study the mathematical properties of each generative model by studying their associated MFG's optimality condition, which is a set of coupled forward-backward nonlinear partial differential equations. The mathematical structure described by the MFG optimality conditions identifies the inductive biases of generative flows. We investigate the well-posedness and structure of normalizing flows, unravel the mathematical structure of SGMs, and derive a MFG formulation of Wasserstein gradient flows. From an algorithmic perspective, the optimality conditions yields Hamilton-Jacobi-Bellman (HJB) regularizers for enhanced training of generative models. In particular, we propose and demonstrate an HJB-regularized SGM with improved performance over standard SGMs. We present this framework as an MFG laboratory which serves as a platform for revealing new avenues of experimentation and invention of generative models. %This laboratory will give rise to a multitude of well-posed generative modeling formulations and will provide a consistent theoretical framework upon which numerical and algorithmic tools may be developed.
\end{abstract}
%We derive these three classes of generative models through different choices of particle dynamics and cost functions. 
% In generative flows a Lagrangian formulation is deployed  so that the  modeling task becomes an optimization problem where each individual particle (generated sample) evolves  while minimizing an objective function over its path. This optimization problem is (in general) an MFG since the objective function depends on other generated particles.

\begin{keywords}
  Generative modeling, mean-field games, Hamilton-Jacobi-Bellman equation, normalizing flows, score-based generative models, Wasserstein gradient flow, inductive bias
\end{keywords}

\tableofcontents

\section{Introduction}

% \todo{Improve Intro/Conclusions using response/MURI}

Generative models are powerful unsupervised learning tools that approximate a target probability distribution via samples. These constructed models are often used to generate new but approximate samples that statistically reproduce and augment the original dataset. Generative models have recently found extraordinary success in applications such as image generation, image processing, audio generation, natural language processing, molecular modeling, and more. Their success has led to a surge of research activity related to formulating, parametrizing, and training new deep generative models. Three of the most successful classes of generative models of the last decade are generative adversarial nets (GANs) \cite{goodfellow2020generative}, denoising diffusion probability models \cite{ho2020denoising}, including score-based generative models (SGM) \cite{song2021score}, and normalizing flows (NFs) \cite{rezende2015variational,chen2018neural,grathwohl2018ffjord}. These three classes of models have spawned numerous competing variants in the pursuit of inexpensive, high-quality generative models. 

 We focus on two groups of generative modeling techniques: continuous-time generative models, and implicit generative models. Continuous normalizing flows \cite{chen2018neural} and score-based generative models \cite{song2021score} are examples of \emph{continuous-time} generative models, in which the model is expressed as an ordinary or stochastic differential equation. The velocity field is learned via data, such that the associated ODE or SDE transports samples of a normal distribution to samples of the target distribution. In contrast, \emph{implicit} generative models, based on Wasserstein gradient flows, are less well-known particle system--based approaches for generative modeling and sampling \cite{arbel2019maximum,wang2021particle,gu2022lipschitz}. Here, rather than learning an explicit map that generates samples from the target distribution, new samples are instead generated through the Wasserstein gradient flow of an empirical distribution toward the target distribution. The Wasserstein gradient flow \cite{santambrogio2017euclidean} has also been used to study the training of GANs \cite{arbel2019maximum,mroueh2019sobolev}.  There is a pervasive sense that flow and diffusion-based generative models have some common fundamental structure and interconnections, and there have been several attempts to connect these two classes of models together \cite{albergo2023stochastic,zhang2021diffusion}. We will show that mean-field games provide a natural framework for describing flow and diffusion-based generative models.

 Mean-field game (MFG) theory is a mathematical field that studies the dynamics and decision-making of individual agents in a large population \cite{lasry2007mean,gueant2012mean}. This emerging field has applications in many disciplines including control theory, economics, finance, epidemiology, and more. In this paper, we present a unifying mathematical framework based on mean-field games for explaining, enhancing, and inventing generative models. 
While the foundational formulations of the generative models we study in this paper have different origins, we will demonstrate that these generative models can be derived through a unifying mean-field games formalism. 

Our primary observation in this paper is that mean-field games naturally encode the characteristics of generative modeling with continuous-time models. A mean-field game involves an objective function that models the action, interaction, and terminal costs of an individual agent in relation to a population of other agents. The solution to the MFG describes the action each agent should take throughout space and time, while describing the evolution of the population distribution over time. The optimality conditions of MFGs are a pair of coupled PDEs: a Fokker-Planck (or continuity) equation, which describes the evolution of the density, and a Hamilton-Jacobi-Bellman (HJB) equation, which describes the optimal velocity field.

Meanwhile, continuous-time generative models have a common set of ingredients. Given samples from some target distribution, an optimal velocity field that evolves a simple reference distribution to the target distribution in finite time is learned. An optimization problem over the space of probability distributions is solved to learn the optimal velocity field. The evolution of the density is determined by a transport PDE (the continuity equation or Fokker-Planck equation), which is intractable to solve beyond low dimensional settings. Instead, the Lagrangian formulation is considered so that the generative modeling task becomes an optimization problem where each individual particle (generated sample) evolves according to its location while minimizing some objective function over the path it takes. This optimization problem appears to be an optimal control problem; it is, however, a mean-field game since the objective function depends on other generated samples (particles). In the language of mean-field games, the Fokker-Planck describes the family of generative models, while the solution of the HJB characterizes the optimal generative model.    

We will derive the major classes of generative models with the MFG formalism and culminate in a sort of laboratory (see Table~\ref{tab:big2}) in which researchers can experiment with existing and new generative modeling approaches. We also create several new examples of 
generative models as demonstration of the potential uses of the MFG laboratory.

\paragraph{Our contributions:} We present a unifying mean-field games framework for explaining, enhancing, and inventing generative models. By \emph{explaining} how generative models can be derived from mean-field games, we bring attention to the richness of the MFG for designing and analyzing generative models. 

\begin{itemize}
    \item \textbf{Explaining generative models}: \edits{Formulating generative flows as  MFGs reveals a \emph{common mathematical structure} through their optimality conditions, i.e.,  a backwards HJB and a forward Fokker-Planck equation (FPE). This backward--forward structure was already observed as a key feature in denoising diffusion models. Our MFG formulation shows that it is shared by \emph{all} flow and diffusion--based generative models, and delineates the \emph{generative} and \emph{discriminative} features required to describe the flows. That is, the FPE describes the generative model while the HJB identifies the optimal velocity field. The mathematical structure described by the mean-field game identifies the inductive bias of generative flows. In particular, we formulate normalizing flow (Theorem~\ref{thm:normalizingflows}), score-based generative models (Theorem~\ref{thm:scoremfg}), and Wasserstein gradient flows (Theorem~\ref{thm:MFG:Wasserstein_grad_flow}) in terms of MFGs. Another key result is that the MFG optimality conditions characterize the well-posedness of optimization problems used to train generative models. For example, in Section~\ref{sec:wellposed}, we study the ill-posedness of the canonical normalizing flows framework using the theory of HJB PDEs, and ways to make the method well-posed.}
    
    % We formulate normalizing flows (Theorem~\ref{thm:normalizingflows}), score-based generative models (Theorem~\ref{thm:scoremfg}), and Wasserstein gradient flows (Theorem~\ref{thm:MFG:Wasserstein_grad_flow}) in terms of mean-field games. The theory of MFGs describe the mathematical structure of generative models as the solution of a Hamilton-Jacobi-Bellman PDE. For example, in Section~\ref{sec:wellposed}, we study the ill-posedness of the canonical normalizing flows framework using the theory of HJB PDEs, and ways to make the method well-posed. Moreover, our paper makes further connections of MFGs to the Schr\"odinger bridge problem (Section~\ref{sec:Schrodinger}). We emphasize that this resulting PDE perspective is necessary in the analysis of continuous-time generative models. 

    \item \textbf{Enhancing generative models}: \edits{MFGs inform structure of generative flows that may be used to accelerate their training. The optimality conditions explicitly relate the optimal velocities to the gradient of the solution of the HJB.  Therefore, even without solving the  HJB explicitly, we know that we can \emph{restrict} the generative vector fields to just gradients of a scalar function. Moreover, in Section~\ref{sec:enhancing} we construct regularization functionals that can be derived from the optimality conditions which may further accelerate learning of the optimal generative model.}

    % Studying the optimality conditions of an MFG associated with a generative model informs structure of the optimal model. In addition, in Section~\ref{sec:enhancing} we construct regularization functions that can be derived from the optimality conditions which may further accelerate learning of the optimal generative model. 

    \item \textbf{Inventing generative models}: This unifying approach establishes a platform  for systematically building new generative models. In Section \ref{sec:designing} we give examples of new generative modeling formulations using MFGs. 
\end{itemize}
     Our work culminates to Table~\ref{tab:big2} in Section~\ref{sec:table}, which not only classifies major generative models, but also provides a laboratory for the experimentation, investigation, and invention of new generative models. We emphasize, however, that not all models developed using the laboratory will necessarily have good practical performance. This framework allows one to formulate generative models with well-defined optimization problems and inform their associated structure, however one must still proceed with numerical and algorithmic investigation to make their implementation practical.

\paragraph{Related work:} The connections between mean-field games and machine learning have been previously investigated. The use of generative modeling tools, such as normalizing flows, to solve mean-field games has been well-explored \cite{ruthotto2020machine,huang2022bridging}. Moreover, connections between reinforcement learning and mean-field games have been explored as well \cite{guo2022entropy}. Studying score-based generative models with stochastic optimal control tools is discussed in \cite{berner2022optimal}. The use of partial differential equations to enhance the training of score-based generative models was applied in \cite{lai2023fp}. Recently, \cite{albergo2022building,albergo2023stochastic} have also unified normalizing flows with diffusion models and the Schr\"odinger bridge through \emph{stochastic interpolants}. To our knowledge, however, there has not been a comprehensive unifying approach of using the MFG formalism to describe several classes of generative models. %\todo{Sony and other references from the reviews}

\subsection{A preview of results} 
We preview the main results of this paper. In Section~\ref{sec:MFG}, we provide a more thorough exposition of mean-field games, but we provide a simplified presentation here to discuss the main ideas of our paper. 

Let $X$ be a $\R^d$--valued random variable with unknown data distribution $\pi$ and density $\pi(x)$. Denote $\rho_0$ be a standard normal distribution with density $\rho_0(x)$. We suppose we have only samples $\{x_i\}_{i = 1}^N $ from $\pi$. The generative modeling problem is the task of producing more samples from $\pi(x)$. 

The potential formulation of mean-field games (MFG) is an infinite dimensional optimization problem over velocity fields $v \in \mathcal{C}^2$, $v:\mathbb{R}^d \times \R \to \R^d$, and densities $\rho(\cdot,t) \in \mathcal{P}(\R^d)$ with finite second moments for some finite time interval $t\in [0,T]$. Broadly speaking, the objective function is decomposed into a sum of three types of cost functions, $L$ to model action costs, $\mathcal{I}$ to model interaction costs, and $\mathcal{M}$ to model terminal costs. The continuity equation constraint models the dynamics of the density of the agents under a particular velocity field. There is a correspondence between the continuity equation and the particle dynamics, which depend on whether the dynamics are deterministic or stochastic. The MFG is defined as
\begin{align}
    &\inf_{v,\rho} \left\{ \mathcal{M}(\rho(\cdot,T))+ \int_0^T \mathcal{I}(\rho(\cdot,t)) \de t + \int_0^T \int_{\R^d} L(x,v(x,t)) \rho(x,t) \de x \de t \right\} \\
    \text{s.t.}&\,\, \frac{\partial \rho}{\partial t} + \nabla \cdot(v\rho) = \frac{\sigma^2}{2} \Delta \rho,\, \rho(x,0) = \rho_0(x)\, . \nonumber % &\de x(t) = v(x(t),t) \de t + \sigma\de W(t). \nonumber
\end{align} 
The optimality conditions for the MFG are given by a coupled system of a continuity (Fokker-Planck) equation  for $\rho$ which corresponds to the optimal generative model, and a (backward) Hamilton-Jacobi-Bellman (HJB) equation which yields the optimal velocity field $v^*$.

We will derive canonical formulations of generative models using the potential formulation of mean-field games. Depending on how cost functions $L,\mathcal{I}$, and $\mathcal{M}$, are chosen, the solutions of the corresponding MFGs will be different classes of generative models. Aspects of the derivations and connections have been noted in the literature for particular cases, such as the optimal transport regularized normalizing flow \cite{onken2021ot,ruthotto2020machine,huang2022bridging}, and the stochastic optimal control perspective of score-based generative modeling \cite{berner2022optimal}. Here, we provide a comprehensive unifying approach, using the full MFG formalism to describe several classes of fundamental generative models. In Tables \ref{tab:preview1} and \ref{tab:preview2}, we overview the main contribution of this paper for three well-studied classes of generative models. Eventually, we will build to Table~\ref{tab:big2}, by demonstrating  how numerous generative modeling variants can be described via the MFG framework. In Figure \ref{fig:vertical_scheme}, we show a schematic illustrating how the mean-field games perspective relates and provides insight to the current understanding of generative models.

 \begin{table}
    \def\arraystretch{1.5}
    \begin{tabular}{|l||l|l|l|}
    \hline
   Generative model &   $\mathcal{M}(\rho)$ & $\mathcal{I}(\rho)$  & $L(x,v)$     \\ \hline\hline
     Normalizing flow & $ \Ex_{\rho}\left[\log \rho - \log \rho_{ref} \right]$ & 0 & 0  \\ \hline
     Score-based & $-\Ex_{\rho }\log \pi $& 0& $\frac{1}{2}|v|^2 - \nabla \cdot f$  \\ \hline
     Wasserstein Gradient Flows ($\epsilon \to 0$)  &  $\mathcal{F}(\rho) e^{-T/\epsilon}$ & $\frac{e^{-t/\epsilon}}{\epsilon}\mathcal{F}(\rho)$ & $\frac{1}{2}|v|^2 e^{-t/\epsilon} $  \\ \hline
    \end{tabular}
    \centering
    \caption{Generative models and corresponding mean-field game cost functions. See a more complete version in Table~\ref{tab:big2}.}
    \label{tab:preview1}
    \end{table}

    \begin{table}
        \def\arraystretch{1.5}
        \begin{tabular}{|l||l|l|l|}
            \hline
           Generative model &  Dynamics & $H(x,p)$  & Optimal $v^*$    \\ \hline\hline
             Normalizing flow & $ \de x = v \de t$& $\sup_{v\in K \subset \R^d} -p^\top v $& $-\nabla_p H(x,\nabla U)$  \\ \hline
             Score-based & $\de x = \left[ f+ \sigma v\right] \de t + \sigma \de W_t $& $-f^\top p + \frac{|\sigma p|^2}{2} + \nabla \cdot f$& $\sigma \nabla \log \eta_t$  \\ \hline
             Wasserstein GF ($\epsilon \to 0$) & $\de x = v \de t$& $\frac{1}{2}e^{t/\epsilon} |p|^2$ & $- \frac{\delta \mathcal{F}}{\delta \rho}$  \\ \hline
            \end{tabular}
            \centering
            \caption{Generative models, their corresponding dynamics, Hamiltonians, and optimal velocity fields.}
            \label{tab:preview2}
    \end{table}

\begin{figure}[H]
    \centering
    \includegraphics[width = 0.66\textwidth,trim = 20 220 20 70,clip]{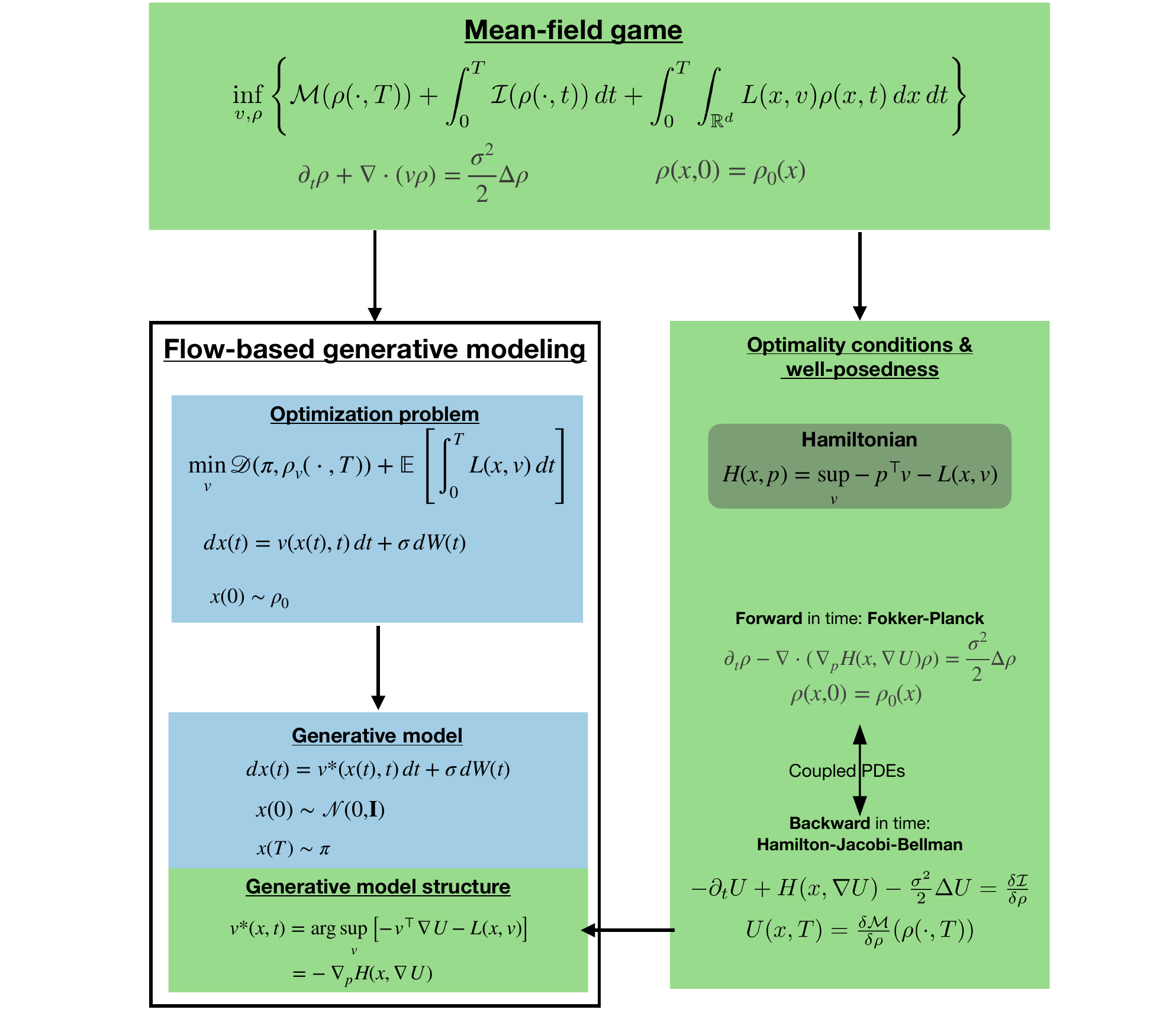}
    \caption{Flow chart describing how mean-field games are related to and provide new insights for flow-based generative modeling. Blue boxes denote the current understanding of flow-based generative models. Green boxes are the new perspective provided by mean-field games.}
    \label{fig:vertical_scheme}
\end{figure}

\section{Background on mean-field games}
\label{sec:MFG}

We review formulations of mean-field games and their optimality conditions. This section is based on \cite{lasry2007mean,bensoussan2013mean}. Mean-field games model the dynamics of individual rational agents whose actions depend on the dynamics of a large population of other agents. When the trajectories and actions take values in $\R^d$, the MFG is a differential game, whose solution determines the optimal action of each individual agent. The optimality conditions are a system of partial differential equations, the Fokker-Planck (continuity) and Hamilton-Jacobi-Bellman (HJB) equations, which characterize the solution of the MFG. The solution of the optimality conditions is often referred to as a \emph{Nash equilibrium}.

Denote $\mathcal{P}(\R^d)$ to be the space of probability distributions on $\R^d$ with finite second moment, and suppose the evolution of an agent $x(s)\in \R^d$ over the time interval $s\in [0,T]$ is given by the stochastic differential equation
\begin{align}
    \begin{dcases}
        \de x(s) &= v(x(s),s) \de s+ \sigma(x(s),s) \de W(s) \\
        x(0) &= x_0 \label{eq:agentdynamics}
    \end{dcases}
\end{align}
where $v:\R^d\times \R \to \R^d$ is a vector field that models the action of the agent at time $s$ and position $x(s)$, $\sigma(x(s),s)$ is the diffusion coefficient, and $W(s)$ is a standard Brownian motion on $\R^d$. The velocity field (action) at $(x,t)$ is determined by minimizing the objective function $J:\mathcal{C}^2(\R^d) \times \mathcal{P}(\R^d) \to \R$,
\begin{align}
    J_{x,t}(v,\rho) = \Ex \left[ \int_t^T L(x(s),v(x(s),s)) + I(x(s),\rho(x(s),s)) \de s + M(x(T),\rho(x(T),T)) \right],\label{eq:MFGobjective}
\end{align} %\todo{Take out explicit time dependence for ease of notation?} 
% \todo{check again where the $M$ goes? (inside or outside the expectation)}
where $L:\R^d\times \R^d \to \R$ is the running cost incurred by the agent by performing action $v(x(s),s)$, $I: \R^d\times \mathcal{P}(\R^d)\to \R$ is the running cost due to interaction with the density of all the agents $\rho(\cdot,s) \in \mathcal{P}(\R^d)$, and $M:\R^d\times \mathcal{P}(\R^d)\to \R$ is the terminal cost which depends on the final position of the agent and the population density. The expectation is taken over the path-space measure of \eqref{eq:agentdynamics} with initial condition $(x,t)$ and over time interval $s \in [t,T]$. Note that the running costs may explicitly depend on time, but for simplicity of notation, we omit the time argument. We consider a time-dependent example in Section \ref{sec:wassersteingradient}. Let $\mathbf{D}(x,t) = \frac{1}{2}\sigma(x,t)\sigma(x,t)^\top$. Given a velocity field $v=v(x, t)$, the evolution of the population density is given by the continuity (Fokker-Planck) equation
\begin{align}
    \frac{\partial \rho}{\partial s} + \nabla \cdot (v\rho) = \nabla^2: \mathbf{D}\rho \label{eq:continuityeqn}\\ 
    \rho(x,0) = \rho_0(x) \nonumber
\end{align}
Here, $\mathbf{A}:\mathbf{B} = \sum_{ij} \mathbf{A}_{ij}\mathbf{B}_{ij}$, and $(\nabla^2)_{ij} = \frac{\partial^2}{\partial x_i\partial x_j}$. This formulation also holds when the agent dynamics in \eqref{eq:agentdynamics} are deterministic, i.e., when $\sigma \coloneqq 0$. 

\subsection{Optimality conditions}
The solution of a mean-field game is a pair $({v^*}(\cdot,t),{\rho^*}(\cdot,t))$, referred to as a Nash equilibrium, such that for all $v$, 
\begin{align}
    J_{x,t}({v^*},{\rho^*}) \le J_{x,t}(v,{\rho^*}).
\end{align}
The optimality conditions characterizing this pair is found by considering the value function 
\begin{align}
    U(x,t) = \inf_{v,\rho} J_{x,t}(v,\rho),
\end{align}
which solves a Hamilton-Jacobi-Bellman (HJB) equation that is coupled to the continuity equation. The optimality conditions are
\begin{align}
    \begin{dcases}
    -\frac{\partial U}{\partial t} + H(x,\nabla U) - \mathbf{D}: \nabla^2 U =I(x,\rho(x,t)) \\ 
    \frac{\partial \rho}{\partial t} -\nabla \cdot(\rho\nabla_p H(x,\nabla U))  = \nabla^2 : \mathbf{D}\rho \\ 
    U(x,T) = M(x,\rho({x},T))  \\ 
    \rho(x,0) = \rho_0(x).  
    \end{dcases}
    \label{eq:HJBopt}
\end{align}
Here, the Hamiltonian is defined as
\begin{align}
    H(x,p) = \sup_{v} \left[- p^\top v -L(x,v)\right],
\end{align}
and the optimal control, ${v^*}(x,t)$, is given by the maximizer, and in this case is equal to  
\begin{align}
    {v^*}(x,t) = -\nabla_p H(x,\nabla U(t,x)).\end{align} Since this relation holds at optimality, it is common to replace $v$ with ${v^*}$ in \eqref{eq:continuityeqn} to explicitly couple the continuity and HJB equations. Note that the HJB equation is solved \emph{backwards} in time while the continuity equation is solved \emph{forwards} time as the latter has a terminal condition, and the former has an initial condition. Moreover, it is important to note that this is not a standard optimal control problem as the HJB equation cannot be solved by itself since it is coupled with the continuity equation. For details about the derivation of the optimality conditions, see \cite{lasry2007mean} and the references therein.

\subsection{Potential formulation of MFGs} 
\label{sec:potentialmfg}
Solving the coupled system of PDEs in \eqref{eq:HJBopt} via traditional numerical methods for PDEs is intractable beyond low dimensional settings. Meanwhile, solving the MFG directly as presented in \eqref{eq:agentdynamics} and \eqref{eq:MFGobjective} is still quite computationally taxing as it requires solving an infinite dimensional optimization problem for each point in time and space. In the celebrated paper \cite{lasry2007mean}, an alternative \emph{potential formulation} is devised for a certain class of MFGs so that the optimal controls and densities can be determined simultaneously for all $(x,t) \in \R^d\times \R$. 

Suppose there exist functionals $\mathcal{M}: \mathcal{P}(\R^d)  \to \R$ and $\mathcal{I}:  \mathcal{P}(\R^d)  \to \R$ such that 
\begin{align}
    I(x,\rho) = \frac{\delta \mathcal{I}(\rho)}{\delta \rho}(x),\,\,\,\, M(x,\rho) = \frac{\delta \mathcal{M}(\rho)}{\delta \rho}(x).
\end{align}
For a test function $\phi \in L^2(\R^d)$ and some density function $\rho\in \mathcal{P}(\R^d)$, recall that the variational derivative of functional $\mathcal{F}$ is given by 
\begin{align*}
\lim_{h \to 0} \frac{\mathcal{F}(\rho + h\phi) - \mathcal{F}(\rho)}{h} = \int_{\R^d} \frac{\delta\mathcal{F} }{\delta \rho} \phi(x) \de x. 
\end{align*} 
The potential formulation of the MFG \eqref{eq:MFGobjective} is given by 
\begin{align}
   \inf_{v,\rho}\mathcal{J}(v,\rho)= &\inf_{v,\rho} \left\{ \mathcal{M}(\rho(\cdot,T))+ \int_0^T \mathcal{I}(\rho(\cdot,t)) \de t + \int_0^T \int_{\R^d} L(x,v(x,t)) \rho(x,t) \de x \de t \right\} \label{eq:potentialform}\\
    \text{s.t.}&\,\, \frac{\partial \rho}{\partial t} + \nabla \cdot(v\rho) = \nabla^2 : \mathbf{D}\rho,\, \rho(x,0) = \rho_0(x). \nonumber
\end{align} 
The corresponding optimality conditions \eqref{eq:HJBopt} are,
\begin{align}
    \begin{dcases}
    -\frac{\partial U}{\partial t} + H(x,\nabla U) - \mathbf{D}: \nabla^2 U = I(x,\rho) = \frac{\delta \mathcal{I}(\rho)}{\delta \rho}(x) \\ 
    \frac{\partial \rho}{\partial t} -\nabla \cdot(\nabla_p H(x,\nabla U)\rho)  = \nabla^2 : \mathbf{D}\rho \\ 
    U(x,T) = M(x,\rho) = \frac{\delta \mathcal{M}(\rho(\cdot,T))}{\delta \rho}(x) \\ 
    \rho(x,0) = \rho_0(x).  
    \end{dcases}
    \label{eq:HJBopt:pot}
\end{align}

The optimization problem given in this formulation is often \emph{implicitly} considered in many machine learning tasks, such as density estimation and generative modeling. We will demonstrate the connections between the potential MFG and generative models clearly.  In the following sections, we will use the potential MFG formulation \eqref{eq:potentialform} to establish for a unified perspective of normalizing flows, score-based generative modeling, Wasserstein gradient flows, and their variants. It will also provide a framework for extending and enhancing these classes of generative models.

\section{Continuous normalizing flows as solutions to MFGs}\label{sec:CNF}

A normalizing flow (NF) is an \emph{invertible} transport map $f:\R^d\to\R^d$ that a transforms \edits{target distribution $\pi$} to a standard normal $\rho_{ref}$ \cite{rezende2015variational,tabak2010density,tabak2013family}. By learning such a map $f$ through samples of $\pi$, one can then produce more realizations from the target distribution by \edits{inverting} the map on sample points of the normal distribution. We say \edits{reference distribution $\rho_{ref}$ is the \emph{pushforward} distribution of $\pi$ under map $f$}, and their densities $\pi(x)$, and $\rho_{ref}(x)$ are related by a change of variables formula
\begin{align}
    \log \pi(x) = \log \rho_{ref}(f(x)) + \log \det \mathbf{J}_f (x)
\end{align}
where $\mathbf{J}_f(x)$ denotes the Jacobian of $f$. For generative modeling, the map is typically learned by maximizing the log-likelihood of the pullback distribution defined by the approximate map $f_\theta$ with respect to the data distribution over a parametrized space of functions, i.e., by solving the variational problem
\begin{align}
   \max_\theta \Ex_\pi \left[\log\pi_\theta(x)  \right]  =   \max_{\theta} \Ex_{\pi}\left[ \log \rho_{ref}(f_\theta(x)) + \log \det \mathbf{J}_{f_\theta}(x) \right]
\end{align}
for parameters $\theta \in \Theta$. 

The map $f_\theta$ is often parameterized as a composition of simple invertible functions, which makes the map amenable to approximation via neural networks. Each layer of the neural network represents one of the composed functions. In practice, the simple invertible functions are chosen depending on the particular application and so that the determinant Jacobian can be computed quickly. There is a rich body of literature outlining different function choices. See \cite{papamakarios2021normalizing,kobyzev2020normalizing} and the references therein for a review of parametrization choices. %Moreover, the inverse map is typically parametrized and learned from data, and the forward map $f_\theta$ can be evaluated by inverting $f^{-1}_\theta$. 

\subsection{Continuous normalizing flows}

A \emph{continuous normalizing flow} (CNF) is a particular type of normalizing flow that can be conceptualized as a standard NF as the number of layers in the neural network tends to infinity \cite{chen2018neural}. The transport map is described as the solution of a differential equation, and the velocity field of the differential equation is learned from data rather than the entire flow map. \edits{The goal then is to learn a $v_\theta(x,t)$ such that we have an ODE model $\frac{\de x}{\de t} = v_\theta(x(t),t)$, $x(0) \sim \pi$, and $x(T) \sim \rho(\cdot,T) = \rho_{ref}$, where $\rho(x,t)$ satisfies the continuity equation \eqref{eq:continuityeqn}. Running the normalizing flow reverse in time then produces a generative model that flows samples from $\rho_{ref}$ to new samples from $\pi$.}

A computable change of variable formula also exists for continuous normalizing flows. Assuming that $v$ is uniformly Lipschitz in $x$ and continuous in $t$. The change in log probability along a trajectory $x(t)$ is given by the instantaneous change of variables formula \cite{chen2018neural}
\begin{align} \label{eq:instchangevar} 
    \frac{D \log \rho(x(t),t)}{D t} = -\nabla \cdot v(x,t)\coloneqq -\sum_{i = 1}^d \frac{\partial v_i}{\partial x_i}.
\end{align} 
The likelihood can be estimated by simply integrating \eqref{eq:instchangevar} along with $\de x/\de t = v(x(t),t)$ simultaneously. A proof of this formula is given in \cite{chen2018neural}; however it is a direct consequence of the definition of \emph{material derivative} and the continuity equation, as we discuss in Remark \ref{remark:materialderiv}.

\begin{remark}[Instantaneous change of variables and the material derivative]
    The instantaneous change of variables formula \cite{chen2018neural} is a direct consequence of a standard result from fluid mechanics. In the Lagrangian description of fluid mechanics, the {material derivative} describes the rate of change of a quantity along a single flow trajectory \cite{munson2013fluid}. For a scalar quantity $\phi(x,t)$ and flow velocity $v(x,t)$, the material derivative is defined as
    \begin{align}
        \frac{D\phi}{Dt} \coloneqq \frac{\partial \phi}{\partial t} + v \cdot \nabla \phi.
    \end{align} 
    Therefore, the change of variables formula is simply the material derivative applied to the log-density of the normalizing flow. Observe that
    \begin{align*}
        \frac{D}{Dt}\log \rho(x(t),t)  & = \frac{1}{\rho(x,t)}\frac{\partial \rho(x,t)}{\partial t} + \frac{v(x,t)}{\rho(x,t)} \cdot \nabla \rho(x,t) \\
        & = -\frac{1}{\rho} \nabla \cdot(v\rho) + \frac{v\cdot \nabla \rho}{\rho} \\
        & = -\nabla \cdot v.
    \end{align*}
    \label{remark:materialderiv}
\end{remark}

Normalizing flows easily fit into the mean-field games framework. Our first result is a general normalizing flows formulation in terms of mean-field games, and is given in Theorem \ref{thm:normalizingflows}. \edits{ Given a reference distribution $\rho_{ref}$, some running cost $L(x,v)$, $\mathcal{I}(\rho) \coloneqq 0 $, and $\mathcal{M}(\rho) = \mathcal{D}_{KL}(\rho \|\rho_{ref})$, with the initial density being the target data distribution $\pi(x)$, many variants of normalizing flows can be derived. Here $\mathcal{D}_{KL}(\rho\| \rho') = \Ex_{\rho'}\left[ \log \frac{\rho}{\rho'}\right]$ is the Kullback-Leibler divergence.} %\BZ{Normalizing flows aim to find a velocity field that evolves data distribution $\pi$ to reference $\rho_{ref}$, typically chosen to be a normal distribution.}

\begin{theorem}
    Let $\pi$ be the target distribution. Generative models based on normalizing flows are solutions to the mean-field game
    \begin{align}
        \min_{v,\rho} \left\{ \edits{\mathcal{D}_{KL}(\rho(\cdot,T) \| \rho_{ref})} + \int_0^T \int_{\R^d} L(x,v) \rho(x,t) \de x \de t : \frac{\partial \rho}{\partial t} + \nabla\cdot\left(\rho v \right) = 0,\, \rho(x,0) = \edits{\pi(x)}  \right\}, \label{eq:mfgnf}
    \end{align}
    with dynamics 
    \begin{align}
        \frac{\de x}{\de t} = v(x(s),s).
    \end{align}
    Moreover, for sets $K \subseteq \R^d$, the solution of the mean-field game satisfies the optimality conditions for $H(x,p)= \sup_{v \in K}\left[ -p^\top v - L(x,v) \right],$
    \begin{align}
        \begin{dcases}
        - \frac{\partial U}{\partial t} + H(x,\nabla U) = 0\\
          \frac{\partial \rho}{\partial t} - \nabla \cdot(\rho \nabla_p  H(x,\nabla U)) = 0 \\
        U(x,T) = \edits{1+ \log \frac{\rho(x,T)}{\rho_{ref}(x)}}, \, \rho(x,0) = \edits{\pi(x)},
        \end{dcases}
        \label{eq:hjbnfgeneral}
    \end{align}
    where the optimal velocity field of the continuous normalizing flow has the representation formula
    \begin{align}
        {v^*}(x,t) = \arg \sup_{v}\left\{ -p^\top v - L(x,v) \right\}= -\nabla_p H(x,\nabla U).
    \end{align}
    \label{thm:normalizingflows}
\end{theorem}

Choosing the running costs to be identically zero, $L = 0$,  $\mathcal{I}=0$, and the terminal cost to be the KL divergence from the density $\rho(\cdot,T)$ to the \edits{reference distribution, $\mathcal{M}(\rho) = \mathcal{D}_{\text{KL}}(\rho\|\rho_{ref})$}, we will recover the standard normalizing flows objective in \cite{chen2018neural,grathwohl2018ffjord}. This is the standard KL divergence minimization problem, and its relationship to maximum likelihood estimation is very well-noted and understood.
\begin{proof}
    We perform the following computations, {while always implicitly assuming sufficient smoothness of all functions}. \edits{First note that by a change of variables, the terminal conditions can be re-expressed as follows
    \begin{align}\label{eq:klchangeofvars}
        \mathcal{D}_{KL}(\rho(x,T)\| \rho_{ref}(x)) = \mathcal{D}_{KL}(\pi(x) \| \tilde{\rho}(x,0)),
    \end{align}
    where $\tilde{\rho}(x,0)$ is the density of the inverse flow at time $t = 0$ given that $\tilde{\rho}(x,T) = \rho_{ref}(x)$. 
    Next observe that % the KL divergence can be expressed as 
    \begin{align*}
        \mathcal{D}_{KL}(\pi(x)\|\tilde{\rho}(x,0)) = \Ex_\pi\left[\log \pi(x) \right]- \Ex_\pi\left[\log \tilde{\rho}(x,0) \right].
    \end{align*}}
Furthermore, $\log \tilde{\rho}(x,0)$ can be expressed in terms of the reference density $\rho_{ref}(x)$ through the continuity equation and the definition of the material derivative. Integrating along a trajectory $x(t)$, we have that for $x(T) \sim \pi$,
    \begin{align} \label{eq:likelihoodintegral}
        \log \tilde{\rho}(x(0),0) &= \log \rho_{ref}(x(T)) - \int_T^0 \nabla \cdot v(x(s),s) \de s \\
        x(0) &= x(T) + \int_T^0 v(x(s),s) \de s.\nonumber
    \end{align}

    Note that for $L(x,v) = 0$, the resulting objective function is precisely the normalizing flows objective as described in \cite{chen2018neural,grathwohl2018ffjord}. Observe that \eqref{eq:mfgnf} is \edits{ now equivalent to 
    \begin{align}
        \min_{v} \left\{ -\Ex_{\pi}\left[\log \rho_{ref}(x) - \int_T^0 \nabla \cdot v(x(s),s) \de s \right]: x(s) = x + \int_T^s v(x(s'),s') \de s',\, x\sim \pi \right\}.
    \end{align}
    As for the optimality conditions, observe that for test function $\chi(x)$ where $\int\chi(x) \de x = 0$, 
    \begin{align*}
        \int \frac{\delta \mathcal{M}}{\delta \rho}\chi(x) \de x & = \lim_{\epsilon \to 0 } \frac{1}{\epsilon} \int \left( \log \frac{\rho(x) + \epsilon \chi(x)}{\rho_{ref}(x)} \right) (\rho(x)+\epsilon\chi(x)) - \left(\log \frac{\rho(x)}{\rho_{ref}(x)} \right) \rho(x) \de x \\
        & = \lim_{\epsilon \to 0} \frac{1}{\epsilon}\int \left(\log \left[\frac{\rho(x) + \epsilon \chi(x) }{\rho(x)} \right]\rho(x) + \epsilon \log\left[\frac{\rho(x)+\epsilon \chi(x)}{\rho_{ref}(x)} \right]\chi(x) \right) \de x \\
        & = \lim_{\epsilon \to 0} \int \left(\rho(x) \frac{\log(\rho(x)+\epsilon\chi(x)) - \log \rho(x) }{\epsilon} + \log\left[\frac{\rho(x)+\epsilon \chi(x)}{\rho_{ref}(x)} \right]\chi(x) \right) \de x \\
        & = \int \left( 1+ \log \frac{\rho(x)}{\rho_{ref}(x)}\right) \chi(x) \de x
    \end{align*}
% \begin{align*}
%     \int \frac{\delta \mathcal{M}}{\delta \rho}\chi(x) \de x &= \lim_{\epsilon \to 0}\frac{1}{\epsilon}\int \left(\log\frac{\pi(x)}{\rho(x)+\epsilon \chi(x)} - \log \frac{\pi (x)}{\rho(x)}\right) \pi(x) \de x \\
%     & = -\lim_{\epsilon \to 0} \int \frac{\pi(x)}{\epsilon} \left[\log(\rho(x) + \epsilon \chi(x)) - \log\rho(x) \right] \\
%     & = -\int \frac{\pi(x)}{\rho(x)}\chi(x) \de x,
% \end{align*}
so $M(x,\rho) =  1+ \log \frac{\rho(x)}{\rho_{ref}(x)}$.} Lastly the Hamiltonian and optimality conditions are clear from the potential formulation of MFG in \eqref{eq:HJBopt:pot}, see also \cite{lasry2007mean}. 
\end{proof}

\begin{remark}
    Even though there is no running cost function due to interaction, $\mathcal{I}(x,\rho)$, \eqref{eq:mfgnf} is still not an optimal control problem since the continuity and HJB equations are coupled by the terminal condition. This implies that {normalizing flows are {necessarily} solutions of mean-field games.} 
\end{remark}

\subsection{Well-posedness and structure of canonical continuous normalizing flows}

\label{sec:wellposed}

In this section we focus on the original formulation of continuous normalizing flows as discussed in \cite{chen2018neural,grathwohl2018ffjord}. While, in practice, training a normalizing flow does not require the MFG perspective, there is, however, significant insight that can be gained  by studying this objective function in the context of MFGs. In particular, we can study the optimality conditions of the corresponding HJB equation to understand the \emph{well-posedness} of the optimization problem and the \emph{structure} of the optimal velocity field $v^*(x,t)$.
\begin{corollary}    \label{cor:wellposed}
    Let $L = 0$, and $K\subset \R^d$ be a compact convex set. Then the optimizer of $\sup_{v\in K} -p^\top v$ is attained on the boundary of $K$. Moreover, when $K$ is a ball of radius $c$, then the optimal velocity field is
    \begin{align}
        v^*(x,t) = -c\frac{\nabla U(x,t)}{|\nabla U(x,t)|}, \label{eq:normalizedgradient}
    \end{align}
    where $U(x,t)$ solves the level-set equation \cite{osher2001level,sethian1999level} in the system of PDEs
    \begin{align}
        \begin{dcases}
        \frac{\partial U}{\partial t} - c|\nabla U| = 0 \\
        \frac{\partial \rho}{\partial t} - \nabla \cdot \left(c \rho\frac{\nabla U}{|\nabla U|} \right) = 0 \\
        U(x,T) = \edits{1+ \log \frac{\rho(x,T)}{\rho_{ref}(x,T)}, \, \rho(x,0) = \pi(x). }
        \end{dcases}
           \label{eq:nfhjb}
    \end{align}
\end{corollary}
\begin{proof}
Following the notation of Theorem~\ref{thm:normalizingflows},     the Hamiltonian of the system is 
$        H(x,p) = \sup_{v\in K} - p^\top v.$ Since $p^\top v$ is a convex function in $v$ and $K$ is convex and compact, by basic convex analysis, the optimizer must lie on the boundary \cite{boyd2004convex}. Therefore, if $K$ is  ball of radius $c$, then the optimizer is  $v = -cp/|p|$
so $H(x,p) = c|p|$, and by \eqref{eq:HJBopt}, we obtain the desired optimality conditions with 
\begin{align}
    v^*(x,t) = -c\frac{\nabla U(x,t)}{|\nabla U(x,t)|}.
\end{align}
\end{proof}
 Corollary~\ref{cor:wellposed} shows that for the standard continuous normalizing flow (CNF) formulation \cite{chen2018neural,grathwohl2018ffjord}, the Hamiltonian cannot be defined properly without additional assumptions on the set of feasible velocity fields $K\subset \R^d$. Specifically, $K$ is necessarily a bounded set for the Hamiltonian to be finite and for the normalizing flows optimization problem to be well-posed.  Furthermore, if $K$ is a convex set, the optimal control must lie on the boundary $\partial K$. We can interpret this as a form of \emph{bang-bang} control, in which the optimal control always lies on the boundary of the feasible set \cite{evans2005introduction}.

\begin{remark}
The typical assumption when learning CNFs is that the velocity fields only need to be Lipschitz continuous so that the resulting ODE has a solution. We argue that our result here shows that the velocity field must come from a bounded set of feasible controls and that successful implementations of CNFs implicitly enforce this constraint.  In fact, \cite{verineexpressivity} states that in practice, most normalizing flows are bi-Lipschitz, meaning that both the map and its inverse are Lipschitz. %This implies that for continuous normalizing flows, the velocity fields are, in practice, bounded. 
Corollary \ref{cor:wellposed} describes the precise mathematical justification for this observation. 
 \end{remark}
 \begin{remark}[Well-posedness of normalizing flows through MFG] 
Learning neural ODEs through empirical risk minimization is an ill-posed problem \cite{ott2020neural}. We argue that the MFG formalism may be one approach to study the well-posedness of normalizing flow. Indeed, taking advantage of the MFG optimality conditions, the normalizing flow variational problem is well-posed {only if} its corresponding MFG PDE system is well-posed \cite{lasry2007mean,bensoussan2013mean}.  Moreover, we may even study the well-posedness of certain GANs formulations based on normalizing flows, such as the Flow-GAN \cite{grover2018flow}. If the generator of a GAN is a continuous normalizing flow, then we may study them as an MFG with different choices of terminal cost. %\todo{Be more precise about which theorem in Lasry-Lions}
\end{remark}

\subsection{Variants of normalizing flows via MFGs}
With the MFG framework for normalizing flows at hand in Theorem \ref{thm:normalizingflows}, we are able to derive known variants of normalizing flows, re-write them as MFG, and analyze them based on their optimality conditions and corresponding PDEs.

\subsubsection{Optimal--transport normalizing flow}\label{subsec:CNF:OT}

 A special MFG that is well-studied is the case when the running cost is a term that can be interpreted as kinetic energy $L(x,v) = \frac{1}{2}|v|^2$. This yields a Hamiltonian that is quadratic in the co-state as $H(x,p) = \frac{1}{2}|p|^2$, which together with Theorem~\ref{thm:normalizingflows} implies that the optimal control is simply the gradient of the value function: $v^*(t,x) = -\nabla U(t,x)$. Not only does this special case simplify and provide additional structure to the problem, it also reveals connections to optimal transport and other classes of generative models. In particular, quadratic running cost functions can be used to regularize the training of normalizing flows \cite{onken2021ot,huang2022bridging}. The addition of a nonzero running cost to the normalizing flow MFG is also another approach to make the normalizing flow optimization problem well-posed. 
 
 In \cite{onken2021ot}, the authors introduce the optimal transport flow (OT-flow), which trains normalizing flows with the additional quadratic running cost. They presented their method as combining optimal transport with normalizing flows. Through numerical examples, they found improved regularity of the normalizing flow models in that the resulting trajectories have straighter solutions, which are easier to integrate over time. \edits{The OT-flow is the following potential MFG:} 
\begin{align}
    \inf_{v,\rho} \left\{\edits{\mathcal{D}_{KL}(\rho(\cdot,T)\| \rho_{ref})} + \int_0^T \int_{\R^d} \frac{1}{2} |v(x,t)|^2 \rho(x,t) \de x \de t: \frac{\partial \rho}{\partial t}+ \nabla \cdot(v\rho) = 0, \, \edits{\rho(x,0) = \pi(x)}\right\}. \label{eq:otnf}
\end{align} 
The Hamiltonian is $H(x,p) = \frac{1}{2}|p|^2 $, so by Theorem~\ref{thm:normalizingflows},
$v(x,t) = -\nabla U(x,t)$ with  optimality conditions  
\begin{align}
    \begin{dcases}
        &-\frac{\partial U}{\partial t} + \frac{1}{2}| \nabla U|^2 = 0 \\
    &\frac{\partial \rho}{\partial t} - \nabla \cdot \left(\rho \nabla U \right) = 0 \\
    &\edits{U(x,T) = 1+\log\frac{\rho(x,T)}{\rho_{ref}(x)}, \, \rho(x,0) = \pi(x). }
        \end{dcases} 
        \label{eq:otnfhjb}
\end{align}
The main takeaway is that the addition of the quadratic running cost function produces a generative model that appears to be a time-dependent gradient flow due to the structure of the continuity equation for $\rho$. Moreover, \cite{onken2021ot} found that OT-flow models are easier and faster to train as it required fewer parameters. This is likely due to the fact that only a scalar potential function needs to be learned rather than every component of a possibly high-dimensional vector field. Moreover, by explicitly constructing a model informed by the inductive bias of the generative flow, the velocity field can train faster. This formulation was also explored in \cite{huang2022bridging} for discrete normalizing flows, in which the flow map of a continuous normalizing flow is learned rather than the velocity field. 

\edits{For completeness, the following optimization problem can be derived from \eqref{eq:otnf}, the log-likelihood formula \eqref{eq:likelihoodintegral}, and the change of variables \eqref{eq:klchangeofvars}, and is exactly the OT-flow problem. The objective function can be computed via training samples from $\pi$:}
\begin{align}
    &\inf_{U}\left\{\Ex_{\pi}\left[ -\log\rho_{ref}(x(T)) - \int_T^0 \Delta U(x(s),s) \de s + \frac{1}{2}  \int_0^T |\nabla U({x}(s),s)|^2 \de s \right]\right\} \\
    &x(s) = x + \int_T^s \nabla U(x(t),t) \de t, \, x \sim \pi. \nonumber
    % &\tilde{x}(s) = \tilde{x} + \int_0^s \nabla U(\tilde{x}(t),t) \de t, \, \tilde{x}\sim \rho_0.   \nonumber
\end{align}
% In practice, \cite{onken2021ot} formulates a modification of the problem here where the flow from $\pi$ to a normal distribution is learned instead. In that case, the variational problem is such that the objective can be computed by only considering simulations of trajectories starting at $\pi$. 

\subsubsection{Boltzmann generators as MFGs}
\label{subsec:CNF:Boltzmann}

Boltzmann generators are normalizing flows that have been successfully used for enhancing sampling from complex distributions that arise in computational chemistry. Sampling Boltzmann distributions $\pi(x) \propto \exp(-V(x))$ is challenging as the energy function $V(x)$, while known up to a normalizing constant, may be highly multimodal. The task is often challenging with traditional sampling methods such as Markov chain Monte Carlo or molecular dynamics methods. Boltzmann generators use a combination of generative modeling and sampling methods to more efficiently explore the energy landscape \cite{noe2019boltzmann,kohler2020equivariant}. Using actual samples from the distribution $\pi(x)$ produced by some traditional sampling method, a normalizing flow is constructed, which is then used to construct an importance sampling distribution to produce more samples of the target distribution. 

The optimization problem that is solved to train the importance sampling distribution can also be cast as a potential MFG, see Section~\ref{sec:potentialmfg}, as we demonstrate next. As both the evaluation of target density and samples from it are accessible, the terminal cost is a convex combination of the forward and reverse KL divergences. That is, we select
\begin{align}\mathcal{M}(x,\rho) = \lambda \mathcal{D}_{KL}(\pi \| \rho) + (1-\lambda)\mathcal{D}_{KL}(\rho\|\pi) \end{align} for some $\lambda\in [0,1]$. In \cite{noe2019boltzmann,kohler2020equivariant}, Boltzmann generators, which are variants of continuous normalizing flows, are introduced and they correspond to the following MFG, where $L(x,v) = 0$ and $\mathcal{I}(\rho) = 0$,
\begin{align}
     \min_{v,\rho} \left\{ \lambda \mathcal{D}_{KL}(\pi \| \rho) + (1-\lambda)\mathcal{D}_{KL}(\rho\|\pi) : \frac{\partial \rho}{\partial t} + \nabla\cdot\left(\rho v \right) = 0,\, \rho(x,0) = \rho_{ref}(x)  \right\}. \label{eq:boltzmanngen}
\end{align}
The optimality conditions are the same as \eqref{eq:nfhjb} except for the terminal condition
\begin{align}
    U(x,T) = \frac{\delta \mathcal{M}}{\delta \rho}(x) = - \lambda \frac{\pi(x)}{\rho(x)} + (1-\lambda)\left(1 + \log\frac{\rho(x)}{\pi(x)} \right)
\end{align}

 Without constraining the feasible velocity fields, the Boltzmann generator as presented in \cite{noe2019boltzmann} is ill-posed in a way identical to the canonical normalizing flows objective, which we discussed in Section \ref{sec:wellposed}. In \cite{kohler2020equivariant}, they instead parameterize the normalizing flow velocity field as the gradient of the potential function, similar to the OT-flow. However, based on the MFG optimality conditions in \eqref{eq:nfhjb}, we already know that the velocity field is not of that form, and should instead have a normalized gradient structure \eqref{eq:normalizedgradient}. Furthermore, we also know that a gradient structure for the velocity field is optimal only if there is a quadratic running cost function as shown in \eqref{eq:otnfhjb}. Therefore, to remedy the theoretical inconsistency in \cite{kohler2020equivariant}, we can simply introduce a quadratic running cost. Based on this MFG-based framework, we will present a corrected version of Boltzmann generators in Section \ref{sec:designing}.

\section{Score-based generative models as solutions to MFGs}\label{sec:SGM}

We derive the score-based generative modeling (SGM) with stochastic differential equations framework as presented in \cite{song2021score} through mean-field games. We first present an overview of SGM with notation that will be consistent with the subsequent MFG formulation. 

Let $\pi$ be the target data distribution and let $f: \R^d \times \R \to \R^d$ be a time-evolving vector field, and $\sigma: \R\to\R$ be a positive function. Let $Y(t)$ be the diffusion process over $s\in[0,T]$, 
\begin{align}\label{eq:SGM:fwd_sde}
    \de Y(s) &= -f(Y(s),T-s) \de t + \sigma(T-s) \de W(s) \\
    Y(0) &\sim \pi \nonumber
\end{align}
which evolves $\eta(\cdot,0) = \pi$ to $\eta(\cdot,T)$. Score-based generative modeling aims to find a process $X(t)\sim \rho(\cdot,t)$ that reverses the evolution of $Y(t)$, i.e., $\rho(x,t) = \eta(x,T-t).$ Remarkably, the reverse process can be expressed in terms of the score function, $\nabla \log \eta(y,s)$ of the diffusion process $Y(s)$ \cite{anderson1982reverse}. For $t \in [0,T]$, $X(t)$ evolves according to 
\begin{align}\label{eq:SGM:bwd_sde}
    \de X(t) &= \left[f(X(t),t) + \sigma(t)^2 \nabla \log \eta(x,T-t) \right] \de t + \sigma(t)\de W(t) \\
    X(0) &\sim \eta(\cdot,T). \nonumber
\end{align}
In other words, the reverse process evolves samples from $\rho(\cdot,0) = \eta(\cdot,T)$ to samples of $\rho(\cdot,T) = \eta(\cdot,0) =\pi(\cdot)$. In practice, the forward process is constructed such that $\eta(y,T)$ is close to a normal distribution. For example, if $f(y,t) \coloneqq y$ and $\sigma(t) \coloneqq \sqrt{2}$ then for sufficiently large $T$, $\eta(\cdot,T)$ will be approximately normal. New approximate samples from $\pi$ can then be generated by evolving samples from the normal distribution through the reverse process. 

For generative modeling, we only require learning the score function and then simulating the reverse SDE. The score function is modeled by some neural network $\mathsf{s}_\theta(y,t)$, which is trained through a score-matching objective. In essence, $\mathsf{s}_\theta(y,t)$ minimizes a least-squares loss function with respect to the true score function. The explicit score-matching objective
\begin{align}
    C_{ESM}(\theta) \coloneqq \int_0^T \frac{1}{2}\Ex_{\eta(y,s)}\left[ \sigma(T-s)^2 | \mathsf{s}_\theta(y,s) - \nabla \log \eta(y,s) |^2 \right] \de s \label{eq:esm}
\end{align}
is not computable since $\nabla \log\eta(y,t)$ is not accessible. Instead, the implicit score-matching objective \cite{song2020sliced} 
\begin{align}
    C_{ISM}(\theta) \coloneqq \int_0^T\Ex_{\eta(y,s)}\left[ \sigma(T-s)^2\left( \frac{1}{2}|\mathsf{s}_\theta(y,s)|^2 + \nabla \cdot \mathsf{s}_\theta(y,s)\right)\right] \de s \label{eq:ism}
\end{align}
and the denoising score-matching objective \cite{vincent2011connection,song2019generative}, which we omit here, are used. These objective functions avoid the need to evaluate $\eta(y,s)$ or its derivatives.

\subsection{Deriving SGM from an MFG}
We show that SGM can be derived from a {potential MFG \eqref{eq:potentialform}} with $L(x,v) = \frac{1}{2}|v|^2 - \nabla \cdot f$, $\mathcal{I} = 0$, and  $\mathcal{M}(\rho) = -\Ex_{\rho}\left[ \log\pi\right]$. 
\begin{theorem}
    Let $\pi$ be the target distribution. Score-based generative models are solutions to the mean-field game
    \begin{align}
        &\inf_{v,\rho} \left\{ -\int_{\R^d} \log \pi(x) \rho(x,T) \de x + \int_0^T \int_{\R^d} \left( \frac{1}{2} |v(x,t)|^2 - \nabla \cdot f(x,t) \right) \rho(x,t) \de x \de t  \right\} \label{eq:sgmmfg}\\
        &\text{s.t. }  \frac{\partial \rho(x,t)}{\partial t} + \nabla \cdot \left[(f(x,t) + \sigma(t) v(x,t)) \rho(x,t)\right] = \frac{1}{2} \sigma(t)^2 \Delta \rho(x,t) \nonumber \\
        &\rho_0(x) = \eta(x,T) \nonumber %\text{propagate $\pi$ through $f$ SDE},\nonumber 
    \end{align}    with controlled dynamics 
    \begin{align}\label{eq:controlledsde}
        \de x(t) =\left[ f(x(t),t)+ \sigma(t) v(x(t),t) \right] \de t + \sigma(t) \de W(t).
    \end{align}
The solution of the mean-field game satisfies the optimality conditions 
\begin{align}\label{eq:SGM:MFG_optimality}
    \begin{dcases}
        - \frac{\partial U}{\partial t} - f ^\top \nabla U + \frac{1}{2}|\sigma \nabla U|^2 + \nabla \cdot f = \frac{\sigma^2}{2}\Delta U \\
        \frac{\partial \rho}{\partial t} +  \nabla \cdot \left(( f - \sigma^2\nabla U)\rho \right) = \frac{\sigma^2}{2} \Delta \rho \\
        U(x,T) = -\log \pi(x), \,\rho(x,0) = \eta(x,T). 
    \end{dcases}
\end{align}   
    where the optimal velocity field has the representation formula
    \begin{align}
        v^*(x,t) = -\sigma(t) \nabla U(x,t).
    \end{align} 
    Moreover, the solution to the HJB equation solves a time-reversed Fokker-Planck equation
    \begin{align}
        \frac{\partial \eta(y,s)}{\partial s} &= \nabla \cdot (f(y,T-s)\eta(y,s)) + \frac{\sigma(T-s)^2}{2} \Delta \eta(y,s) \label{eq:hjbfp} \\
        \eta(y,0) &= \pi(y). \nonumber
    \end{align}
    where 
    \begin{align}
        U(x,t) = -\log \eta(x,T-t),
    \end{align}
    and $\rho^*(x,t) = \eta(x,T-t)$. Note that \eqref{eq:hjbfp} corresponds with {the uncontrolled forward SDE arising in SGM \eqref{eq:SGM:fwd_sde}, namely}
    \begin{align}\label{eq:Thm:score:uncontrolled}
        \de y(s) &= -f(y(s),T-s) \de s + \sigma(T-s) \de W_s \\
        y(0) &\sim \eta(\cdot,0) = \pi(\cdot)\, . \nonumber
    \end{align}
    {Finally, the continuity equation in the MFG optimality conditions \eqref{eq:SGM:MFG_optimality} is the Fokker-Planck  equation corresponding with the reverse SGM dynamics in \eqref{eq:SGM:bwd_sde}, i.e.  the generative part of the SGM.}
    \label{thm:scoremfg}
\end{theorem}
Here, the running cost function is a kinetic energy-like term minus the divergence of $f$, and the terminal cost is the \emph{cross-entropy} of target distribution $\pi$ relative to $\rho(\cdot,T)$, $\mathcal{M}(\rho) = -\Ex_\rho\left[\log \pi \right]$. The cross-entropy can also be interpreted as the sum of the entropy of $\rho(\cdot,T)$ and the KL divergence from $\pi$ to $\rho(\cdot,T)$, i.e., observe that $-\Ex_{\rho}[\log\pi] =  - \Ex_\rho[\log \rho]+\mathcal{D}_{KL}(\rho\|\pi)$. In the proof of this theorem, we will show that the forward-reverse SDE pair is \emph{naturally} induced via the optimality conditions of the MFG. 
\begin{proof}
    First note that the Hamiltonian is
    \begin{align}
        H(x,p) &= \sup_{v} -p^\top( f + \sigma v) - \frac{1}{2} |v|^2 + \nabla \cdot f \\
        & = -f ^\top p + \frac{1}{2}|\sigma p|^2 + \nabla \cdot f,\nonumber
    \end{align}
    where the optimizer is $v^* = -\sigma p$. The variational derivative of the cross entropy is
    \begin{align}
        M(x,\rho) = - \frac{\delta}{\delta \rho} \Ex_{\rho}\left[\log \pi \right] = -\log \pi(x).
    \end{align}
    This implies that the optimal velocity field is $v^* = -\sigma \nabla U$, where $U(x,t)$ solves the  Hamilton-Jacobi-Bellman equation
    \begin{align}
        -\frac{\partial U}{\partial t} - f^\top \nabla U + \frac{\sigma^2}{2} |\nabla U|^2 + \nabla \cdot f = \frac{\sigma^2}{2} \Delta U \label{eq:sgmhjb} \\
        U(x,T) = -\log \pi(x). \nonumber
    \end{align}
    Applying the logarithmic transformation $U(x,t) = - \log \eta(x,T-t)$ \cite{fleming2006controlled}, also known as the Cole-Hopf transform \cite{evans_PDE,berner2022optimal}, to \eqref{eq:sgmhjb}, we obtain
    \begin{align*}
        &\frac{1}{\eta}\frac{\partial \eta}{\partial t} + \frac{f^\top}{\eta}  \nabla \eta + \frac{\sigma^2}{2\eta^2} |\nabla \eta|^2 + \nabla \cdot f = -\frac{\sigma^2(\eta \Delta \eta- |\nabla \eta|^2)}{2\eta^2} \\
        \implies & \frac{\partial \eta}{\partial t} + f^\top\nabla \eta + \frac{\sigma^2}{2\eta}|\nabla \eta|^2 + \eta \nabla \cdot f = -\frac{\sigma^2}{2}\Delta \eta +\frac{\sigma^2|\nabla \eta^2|}{2\eta} 
    \end{align*}
which is equivalent to the following linear PDE
\begin{align*}
      \frac{\partial \eta(x,T-t)}{\partial t} + \nabla \cdot (f(x,t) \eta(x,T-t)) + \frac{\sigma(t)^2}{2}\Delta \eta(x,T-t) = 0, \, \eta(x,T-T) = \pi(x).
\end{align*}
If we then reparametrize time to be $s = T-t$, the linear PDE is equivalent to the Fokker-Planck equation
\begin{align}\label{eq:uncontrolledfp}
       & \frac{\partial \eta(x,s) }{\partial s} = -\nabla \cdot (-f(x,T-s)\eta(x,s) ) + \frac{\sigma(T-s)^2}{2} \Delta \eta \\
        &\eta(x,0) = \pi(x). \nonumber
\end{align}
This implies that the optimal velocity $v^*(x,t) = \sigma(t)\nabla \log \eta(x,T-t)$ is related to the solution of the Fokker-Planck of the uncontrolled SDE
\begin{align*}
    &\de y(s) = -f(y(s),T-s) \de s + \sigma(T-s) \de W(s) \\
    &y(0) \sim \pi \nonumber. 
\end{align*}
Moreover, we have that \[\rho^*(x,t) = \eta(x,T-t)\] is the solution to the Fokker-Planck of the controlled diffusion process in {Theorem~\ref{thm:scoremfg}} for {$v^*(x,t) = \sigma(t) \nabla \log \eta(x,T-t)$}. Observe that
\begin{align*}
    \frac{\partial \rho^*(x,t)}{\partial t} + \nabla \cdot \left[(f(x,t) + \sigma(t){  v^*(x,t)})\rho^*(x,t)\right] &= -\frac{\partial \eta }{\partial s} + \nabla \cdot\left[(f + \sigma^2 \nabla \log \eta )\eta \right] \\
    & = - \frac{\partial \eta}{\partial s} + \nabla \cdot\left[ f\eta + \sigma^2 \nabla \eta  \right] \\
    & = - \frac{\sigma^2}{2}\Delta \eta + \sigma^2\Delta \eta \\
    &= \frac{\sigma^2}{2} \Delta \rho^*.
\end{align*}
Finally, notice that the initial condition is determined by the solution of the Fokker-Planck $x(0) \sim \eta(\cdot,T)$. 
\end{proof}

In \cite{berner2022optimal}, it is shown that the reverse SDE can be derived by considering a stochastic optimal control problem. However, they pre-supposed the existence of the forward-reverse SDE pair to derive the optimal control problem. Through the MFG framework, we show that the forward-reverse SDE pair is \emph{induced} through the HJB equation of the MFG optimality condition, which in this particular setup, is equivalent to the Fokker-Planck equation for $\eta$ in Theorem \ref{thm:scoremfg} via a logarithmic transformation. 

\begin{remark}[Natural boundary conditions of the SGM MFG]
Notice that the initial condition in \eqref{eq:sgmmfg} are implicitly determined by the optimality conditions of the MFG. This is an example of a \emph{natural boundary condition} since the initial density $\rho_0(x)$ is not imposed explicitly in the MFG formulation. 
\end{remark}

\subsection{Score-matching objective through duality }

We relate the implicit score-matching objective \eqref{eq:ism} to the mean-field game formulation we study in Theorem \ref{thm:scoremfg}. It is unclear that \eqref{eq:sgmmfg} can be related to \eqref{eq:ism} directly, mainly due to the fact that in the former, the velocity field and densities must be optimized simultaneously, while the latter problem has the density and the velocity field decoupled. To make the connection to score-matching, we first need to discuss an inherent duality in the SGM MFG. 

The proof of Theorem \ref{thm:scoremfg} demonstrates that there is a dual nature of the MFG formulation since the optimality conditions reduce to a pair of Fokker-Planck equations --- one corresponding to the controlled SDE, the other corresponding to an uncontrolled SDE. The Hamilton-Jacobi-Bellman equation, which comes from the optimal control problem of the controlled SDE, \emph{induces} the existence of the uncontrolled SDE. The duality allows us to \emph{exchange the roles of the controlled and uncontrolled} systems through a reparametrization. The stochastic optimal control objective derived in \cite{berner2022optimal} makes the connection to the score-matching objectives through a reparametrization of the optimal control problem. The discussion we have here will properly justify that reparametrization considered in \cite{berner2022optimal,huang2021variational} {using our MFG framework}. We first prove a lemma that states the duality of the MFG and its optimality conditions. 

\begin{lemma}{(Duality and reparametrization)}
At optimality, the roles of the Fokker-Planck and HJB equations in Theorem \ref{thm:scoremfg} can be exchanged through a reparametrization. That is, the following two systems are equivalent

\begin{align}
    &\begin{dcases}
        -\frac{\partial U(x,t)}{\partial t} - f(x,t) ^\top \nabla U(x,t) + \frac{\sigma(t)^2}{2}|\nabla U(x,t)|^2 + \nabla \cdot f(x,t) = \frac{\sigma(t)^2}{2}\Delta U(x,t) \\
        \frac{\partial \rho(x,t)}{\partial t} + \nabla \cdot \left((f(x,t) - \sigma(t)^2 \nabla U(x,t))\rho(x,t)\right) = \frac{\sigma(t)^2}{2}\Delta \rho(x,t) 
    \end{dcases} \\
    \vspace{10pt}
   & \begin{dcases}
        \frac{\partial \eta(y,s)}{\partial s} + \nabla \cdot \left((-g(y,s) - \sigma(T-s)^2 \nabla V(y,s))\eta(y,s) \right) = \frac{\sigma(T-s)^2}{2}\Delta \eta(y,s)\\
       -  \frac{\partial V(y,s)}{\partial s} + g(y,s) ^\top \nabla V(y,s) - \frac{\sigma(T-s)^2}{2}|\nabla V(y,s)|^2 - \nabla \cdot g(y,s)  = \frac{\sigma(T-s)^2}{2}\Delta V(y,s), 
    \end{dcases}
\end{align}
where \begin{align}\label{eq:reparam} g(y,s) = f(y,T-s) + \sigma(T-s)v^*(y,T-s) = f(y,T-s) - \sigma(T-s)^2 \nabla U(y,T-s)\, . \end{align} Moreover, $U(x,t) = -\log \eta(x,T-t)$, $V(y,s) = -\log \rho(y,T-s)$, and $\rho(\cdot,t) = \eta(\cdot,T-t)$. 
\begin{proof}
Since the logarithmic transformation converts the HJB into a time-reversed Fokker-Planck equation, we can also apply the logarithmic transformation to the Fokker-Planck equation to produce a HJB equation. From Theorem \ref{thm:scoremfg}, if we let $U(x,t) = -\log \eta(x,T-t)$, then the optimality conditions yield a pair of Fokker-Planck equations,
\begin{align*}
    \begin{dcases}
        \frac{\partial \eta(y,s)}{\partial s} + \nabla \cdot \left((-f(y,T-s))\eta(y,s) \right) = \frac{\sigma(T-s)^2}{2}\Delta \eta(y,s)\\
        \frac{\partial \rho(x,t)}{\partial t} + \nabla \cdot \left((f(x,t) - \sigma(t)^2 \nabla U(x,t))\rho(x,t)\right) = \frac{\sigma(t)^2}{2}\Delta \rho.
    \end{dcases} 
\end{align*}
To exchange the roles of the uncontrolled FP with the controlled one, define $g(y,s) = f(y,T-s) - \sigma(T-s) \nabla U(y,T-s)$, and $V(y,s) = U(y,T-s)$. We then have
\begin{align*}
    \begin{dcases}
        \frac{\partial \eta(y,s)}{\partial s} + \nabla \cdot \left((-g(y,s)-\sigma(T-s) \nabla V(y,s))\eta(y,s) \right) = \frac{\sigma(T-s)^2}{2}\Delta \eta(y,s)\\
        \frac{\partial \rho(x,t)}{\partial t} + \nabla \cdot \left(g(x,T-t)\rho(x,t)\right) = \frac{\sigma(t)^2}{2}\Delta \rho.
    \end{dcases} 
\end{align*}
Lastly, perform a second logarithmic transformation $V(y,s) = -\log \rho(y,T-s)$, so that we derive
\begin{align*}
    \begin{dcases}
        \frac{\partial \eta(y,s)}{\partial s} + \nabla \cdot \left((-g(y,s)-\sigma(T-s) \nabla V(y,s))\eta(y,s) \right) = \frac{\sigma(T-s)^2}{2}\Delta \eta(y,s)\\
       - \frac{\partial V(y,s)}{\partial s} + g(y,s) ^\top \nabla V(y,s) + \frac{\sigma(T-s)^2}{2}| \nabla V(y,s)|^2- \nabla \cdot g(y,s)= \frac{\sigma(T-s)^2}{2}\Delta V(y,s).
    \end{dcases} 
\end{align*}
\end{proof}
\label{lem:dual}
\end{lemma}

Lemma \ref{lem:dual} shows two sets of optimal conditions have the same solution through a reparametrization at the optimum. Therefore, there is a second MFG that corresponds to the second set of optimality conditions in the lemma which has the same optimizer as \eqref{eq:sgmmfg}. Furthermore, we can show that this second MFG is equivalent to score-matching. 

\begin{theorem}
Mean-field games \eqref{eq:sgmmfg} and \eqref{eq:sgmmfgpar} 
\begin{align}
    &\inf_{v,\eta} \left\{ -\int_{\R^d} \log \rho(y,0)\eta(y,T) \de y + \int_0^T \int_{\R^d} \left( \frac{1}{2} |v(y,T-s)|^2  + \nabla \cdot g(y,s)\right) \eta(y,t) \de y \de t\right\}\label{eq:sgmmfgpar} \\
    &\text{s.t.}\, \frac{\partial \eta(y,s)}{\partial s} + \nabla \cdot\left[(-g(y,s) + \sigma(T-s) v(y,T-s) ) \eta(y,s) \right] = \frac{1}{2}\sigma(T-s)^2 \Delta \eta(y,s) \nonumber \\
    &\eta(y,0) = \pi(y), \nonumber
\end{align}
 with $g=g(y, s)$  defined in Lemma~\ref{lem:dual}, and the implicit score-matching problem \eqref{eq:impscorematch} {\color{black} with $\eta=\eta(y,s)$ given by \eqref{eq:hjbfp}},
\begin{align}
    \inf_{\mathsf{s}} \left\{\int_0^T \int_{\R^d} \sigma(T-s)^2 \left( \frac{1}{2}|\mathsf{s}(y,s)|^2 + \nabla \cdot \mathsf{s}(y,s)\right) \eta(y,s)\de y \de s \right\} \label{eq:impscorematch}
\end{align}
have the same minimizers. Note that the latter functional \eqref{eq:impscorematch} is identical to \eqref{eq:ism}.
\end{theorem}

\begin{proof}
The optimality conditions of mean-field games \eqref{eq:sgmmfg} and \eqref{eq:sgmmfgpar} were shown to be equivalent in Lemma \ref{lem:dual}. To show equivalence of MFG \eqref{eq:sgmmfgpar} to the score-matching objective, recall from Lemma \ref{lem:dual} that at optimality, $g(y,s) = f(y,T-s) + \sigma(T-s) v^*(y,T-s)$, and so
\begin{align*}
    &-\int_{\R^d} \log \rho(y,0) \eta^*(y,T) \de y + \\ & \int_0^T \int_{\R^d} \left( \frac{1}{2}|v^*(y,T-s)|^2 + \nabla \cdot (f(y,T-s) + \sigma(T-s) v^*(y,T-s))  \right)\eta^*(y,s) \de y \de s \\
    &= \inf_{v} \left\{  -\int_{\R^d} \log \rho(y,0) \eta^*(y,T) + \int_0^T \int_{\R^d} \left( \frac{1}{2}|v|^2 + \nabla \cdot (f + \sigma v)  \right)\eta^*(y,s) \de y \de s \right\},
\end{align*}
where $\eta^*(y,s)$ is the optimal density evolution. By Theorem~\ref{thm:scoremfg} and Lemma~\ref{lem:dual}, (and by some abuse of notation), we know that $\eta^*(y,s) = \eta(y,s)$ given by \eqref{eq:hjbfp}. Defining $\sigma(T-s)\mathsf{s}(y,s) = v(y,T-s)$, and ignoring terms that are independent of $v$, the above problem is equivalent to 
\begin{align*}
    \inf_{\mathsf{s}} \left\{ \int_0^T\int_{\R^d} \sigma(T-s)^2 \left( \frac{1}{2} |\mathsf{s}(y,s)|^2 + \nabla \cdot \mathsf{s}\right) \eta(y,s) \de y \de s \right\},
\end{align*}
which is precisely implicit score-matching. 
\end{proof}

For connections between the score-matching objectives, \eqref{eq:esm} and \eqref{eq:ism}, and the evidence lower bound (ELBO), see \cite{huang2021variational,song2021maximum}. The equivalence of implicit score matching and the frequently used denoising score-matching can be found in \cite{berner2022optimal,vincent2011connection}.

\subsection{MFGs unravel the curious mathematical serendipity of score-based generative modeling}
We take a broader perspective on the curious mathematical structure of the mean-field game associated with score-based generative modeling. The optimality conditions of any MFG \eqref{eq:potentialform} are two coupled PDEs \eqref{eq:HJBopt:pot}--- a continuity equation that is solved forwards in time, which determines the evolution of the density, and a HJB equation that is solved backwards in time, which determines the optimal control. In general, it is quite difficult to find structure that simplifies the coupled set of PDEs. It is intriguing, however, that for the MFG associated with SGM \eqref{eq:sgmmfg}, a series of curious mathematical coincidences conspire to create a mean-field game that not only decouples the pair of PDEs \eqref{eq:SGM:MFG_optimality}, but also yields two stochastic differential equations -- one controlled \eqref{eq:controlledsde}, one uncontrolled \eqref{eq:Thm:score:uncontrolled}, where the optimal control of the controlled SDE \eqref{eq:sgmhjb} is determined by the density evolution \eqref{eq:uncontrolledfp} of the uncontrolled SDE. Furthermore, the duality of the optimality conditions (Lemma \ref{lem:dual}) shows that there is, in fact, a pair of MFGs, \eqref{eq:sgmmfg} and \eqref{eq:sgmmfgpar}, that share the same solution, one of which \eqref{eq:sgmmfgpar} can be related to implicit score-matching \eqref{eq:ism}. 

We emphasize three key characteristics of the MFG formulation of SGM that makes the structure described above possible:
\begin{itemize}
    \item The logarithmic transformation that converts the HJB equation \eqref{eq:sgmhjb} into a Fokker-Planck \eqref{eq:uncontrolledfp} is only possible if \textbf{the controlled dynamical system \eqref{eq:controlledsde} is a stochastic differential equation}. The presence of the Laplacian operator on the right hand side of \eqref{eq:sgmhjb} is crucial for the logarithmic transformation to convert the nonlinear Hamilton-Jacobi equation into a linear PDE. 

    \item Both terms in the running cost $L(x,v) = |v|^2/2 - \nabla \cdot f$ are important. \textbf{The quadratic cost function yields a gradient structure} for the velocity field, which we later find is a score function. \textbf{The divergence term $\nabla \cdot f$} is needed so that after a logarithmic transformation is applied to the HJB, \textbf{the advection term of a Fokker-Planck equation \eqref{eq:uncontrolledfp} appears. }

    \item \textbf{The terminal cost is the \emph{cross-entropy} of $\pi$ with respect to the terminal density $\rho(\cdot,T)$}. This yields a terminal condition in the HJB equations \eqref{eq:sgmhjb} that is independent of the terminal density of the Fokker-Planck. Therefore, the mean-field game \eqref{eq:sgmmfg} is, in fact, a stochastic optimal control problem, \emph{cf.} the terminal conditions of the normalizing flow optimality conditions in \eqref{eq:hjbnfgeneral}. The cross-entropy cost also results in the solution to the HJB to be the logarithm of a density function. Further connections between the score-matching objective and the cross-entropy is discussed in \cite{vahdat2021score}.

\end{itemize}
The first two features imply that the HJB can be transformed into a time-reversed Fokker-Planck equation. These two characteristics alone do not necessarily imply that the HJB corresponds to the density evolution of an uncontrolled SDE. The third point regarding the cross-entropy is crucial since it determines that the terminal condition of the time-reversed Fokker-Planck equation is, indeed, a density function. Altogether, these three characteristics imply that the solution to the HJB \eqref{eq:sgmhjb} is related to the solution of an uncontrolled SDE. Without the cross-entropy terminal cost in SGM, the pair of SDEs still exists, except that both of them are controlled SDEs. In fact, the problem is the Schr\"odinger bridge problem, which we discuss in Section \ref{sec:Schrodinger}.

A fourth fact allows us to relate the MFG objective function to the implicit score-matching objective: 
\begin{itemize}
     \item The dual nature of the pair of Fokker-Planck equations imply that there are two equivalent mean-field games with the same solution (Lemma \ref{lem:dual}). That is, there is no requirement that one of the processes must be controlled or uncontrolled. A simple reparametrization \eqref{eq:reparam} can swap the roles of the two SDEs. 
\end{itemize}

\begin{remark}[A shooting method analogy]
Score-based generative modeling has been found to be superior to normalizing flows \cite{song2021score}, and the MFG perspective can give some intuitive insight into why this may be. One can interpret generative modeling as finding the trajectory on the space of measures that connects an initial distribution $\rho_0$ to the target distribution $\pi$. The MFG analysis of SGM shows that because the pair of PDEs yields a pair of SDEs, \eqref{eq:controlledsde} and \eqref{eq:Thm:score:uncontrolled}, that invert each other, the trajectory on the space of measures is prescribed \emph{a priori}. Therefore, score-matching is simply finding a least squares approximation \eqref{eq:esm} to that trajectory. 

In contrast, normalizing flows based on minimizing the KL divergence \eqref{eq:mfgnf} is {necessarily} a mean-field game, and so the \emph{velocity and density have to be optimized simultaneously}. Unlike SGM, there is no \emph{elegant relation between the density and the optimal velocity field}. %\MK{[Comments in red are really important, if we give a talk, they should be a couple of slides to show why MFG "explains" SGM/CNFs! So maybe repeat (minimally) some of the corresponding math formulas and refer to corresponding Theorems]}.
Intuitively, and in practice, % \MK{[Maybe the implications of the "red" comments are both intuitive and practical/algorithmic, not just math/explanatory?]},
the CNF problem can be thought of as a shooting problem on the space of probability measures. We first guess a velocity field, which results in an approximate density to $\pi$. Based on the discrepancy between the approximate and target distribution, the velocity field is updated so that the approximate density is closer to $\pi$. Reasoning by analogy, two point boundary value problems for ODEs on $\R^d$ are generally numerically difficult to solve with the shooting method. Therefore, we argue by analogy that normalizing flows are more difficult to train than SGMs. 
\end{remark}

\begin{remark}[Score-based generative models as normalizing flows]
By comparing Theorem \ref{thm:normalizingflows} with Theorem \ref{thm:scoremfg}, we can interpret score-based generative models as a normalizing flow with initial distribution $\eta(x,T)$ defined in \eqref{eq:hjbfp}. Specifically, SGMs are a normalizing flow with stochastic dynamics, a quadratic running cost, and the cross-entropy as the terminal cost. 
\label{remark:sgmnf}
\end{remark}

\subsection{Score-based probability flows as solutions of MFGs}\label{sec:SGM:prob_flow}
The probability flow ODE presented in \cite{maoutsa2020interacting,song2021score} defines a system of interacting particles with deterministic dynamics that approximates the solution of the Fokker-Planck equation. It also provides a connection between normalizing flows and score-based generative models. We present a theorem where we show that the probability flow can derived via a mean-field game without appealing to its relation to score-based models. 

We show that the probability flow ODE can be derived from a potential MFG with $L(x,v) = \frac{1}{2}|v|^2 - \frac{1}{2} \nabla \cdot f$, $\mathcal{I}(\rho) = \Ex_\rho\left[\frac{\sigma^2}{8}|\nabla \log \rho|^2 \right]$, and $\mathcal{M}(\rho) = -\Ex_\rho[\log\pi]$. 
\begin{theorem}
    Let $\pi$ be the target distribution. The score-based probability flow ODE is the solution of the mean-field game
    \begin{align}
        &\inf_{v,\rho} \left\{ - \frac{1}{2} \Ex_{\rho(x,T)}\left[ \log \pi(x) \right]+ \int_0^T \Ex_{\rho(x,t)}\left[\left(\frac{\sigma(t)^2}{8}|\nabla \log \rho(x,t)|^2+\frac{|v(x,t)|^2 - \nabla \cdot f(x,t)}{2}  \right) \right]\de t  \right\}
        \label{eq:probflowmfg}\\
        &\text{s.t. }  \frac{\partial \rho(x,t)}{\partial t} + \nabla \cdot \left[(f(x,t) + \sigma(t) v(x,t)) \rho(x,t)\right] = 0 \nonumber \\
        &\rho_0(x) = \eta(x,T) \nonumber %\text{propagate $\pi$ through $f$ SDE},\nonumber 
    \end{align}    with deterministic dynamics 
    \begin{align}\label{eq:probflow}
        \de x(t) =\left[ f(x(t),t)+ \sigma(t) v(x(t),t) \right] \de t.
    \end{align}
The solution of the mean-field game satisfies the optimality conditions 
\begin{align}\label{eq:probabilityflow:MFG_optimality}
    \begin{dcases}
        - \frac{\partial U}{\partial t} - f ^\top \nabla U + \frac{1}{2}|\sigma \nabla U|^2 + \frac{1}{2}\nabla \cdot f = \frac{\sigma^2|\nabla \rho|^2}{8\rho^2} - \frac{\sigma^2\Delta \rho}{4\rho} \\
        \frac{\partial \rho}{\partial t} +  \nabla \cdot \left(( f - \sigma^2\nabla U)\rho \right) = 0\\
        U(x,T) = - \frac{1}{2}\log \pi(x), \,\rho(x,0) = \eta(x,T). 
    \end{dcases}
\end{align}   
    where the optimal velocity field has the representation formula
    \begin{align}
        v^*(x,t) = -\sigma(t) \nabla U(x,t).
    \end{align} 
    Moreover, the solution to the HJB equation solves a time-reversed Fokker-Planck equation
    \begin{align}
        \frac{\partial \eta(y,s)}{\partial s} &= \nabla \cdot (f(y,T-s)\eta(y,s)) + \frac{\sigma(T-s)^2}{2} \Delta \eta(y,s) \label{eq:hjbfpprobflow} \\
        \eta(y,0) &= \pi(y). \nonumber
    \end{align}
    where 
    \begin{align}
        U(x,t) = -\frac{1}{2}\log \eta(x,T-t),
    \end{align}
    and $\rho^*(x,t) = \eta(x,T-t)$. Again, \eqref{eq:hjbfpprobflow} corresponds with the uncontrolled SDE:
    \begin{align}
        \de y(s) &= -f(y(s),T-s) \de s + \sigma(T-s) \de W_s \\
        y(0) &\sim \eta(\cdot,0) = \pi(\cdot)\, . \nonumber
    \end{align}
    % {Finally, the continuity equation in the MFG optimality conditions \eqref{eq:SGM:MFG_optimality} is the Fokker-Planck  equation corresponding with the reverse SGM dynamics in \eqref{eq:SGM:bwd_sde}, i.e.  the generative part of the SGM.}

\end{theorem}

\begin{proof}
First note that the Hamiltonian is
\begin{align}
    H(x,p) &=\sup_{v} -p^\top(f + \sigma v) - \frac{1}{2}|v|^2 + \frac{1}{2}\nabla \cdot f \nonumber \\
    & = -p^\top f + \frac{\sigma^2}{2}|p|^2 + \frac{1}{2} \nabla \cdot f,
\end{align}
with the optimizer being $v^* = -\sigma p$. Next, notice that the interaction cost functional is
\begin{align}
    \mathcal{I}(\rho) = \int_{\R^d} \frac{\sigma^2}{8} |\nabla \log \rho(x,t)|^2 \rho(x,t) \de x.
\end{align}
We compute the variational derivative; for some test function $\chi$, such that $\int \chi \de x = 0$, observe that
\begin{align*}
   \lim_{\epsilon \to 0} \frac{\mathcal{I}(\rho + \epsilon \chi) - \mathcal{I}(\rho)}{\epsilon} &= \lim_{\epsilon \to 0} \frac{\sigma^2}{8\epsilon} \int_{\R^d} \frac{|\nabla(\rho + \epsilon \chi)|^2}{\rho + \epsilon \chi} - \frac{|\nabla \rho|^2}{\rho} \de x \\
   & = \lim_{\epsilon \to 0} \frac{\sigma^2}{8\epsilon} \int_{\R^d} \frac{\rho|\nabla\rho|^2 + 2\rho\epsilon \nabla \rho \cdot \nabla \chi + \epsilon^2 \rho|\nabla \chi|^2 - \rho|\nabla \rho|^2 - \epsilon \chi |\nabla \rho|^2}{\rho(\rho+\epsilon \chi)}  \de x\\
   & = \lim_{\epsilon \to 0} \frac{\sigma^2}{8} \int_{\R^d} \frac{2\rho\nabla \rho \cdot \nabla \chi + \epsilon \rho |\nabla \chi|^2 - \chi |\nabla \rho|^2}{\rho(\rho+\epsilon\chi)} \de x \\
   & = \frac{\sigma^2}{8} \int_{\R^d} 2\nabla \log \rho \cdot \nabla \chi - \frac{|\nabla \rho|^2}{\rho^2}\chi \de x \\
   & = \int_{\R^d} \left(-\frac{\sigma^2}{4}\Delta \log \rho - \frac{\sigma^2 |\nabla \rho|^2}{8\rho^2} \right) \chi \de x \\
   & = \int_{\R^d}  \left(-\frac{\sigma^2\Delta \rho}{4\rho} + \frac{\sigma^2|\nabla \rho|^2}{4\rho^2} - \frac{\sigma^2 |\nabla \rho|^2}{8\rho^2} \right) \chi \de x \\
    & = \int_{\R^d}  \left(-\frac{\sigma^2\Delta \rho}{4\rho}  + \frac{\sigma^2 |\nabla \rho|^2}{8\rho^2} \right) \chi \de x.
\end{align*}
Therefore, the HJB-FP system of PDEs is
\begin{align}
\begin{dcases}
    - \frac{\partial U}{\partial t}- f^\top \nabla U + \frac{1}{2}\sigma^2|\nabla U|^2 + \frac{1}{2}\nabla \cdot f = \frac{\sigma^2 |\nabla \rho|^2}{8\rho^2}-\frac{\sigma^2\Delta \rho}{4\rho} \\
    \frac{\partial \rho}{\partial t} + \nabla \cdot (\rho(f- \sigma^2 \nabla U)) = 0  \\
    \rho(x,0) = \eta(x,T), \, U(x,T) = - \frac{1}{2} \log \pi(x),
    \end{dcases}
\end{align}
and by the derivation of the Hamiltonian, the optimal velocity is $v^*(x,t) = -\sigma(t) \nabla U(x,t)$. We claim that the solution of the HJB-FP system is of the form $U(x,t) = -\frac{1}{2}\log \eta(x,T-t)$ and $\rho(x,t) = \eta(x,T-t)$, and that $\eta$ solves a Fokker-Planck equation. To see that the optimal $U$ and $\rho$ can be expressed in terms of the density of an uncontrolled SDE, note the following computations. Define $s = T-t$, and observe that 
\begin{align*}
  & -\frac{1}{2\eta } \frac{\partial \eta}{\partial s}+ \frac{1}{2} f^\top \nabla \log \eta + \frac{\sigma^2}{8} |\nabla \log \eta|^2 + \frac{1}{2}\nabla \cdot f = \frac{\sigma^2|\nabla \eta|^2}{8\eta^2} - \frac{\sigma^2\Delta \eta}{4\eta} \\
   \implies & \frac{\partial \eta}{\partial s} - f^\top \nabla \eta - \frac{\sigma^2}{4}\eta|\nabla \log \eta|^2 - \eta \nabla \cdot f = \frac{\sigma^2\Delta \eta}{2}- \frac{\sigma^2}{4}\eta|\nabla \log \eta|^2 \\
   \implies & \frac{\partial \eta(x,s)}{\partial s}+  \nabla \cdot \left(-f(x,T-s)\eta (x,s)\right)= \frac{\sigma(T-s)^2}{2}\Delta \eta(x,s).
\end{align*}
With the time and logarithmic transformation, the terminal condition becomes an initial condition, where $\eta(x,0) = \pi(x)$. The HJB, once again, becomes the Fokker-Planck of an uncontrolled SDE.

Lastly, we check that $\rho(x,t) = \eta(x,T-t)$ is the solution of the controlled SDE with the optimal control $v^*(x,t) = -\sigma(t) \nabla U(x,t)$:

\begin{align*}
    \frac{\partial \rho}{\partial t} +\nabla \cdot (\rho(f- \sigma^2 \nabla U)) &= - \frac{\partial \eta}{\partial s} + \nabla \cdot \left[\eta \left(f + \frac{\sigma^2}{2} \nabla \log \eta \right) \right] \\
    & = - \frac{\partial \eta}{\partial s} + \nabla \cdot (f\eta) + \frac{\sigma^2}{2}\Delta \eta \\
    & = 0.
\end{align*}

\end{proof}
The interaction term here is the Fisher information functional of $\rho$. The MFG corresponding to the score probability flow in Equation \eqref{eq:probflowmfg} can also be interpreted as a regularization of the Benamou-Brenier problem. See \cite{chen2021stochastic,pavon2021data} for further discussion about the Fisher information functional, Schr\"odinger bridges, and the Benamou-Brenier formula.

\subsection{Schr\"odinger bridges} 
\label{sec:Schrodinger}

We discuss connections between the Schr\"odinger bridge problem (SBP), score-based generative modeling, and mean-field games. While there are many formulations and variations of the SBP, see \cite{chen2021stochastic,pavon2021data,leonard2014survey}, we focus on the optimal control version. Let $\rho_0$ and $\rho_T = \pi$ be two distributions. The Schr\"odinger bridge problem aims to find a stochastic process whose $t = 0$ and $t = T$ marginals are $\rho_0$ and $\rho_T$ respectively, 
\begin{align}
    \inf_{v} \left\{\int_0^T \int_{\R^d} \frac{1}{2} |v(x,t)|^2 \rho(x,t) \de x \de t\right\} \label{eq:sbpopt}  \\
    \text{s.t. } \de x(t) = \sigma v(x(t),t) \de t + \sigma \de W(t) \nonumber \\
    x(0) \sim \rho_0, \, x(T) \sim \rho_T = \pi. \nonumber
\end{align}
Here $\sigma>0$ is a constant. From \cite{leonard2014survey,pavon2021data}, it is shown that the solution to the SBP is the control 
\begin{align}
    v^*(x,t) = \sigma \nabla \log \phi(x,t)
\end{align}
 where $\phi(x,t)$ is the solution of the so-called Schr\"odinger system \cite{chen2021stochastic}

\begin{align}
    \begin{dcases}
            \frac{\partial \phi}{\partial t} + \frac{\sigma^2}{2}\Delta \phi = 0 \\
            \frac{\partial \widehat{\phi}}{\partial t} - \frac{\sigma^2}{2}\Delta {\widehat{\phi}} = 0 \\
            \phi(x,0)\widehat{\phi}(x,0) = \rho_0(x),\,          \phi(x,T)\widehat{\phi}(x,T) = \pi(x). 
    \end{dcases} \label{eq:schrodingersystem}
\end{align}
The above stochastic optimal control problem \eqref{eq:sbpopt} has a form that is quite similar to a mean-field game. The main difference is that the terminal cost is instead a terminal constraint $\rho_T = \pi$. It is possible to derive the Schr\"odinger system \eqref{eq:schrodingersystem} by interpreting the SBP as a mean-field game. To do so, we first consider a more general MFG.

\subsubsection{MFGs, Schr\"odinger bridges, and a system of forward-backward heat equations}

Consider the mean-field game with a quadratic running cost $L(x,v) = \frac{1}{2} |v|^2$ and arbitrary interaction and terminal cost functionals $\mathcal{I}$ and $\mathcal{M}$. Let $\sigma$ be constant. We have
\begin{align}
&\inf_{v,\rho} \left\{\mathcal{M}(\rho(\cdot,T)) + \int_0^T \mathcal{I}(\rho(\cdot,t)) \de t + \int_0^T \int_{\R^d} \frac{1}{2}|v(x,t)|^2 \rho(x,t) \de x \de t \right\} \label{eq:mfgquad} \\
&\text{s.t. } \frac{\partial \rho}{\partial t} + \nabla \cdot(v\rho) = \frac{\sigma^2}{2} \Delta \rho, \, \rho(x,0) = \rho_0. \nonumber
\end{align}
This particular mean-field game has been previously studied due to the structure of the resulting optimality conditions \cite{gueant2012mean}. One can easily check that the Hamiltonian is quadratic in the co-state $H(x,p) = \frac{1}{2}|p|^2$, and the optimality conditions are
\begin{align}
    \begin{dcases}
        -\frac{\partial U}{\partial t} + \frac{1}{2}|\nabla U|^2 - \frac{\sigma^2}{2} \Delta U = I(x,\rho) \\
        \frac{\partial \rho}{\partial t} - \nabla \cdot(\rho \nabla U ) - \frac{\sigma^2}{2}\Delta U = 0 \\
        \rho(x,0)= \rho_0(x), \, U(x,T) = M(x,\rho(x,T)) \\
    \end{dcases}
    \label{eq:quadhjb}
\end{align}
In \cite{gueant2012mean}, a variable transformation, which can be interpreted as a generalized logarithmic transformation, is used to reduce \eqref{eq:quadhjb} into a system of two inhomogeneous heat equations with related source terms. Define $\psi(x,t) = \exp\left(-\frac{U(x,t)}{\sigma^2} \right)$ and $\widehat{\psi}(x,t) = \rho(x,t) \exp\left( \frac{U(x,t)}{\sigma^2}\right)$. The optimality conditions \eqref{eq:quadhjb} reduce to
\begin{align}
    \begin{dcases}
    \frac{\partial \psi}{\partial t} + \frac{\sigma^2}{2}\Delta \psi = \frac{1}{\sigma^2} I(x,\psi\widehat{\psi}) \psi \\
    \frac{\partial \widehat{\psi}}{\partial t} - \frac{\sigma^2}{2}\Delta \widehat{\psi}  = -\frac{1}{\sigma^2} I(x,\psi\widehat{\psi})\widehat{\psi} \\
    \psi(x,T) = \exp\left(-\frac{M(x,\rho(x,T))}{\sigma^2} \right)\\
    \widehat{\psi}(x,0) = \frac{\rho(x,0)}{\psi(x,0)}.
    \end{dcases}
\end{align}
With the variable transformation, it is simple to see that this system closely matches the Schr\"odinger system \eqref{eq:schrodingersystem}. The main differences are that there is no interaction term $I$, and that the terminal condition is not a strict terminal constraint. Replacing a terminal constraint that is difficult to enforce with a penalty is sometimes considered to relax the optimization problem \cite{genevay2017gan}. 

The link between the SBP and mean-field games has attracted recent attention. For example, \cite{liudeep} argued the connection via the MFG objective, as we do here, and they further generalize the SBP by introducing an interaction cost function, which naturally arises in the MFG formulation, by choosing $I$ to be nonzero. Moreover, using solutions of the SBP for generative modeling has also been an active area of research. For example, \cite{wang2021deep,de2021diffusion,shi2023diffusion} have trained generative models based on solutions of the Schr\"odinger bridge problem. Moreover, \cite{chenlikelihood} learns Schr\"odinger bridge-based generative models using forward-backward SDE theory.

\section{Generative models based on  Wasserstein gradient flows as solutions to MFGs}

In this section we discuss how the Wasserstein gradient flow can be derived with mean-field games. Specifically, we construct a sequence of mean-field games whose limit is the Wasserstein gradient flow. This connection will allow us to introduce relaxations to the Wasserstein gradient flows and may lead to future computational investigations. 
\label{sec:wassersteingradient}

\subsection{Wasserstein gradient flows as solutions to MFG} 
\label{sec:wassersteinmfg}
We now formulate the Wasserstein gradient flow as the solution of an MFG. The key results that allow us to connect these two concepts come from the study of general metric gradient flows through variational principles. The main idea here is that a variational principle appears by combining gradient flows with elliptic regularization. It was shown in  \cite{rossi2011variational,segatti2013variational,rossi2019weighted,mielke2011weighted} that gradient flows on metric spaces can be described as minimizers of a so-called weighted energy dissipation (WED) functional. For Wasserstein space, this minimization problem can be related to the mean-field game described in \eqref{eq:wassersteinmfg}. Let $\mathcal{F}(\nu)$ be a geodesically convex energy functional on $\mathcal{P}(\Omega)$.\footnote{A constant speed geodesic between two measures $\mu,\nu \in \mathcal{P}(\Omega)$ is defined as $\mu_t = ((1-t)\text{Id} + t\mathcal{T})_\sharp \mu$, where $\mathcal{T}_\sharp \mu = \nu$, $\mathcal{T}$ is the optimal map, and $t \in [0,1]$. A functional $\mathcal{F}(\rho)$ is $\lambda$-geodesically convex if $\mathcal{F}(\mu_t) \le (1-t)\mathcal{F}(\mu) + t\mathcal{F}(\nu)$. } The following theorem establishes the connection between Wasserstein gradient flows and MFGs:

\begin{theorem}\label{thm:MFG:Wasserstein_grad_flow}
    For any $\epsilon>0$ and geodesically convex functional $\mathcal{F}(\rho)$, the MFG
\begin{align}
    &\inf_{v,\rho} \left\{\mathcal{F}(\rho(\cdot,T))e^{-T/\epsilon} + \int_0^T \frac{e^{-t/\epsilon}}{\epsilon}\mathcal{F}(\rho(\cdot,t)) \de t + \int_0^T\int_{\R^d} \frac{e^{-t/\epsilon}}{2}|v(x,t)|^2\rho(x,t) \de x \de t\right\} \label{eq:wassersteinmfg}\\
    &\text{s.t. } \frac{\partial \rho}{\partial t} + \nabla \cdot(v\rho) = 0\nonumber \\
    &\rho(x,0) = \rho_0(x) \nonumber
\end{align}
with dynamics
\begin{align}
    \frac{\de x}{\de t} = v(x(t),t)
\end{align}
satisfies the optimality conditions
\begin{align}
    \begin{dcases}
        -\epsilon\frac{\partial U}{\partial t} + U + \frac{\epsilon}{2}|\nabla U|^2 = \frac{\delta \mathcal{F}}{\delta \rho} \\
        \frac{\partial \rho}{\partial t} - \nabla\cdot (\rho\nabla U) = 0 \\
        U(x,T) = \frac{\delta \mathcal{F}}{\delta \rho} (\rho(\cdot,T)), \,\rho(x,0) = \rho_0(x), 
    \end{dcases}
\end{align}
with optimal velocity field $v^*(x,t) = -\nabla U(x,t)$. In particular, as $\epsilon \to 0$, the solution is the Wasserstein gradient flow, in which
\begin{align}
\begin{dcases}
    U = \frac{\delta \mathcal{F}}{\delta \rho} \\
    \frac{\partial \rho}{\partial t} = \nabla \cdot\left(\rho \nabla \frac{\delta \mathcal{F}}{\delta \rho}\right) \\
        \rho(x,0) = \rho_0(x).
 \end{dcases}
\end{align}

\end{theorem}
\begin{proof}

    To derive the optimality conditions, note that the Hamiltonian is
    \begin{align}
        H(x,p,t) &= \sup_{v} \left[ -p^\top v - \frac{e^{-t/\epsilon}}{2}|v|^2 \right] \nonumber\\
        &= \frac{e^{t/\epsilon}}{2}|p|^2, 
    \end{align}
    where the maximum is attained when $v^* = -e^{t/\epsilon}p$. Therefore, the optimality conditions are 
\begin{align}
    \begin{dcases}
- \frac{\partial U_\epsilon}{\partial t} + \frac{e^{t/\epsilon}}{2} |\nabla U_\epsilon|^2 = \frac{e^{-t/\epsilon}}{\epsilon} \frac{\delta \mathcal{F}}{\delta \rho} \\
\frac{\partial \rho}{\partial t} - \nabla \cdot \left(\rho e^{t/\epsilon}\nabla U_\epsilon \right) = 0  \\
U_\epsilon(x,T) = e^{-T/\epsilon} \frac{\delta \mathcal{F}}{\delta\rho}(\rho(\cdot,T)), \, \rho(x,0) = \rho_0(x). 
    \end{dcases}    
\end{align}
We may simplify the optimality conditions further. Let $U^\epsilon(x,t) = U(x,t) e^{-t/\epsilon}$, then the HJB equation is
\begin{align*}
    - \frac{\partial U}{\partial t} e^{-t/\epsilon} + \frac{e^{-t/\epsilon}}{\epsilon}U + \frac{e^{-t/\epsilon}}{2}|\nabla U|^2 = \frac{e^{-t/\epsilon}}{\epsilon} \frac{\delta\mathcal{F}}{\delta \rho},
\end{align*}
which after simplifying, we get the desired result
\begin{align}\label{eq:Wass:optimality:rescaled}
      \begin{dcases}
       - \epsilon \frac{\partial U}{\partial t} + U + \frac{\epsilon}{2}|\nabla U|^2 = \frac{\delta \mathcal{F}}{\delta \rho}
 \\
\frac{\partial \rho}{\partial t} - \nabla \cdot \left(\rho \nabla U\right) = 0. \\
U(x,T) = \frac{\delta \mathcal{F}}{\delta \rho}(\rho(\cdot,T)),\, \rho(x,0) = \rho_0(x). 
    \end{dcases}
\end{align}
Moreover, sending $\epsilon \to 0$, and assuming smoothness of solutions, we obtain the Wasserstein gradient flow. This limit is well-defined based on the results in \cite{rossi2011variational,rossi2019weighted}. 
\end{proof}

\begin{remark}[Regularized Wasserstein gradient flow]
Alternatively, we may rewrite the continuity equation in \eqref{eq:Wass:optimality:rescaled} as a regularized Wasserstein gradient flow:
\begin{align}
    \frac{\partial \rho}{\partial t} = \nabla \cdot\left( \rho \nabla \frac{\delta \mathcal{F}}{\delta\rho}\right) + \epsilon\nabla \cdot\left( \rho \nabla\left( \frac{\partial U}{\partial t} - \frac{1}{2}|\nabla U|^2 \right)\right) 
\end{align}
Again, we observe that if $\epsilon \to 0$, we recover the Wasserstein gradient flow. 
\end{remark}

Finally, we note that in the $\epsilon \to \infty$ limit, the MFG \eqref{eq:wassersteinmfg} converges to a regularized optimal transport problem, as discussed in detail in Section~\ref{subsec:geograd}. As a result, we can interpret different choices of $\epsilon$ as \emph{interpolating} between Wasserstein gradient flows ($\epsilon \to 0$) and (regularized) optimal transport ($\epsilon \to \infty$).

\subsection{Mean-field games with relaxation}
\label{sec:wassersteinrelax}

An alternative asymptotic analysis for $\epsilon \to 0$ can be carried out by  rewriting  the MFG system \eqref{eq:Wass:optimality:rescaled} in \textit{relaxation form},  \cite{jin1995relaxation}, namely as 
\begin{align}\label{eq:Wass:optimality:rescaled:relax}
      \begin{dcases}
       \frac{\partial U}{\partial t} - \frac{1}{2}|\nabla U|^2 = \frac{1}{\epsilon} \left(U- \frac{\delta \mathcal{F}}{\delta \rho}\right)
 \\
\frac{\partial \rho}{\partial t} - \nabla \cdot \left(\rho \nabla U\right) = 0
 \\
U(x,T) = \frac{\delta \mathcal{F}}{\delta \rho}(\rho(\cdot,T)),\, \rho(x,0) = \rho_0(x). 
    \end{dcases}
\end{align}
We refer to  \eqref{eq:Wass:optimality:rescaled:relax} as an MFG with relaxation due to the relaxation term $\frac{1}{\epsilon} \left(U- \frac{\delta \mathcal{F}}{\delta \rho}\right)$ in the Hamilton-Jacobi equation. We note the resemblance in structure to  the relaxation approximations for hyperbolic conservation laws  introduced  in \cite{jin1995relaxation}. In the latter work, an asymptotic analysis based on a Chapman-Enskog expansion demonstrates that the relaxation approximations converge as $\epsilon \to 0$ to the (physically correct) entropy solutions of multi-dimensional hyperbolic conservation laws. Here the Chapman-Enskog expansion of \eqref{eq:Wass:optimality:rescaled:relax} is due to the relaxation term $\frac{1}{\epsilon} \left(U- \frac{\delta \mathcal{F}}{\delta \rho}\right)$ which in the $\epsilon  \ll 1$ regime, enforces a local equilibrium --- hence a Chapman-Enskog expansion --- which takes the form
\begin{equation}\label{eq:Chapman_Enskog}
    U(x, t)= \frac{\delta \mathcal{F}}{\delta \rho}+\epsilon v(x, t)\, .
\end{equation}
We can substitute this in \eqref{eq:Wass:optimality:rescaled:relax} to obtain
\begin{align}\label{eq:Wass:optimality:rescaled:relaxCE}
      \begin{dcases}
       \frac{\partial U}{\partial t} - \frac{1}{2}|\nabla U|^2 = v
 \\
\frac{\partial \rho}{\partial t} = \nabla \cdot\left( \rho \nabla \frac{\delta \mathcal{F}}{\delta\rho}\right) + \epsilon\nabla \cdot\left( \rho \nabla v\right)
\\
U(x,T) = \frac{\delta \mathcal{F}}{\delta \rho}(\rho(\cdot,T)),\, \rho(x,0) = \rho_0(x).
    \end{dcases}
\end{align}
As $\epsilon \to 0$ the continuity equation yields the Wasserstein gradient flow 
\begin{align}
    \frac{\partial \rho}{\partial t} = \nabla \cdot\left( \rho \nabla \frac{\delta \mathcal{F}}{\delta\rho} \right)\, , \quad  \rho(x,0) = \rho_0(x).
\end{align}
In the derivations of \cite{jin1995relaxation} the authors also had to resolve  the next order term $ \mathcal{O}(\epsilon)$ to ensure the entropy condition is satisfied for the hyperbolic conservation laws.
In our case, the zeroth order term in the continuity equation in \eqref{eq:Wass:optimality:rescaled:relaxCE} is the Wasserstein gradient flow which is always well-posed. Therefore,  resolving the next order term
$\epsilon \nabla \cdot\left( \rho \nabla v\right)$, while feasible,  appears to be  less essential.

\begin{remark}[Well-posedness and boundary layers]
An important  use of the relaxation formulation \eqref{eq:Wass:optimality:rescaled:relax} is related to the well-posedness of the MFG formulation. Indeed,  the corresponding Chapman-Enskog expansion \eqref{eq:Chapman_Enskog} allows us
to justify the choice of the terminal condition $\mathcal{M}=\mathcal{F}$ in the potential MFG in Theorem~\ref{thm:MFG:Wasserstein_grad_flow}, as follows. In principle,  the choice of any $\mathcal{M}$ instead of $\mathcal{F}$ in \eqref{eq:wassersteinmfg},   may appear to be possible  or even user-defined. However in Theorem~\ref{thm:MFG:Wasserstein_grad_flow} we pick 
$\mathcal{M}=\mathcal{F}$ to avoid the boundary layer in the Hamilton Jacobi equation at $t=T$  and thus 
facilitate the $\epsilon \to 0$ limit,   by ensuring that  $U = \frac{\delta \mathcal{F}}{\delta \rho}$ at $t=T$. 
By a boundary layer we mean that  a choice  $\mathcal{M} \ne \mathcal{F}$ would imply a mismatch between the terminal condition in \eqref{eq:Wass:optimality:rescaled:relax}, which in this case would be $U(x,T) = \frac{\delta \mathcal{M}}{\delta \rho}(\rho(\cdot,T))$, and the Chapman-Enskog expansion \eqref{eq:Chapman_Enskog}, as $\epsilon \ll 1$.
In this sense, the Chapman-Enskog expansion \eqref{eq:Chapman_Enskog} of the relaxation MFG \eqref{eq:Wass:optimality:rescaled:relax} demonstrates that  the choice $\mathcal{M}=\mathcal{F}$ constitutes a  ``natural" boundary condition for the MFG \eqref{eq:wassersteinmfg}.
\end{remark}

\subsection{Applications to generative modeling}

Wasserstein gradient flows have applications in many problems in machine learning, including generative modeling and sampling. A classic example of a Wasserstein gradient flow is the overdamped Langevin dynamics, where the interaction functional is the KL divergence. This celebrated result was first noted in \cite{jordan1998variational}. Using this gradient flow structure for sampling and inference has been explored in \cite{maoutsa2020interacting,garbuno2020interacting,wang2022accelerated,wang2021particle,liu2016stein,nusken2023geometry}. Particle-based, non-parametric implicit generative models based on gradient flows is an active area of research. For example, the use of the kernelized maximum mean discrepancy for generative modeling was explored in \cite{arbel2019maximum}. Moreover, discretizations of the gradient flows have been used to construct implicit generative models \cite{arbel2019maximum,glaser2021kale,gu2022lipschitz,liutkus2019sliced,ansari2020refining}. A broader class of interaction functionals based on the generalization of $f$-divergences and integral probability metrics for implicit generative models in the small data regime was presented in \cite{gu2022lipschitz}. 

There has also been some investigation into the training of generative adversarial nets (GANs) using the tools of Wasserstein gradient flows. The training of a GAN can be interpreted as a gradient flow on Wasserstein space, and implicit generative models can be interpreted as particles moving according to the gradient of a learned discriminator of a GAN. Studying the convergence of a restricted class of GANs as a Wasserstein gradient flow was explored in \cite{mroueh2019sobolev}.

\section{Enhancing generative models with MFGs: HJB regularizers}
\label{sec:enhancing}
\edits{Characterizing generative flows in terms of MFGs provide formulations of generative models as solutions to partial differential equations. In particular, the optimal velocity field that is typically trained by solving an optimization problem is the solution to the Hamilton-Jacobi-Bellman equation. The use of HJB equations for accelerating training of generative models has been studied previously in specific flows \cite{onken2021ot,yang2020potential}. The MFG formulation of generative modeling in this paper allows us to incorporate HJB regularizers to 
\emph{any} generative model that can be written as a MFG, such as  score-based generative models, in a systematic fashion. Using HJB regularizers is a way to explicit encode the inductive biases of generative flows during their training. }

\edits{After reviewing HJB regularizers for normalizing flows in Section~\ref{sec:otflowhjbreg}, we formulate an HJB regularizer for score-based generative models in Section~\ref{sec:SGMHJBreg} and report the related numerical experiments in Section~\ref{sec:hjbsgmnum}.}

\subsection{Hamilton-Jacobi-Bellman regularizers}
\label{sec:hjbregularizer}
%The optimality conditions \eqref{eq:hjbnfgeneral} can also be of practical use computationally.  
Notice that the following penalty term can be added to the potential MFG objective function, $\mathcal{J}(v,\rho)$ in \eqref{eq:potentialform}, or any of its equivalent forms, without changing the true minimizer,
\begin{align}
\label{eq:HJBreg:CNF:general}
    \mathcal{R}(U,\rho) =  &\alpha_1 \int_0^T \int_{\R^d}\left| \frac{\partial U}{\partial t} - H(x,\nabla U)
    %+ \frac{\sigma^2}{2}\Delta U
    \right|\rho(x,t) \de x\de t \\ & + \alpha_2 \int_{\R^d}\left|U(x,T) - M(x,\rho(\cdot,T)) \right| \rho(x,T)\de x, \nonumber
\end{align}
where $\alpha_1,\alpha_2>0$. 
Recall that $v$ is related to $U$ through $v = -\nabla_p H(x,\nabla U)$, so we may also minimize the MFG objective in \eqref{eq:potentialform} over $U$. Notice that the first term enforces the solution of the HJB over $\R^d \times [0,T]$, while the second term enforces the terminal condition. In essence, since we know any minimizer of the potential MFG necessarily satisfies the HJB equation, the addition of the HJB regularizer penalizes solutions that violate the optimality conditions and may help find the minimizer faster. 
%Using the Hamilton-Jacobi-Bellman equations as a regularizer while solving the potential MFG \eqref{eq:potentialform} has been used for solving MFGs \cite{ruthotto2020machine} and for training generative models based on certain classes of normalizing flows \cite{onken2021ot,yang2020potential}. In particular, in \cite{onken2021ot}, the authors empirically found that by introducing this regularizer, fewer discretization steps and parameters are needed to train the optimal transport-based normalizing flow, thereby cutting down on computational costs, see also Section~\ref{subsec:CNF:OT}. 

%While the HJB regularizer \eqref{eq:HJBreg:CNF:general} can be used for any continuous normalizing flows, in Section~\ref{sec:designing}, we will show that it can be applicable to any generative model that can be written as an MFG, such as score-based generative models in Section \ref{sec:SGM}. In Section~\ref{sec:hjbsgmnum}, we show some numerical experiments demonstrating the effect of the HJB regularizer for score-based generative modeling. 

\subsection{HJB regularizers for normalizing flows}
\label{sec:otflowhjbreg}

\subsubsection{Optimal-transport normalizing flows}
In addition to introducing the optimal transport cost to the normalizing flow, \cite{onken2021ot} empirically found that incorporating the HJB optimality conditions as an additional regularizer to the optimization objective improved training performance. They also found that the solutions to be more regular, in the sense that particle trajectories were ``straighter.'' To be precise, the optimality conditions \eqref{eq:otnfhjb} defines the regularization functional to be
\begin{align}    \label{eq:hjbregularizer} 
    \mathcal{R}(U,\rho) =&\alpha_1 \int_0^T  \int_{\R^d}\left| \frac{\partial U}{\partial t} - \frac{1}{2}|\nabla U|^2 \right| \rho(x,t) \de x\de t \\ & + \alpha_2 \int_{\R^d} \left|U(x,T) - \left(1+ \log \frac{\rho(x,T)}{\rho_{ref}(x)}\right) \right| \rho(x,T)\de x, \nonumber
\end{align}  
for some $\alpha_1,\alpha_2\ge 0$. This regularizer does not change the true minimizer of the MFG \eqref{eq:otnf} as they share the same minimizer. In \cite{onken2021ot}, they do not enforce the terminal condition, i.e., $\alpha_2 = 0$, as it requires the ability to evaluate the data distribution $\pi(x)$. This regularizer can be cheaply evaluated since $U$ is exactly the potential function that is learned, and so only derivatives of $U$ are needed to evaluate the functional. 

The idea of using the HJB as a regularizer has been introduced previously in \cite{ruthotto2020machine,yang2020potential}. In \cite{ruthotto2020machine}, the authors solved a mean-field game with the tools of normalizing flows, and chose to include HJB regularizer to improve stability when solving the optimization problem. In \cite{yang2020potential}, GANs and flow-based generative models are learned with the regularization functional. This regularizer is a weak constraint since it is integrated over space and time rather than being enforced pointwise. The inclusion of a partial differential equation in a loss function in this weak sense is similar to the objective function for training physics-informed neural networks (PINNs) \cite{raissi2019physics}.
\edits{\subsubsection{Normalizing flows with bounded velocities}
We can use the results of Corollary~\ref{cor:wellposed} to formulate an HJB regularizer for the canonical formulation of normalizing flows with bounded velocities. The primary difference is the form of the HJB equation; the terminal condition is the same as for the OT-flow in \eqref{eq:hjbregularizer}
\begin{align}
    \mathcal{R}(U,\rho) =&\alpha_1 \int_0^T  \int_{\R^d}\left| \frac{\partial U}{\partial t} - c|\nabla U| \right| \rho(x,t) \de x\de t \\ & + \alpha_2 \int_{\R^d} \left|U(x,T) - \left(1+ \log \frac{\rho(x,T)}{\rho_{ref}(x)}\right) \right| \rho(x,T)\de x.\nonumber
\end{align}  }

\subsection{HJB regularizers for score-based generative modeling and beyond} 

\label{sec:SGMHJBreg}

The use of an HJB regularizer during the training of a generative model is not limited to the normalizing flows. Here we devise an HJB regularizer, discussed in Section \ref{sec:hjbregularizer}, for score-based generative models. Recall that for SGMs, the score function $\mathsf{s}(y,s) = \nabla \log \eta(y,s)$ is learned via score-matching. Furthermore, the score function is equal to the negative gradient of solution to the Hamilton-Jacobi-Bellman equation \eqref{eq:SGM:MFG_optimality} which comes from the SGM MFG \eqref{eq:sgmmfg} optimality conditions. We cannot use the HJB equation as a regularizer directly here since, in practice, each component of the score function is learned individually and so the approximate score function will not exactly be the gradient of a scalar function. However, by taking the gradient of the HJB, it is possible to derive a regularizer for the score function directly:
\begin{align}
    \mathcal{R}_1(\mathsf{s}) &= \alpha_1 \int_{\R^d} \int_0^T \left| \frac{\partial \mathsf{s}}{\partial s} - \nabla(f^\top \nabla \mathsf{s}) - \frac{\sigma^2}{2} \nabla |\mathsf{s}|^2 - \nabla(\nabla \cdot f) - \frac{\sigma^2}{2}\nabla(\nabla \cdot \mathsf{s}) \right|^p \eta(y,s) \de s \de y  \\
    &+ \alpha_2\int_{\R^d} |\mathsf{s}(y,0) - \nabla \log \pi(y)|^2 \eta(y,0) \de y,  \nonumber
\end{align}
where $\eta(y,s)$ is given by \eqref{eq:uncontrolledfp} and $p\ge 1$. Recall that by reparametrizing time, the terminal condition is instead an initial condition. Moreover, observe that enforcing the terminal condition amounts to adding an explicit score-matching objective for the data distribution \cite{song2019generative}! In particular, if the norm is chosen to be the Euclidean norm, then the terminal condition can be enforced by adding an additional implicit score-matching objective for the initial distribution $\eta(\cdot,0) = \pi$. In particular, observe that 

\begin{align}
    \mathcal{R}_1(\mathsf{s}) &= \alpha_1 \sum_{i = 1}^d \int_{\R^d} \int_0^T \left|\frac{\partial \mathsf{s}^{(i)}}{\partial s} - \frac{\partial f^\top}{\partial x_i} \mathsf{s} - f^\top\frac{\partial \mathsf{s}}{\partial x_i} - \sigma^2 \mathsf{s}^\top \frac{\partial \mathsf{s}}{\partial x_i} - \nabla\cdot \frac{\partial f}{\partial x_i}  - \frac{\sigma^2}{2} \Delta \mathsf{s}^{(i)}  \right|^p \eta(y,s) \de s \de y \\
    &+ \alpha_2\int_{\R^d} \left(|\mathsf{s}(y,0)|^2 + 2\nabla \cdot \mathsf{s}(y,0)\right)  \pi(y) \de y, \nonumber
\end{align}
where $\mathsf{s}^{(i)}$ denotes the $i$-th component of the score function. In Section \ref{sec:hjbsgmnum}, we will demonstrate the effect of this particular regularizer on a simple example.

Alternatively, if we implement a version of SGM that learns the log-density instead of the score function, we can use the HJB regularizer directly like in the OT-Flow in Section~\ref{sec:otflowhjbreg}. In this case, a scalar function $\varphi(y,s)$ is learned using the score matching objective, where $\mathsf{s}(y,s) = \nabla \varphi (y,s)$, and $U(y,s) = -\varphi(y,s)$. Moreover, since the terminal constraint cannot be imposed directly, as we do not have access to evaluations of $\pi(y)$, we can again use the score-matching objective: 
\begin{align}
    \mathcal{R}_2(\varphi) &= \alpha_1 \int_{\R^d} \int_0^T \left| \frac{\partial \varphi}{\partial s} -f\cdot \nabla \varphi - \frac{\sigma^2}{2} |\nabla \varphi|^2 - \nabla \cdot f - \frac{\sigma^2}{2}\Delta \varphi\right| \eta(y,s) \de s \de y \\ &+ \alpha_2 \int_{\R^d}\left( |\nabla \varphi(y,0)|^2 + 2 \Delta \varphi(y,0) \right) \pi(y) \de y . \nonumber
\end{align}
In contrast, for the OT-flow discussed in Section~\ref{sec:otflowhjbreg}, the terminal constraint cannot be applied as it requires evaluations of the target log-density. Learning such a scalar function via score-matching has been studied previously \cite{song2021train}. %We emphasize the use of an HJB regularizer is not restricted to normalizing flows in Section~\ref{sec:otflowhjbreg} or SGMs here. Any generative model formulated via the MFG framework can make use of an HJB regularizer, as we will also mention in the next two sections.

\edits{A similar PDE-based regularizer for SGM was studied in \cite{lai2023fp} where the uncontrolled Fokker-Planck equation is used to derive a PDE for the score function. A crucial difference is that the regularizer presented in \cite{lai2023fp} does not enforce the terminal condition associated with the PDE problem. In our paper, we can impose the terminal condition with little additional computational cost by score-matching. %The inclusion of the terminal condition significantly impacts the performance of the generative model. 
In Section~\ref{sec:hjbsgmnum} we show that the inclusion of the terminal cost can dramatically affect the quality of the generative model.  }

\begin{remark}
In this paper, we have focused on the training of score-based generative models by solving a mean-field game. On the other hand, we may also study mean-field games that can be solved through score-matching. Similar to how \cite{ruthotto2020machine,huang2022bridging} use tools for training normalizing flows to solve mean-field games, we may also study how score-matching can solve certain mean-field games and its associated high-dimensional PDEs. For example, the use of entropy-based regularizers for solving mean-field games have also been recently explored \cite{guo2022entropy}.  In particular, one could solve certain high-dimensional HJB equation with only samples of the terminal data using score-based generative modeling methods. 
\end{remark}

\subsection{A numerical demonstration: HJB-regularized SGM}
\label{sec:hjbsgmnum}

%In this section we demonstrate one computational enhancement that is provided by the MFG perspective of generative models. In particular, 

We demonstrate the effects of the HJB regularizer for score-based generative modeling discussed in Section~\ref{sec:hjbsgmnum}. We will highlight numerical and algorithmic aspects one could explore in the future. The results here constitute only a preliminary investigation of the potential computational contributions of MFG-based generative modeling.

In this example, we choose $f(x,t) = \frac{x}{2}$ and $\sigma(t) = 1$, so that the induced noising process is an Ornstein-Uhlenbeck (OU) process. Recall from Section~\ref{sec:SGMHJBreg} that the regularizer for this example is
\begin{align}
    \mathcal{R}_1(\mathsf{s}) =& \alpha_1 \sum_{i = 1}^d \int_{\R^d} \int_0^T \left|\frac{\partial \mathsf{s}^{(i)}}{\partial s} - \frac{\partial f^\top}{\partial x_i} \mathsf{s} - f^\top\frac{\partial \mathsf{s}}{\partial x_i} -  \mathsf{s}^\top \frac{\partial \mathsf{s}}{\partial x_i}   - \frac{1}{2} \Delta \mathsf{s}^{(i)}  \right|^p \eta(y,s) \de s \de y \label{eq:hjbregularizer2} \\
    &+ \alpha_2\int_{\R^d} \left(|\mathsf{s}(y,0)|^2 + 2\nabla \cdot \mathsf{s}(y,0)\right)  \pi(y) \de y, \nonumber
\end{align}
where for some $p\ge 1$. In this example, we choose $p = 1$ or $ p = 2$. The regularized implicit score-matching problem is then
\begin{align}
    \min_\theta \alpha_0\int_0^T \int_{\R^d} \left(\frac{1}{2}|\mathsf{s}_\theta(y,s)|^2 + \nabla \cdot \mathsf{s}_\theta(y,s) \right) \eta(y,s) \de y \de s + \mathcal{R}_1(\mathsf{s}_\theta). \label{eq:regscorematching}
\end{align}
The data distribution $\pi(y)$ is the checkerboard distribution, see Figure~\ref{fig:a}. The checkerboard distribution presents particular challenges to score-based generative models, mainly in that the distribution has disconnected support and the density is discontinuous. 

We informally investigate and observe the impact of different choices of the HJB regularization parameters, $\alpha_0$, $\alpha_1$, $\alpha_2$, and $p$. For all examples, the total simulation time is $T = 3$, and the reverse SDE is simulated with an Euler-Maruyama scheme with $\Delta t = 0.001.$ The forward SDE is an OU process and can be simulated exactly. For all experiments, the score function $\mathsf{s}_\theta:\R^2\times\R \to \R^2$ is a neural network with two hidden layers, with 32 nodes per hidden layer. We use the Gaussian error linear unit (GeLU) activation function, and train the network over $10^5$ batches of 64 samples, with learning rate $0.001$. The derivatives are computed using the DeepXDE package \cite{lu2021deepxde}. 

\begin{figure}[h]
\captionsetup[subfigure]{justification=centering}
    \centering
    \begin{subfigure}{0.24\textwidth}
    \centering
\includegraphics[width = \textwidth,trim = 180 30 180 30,clip]{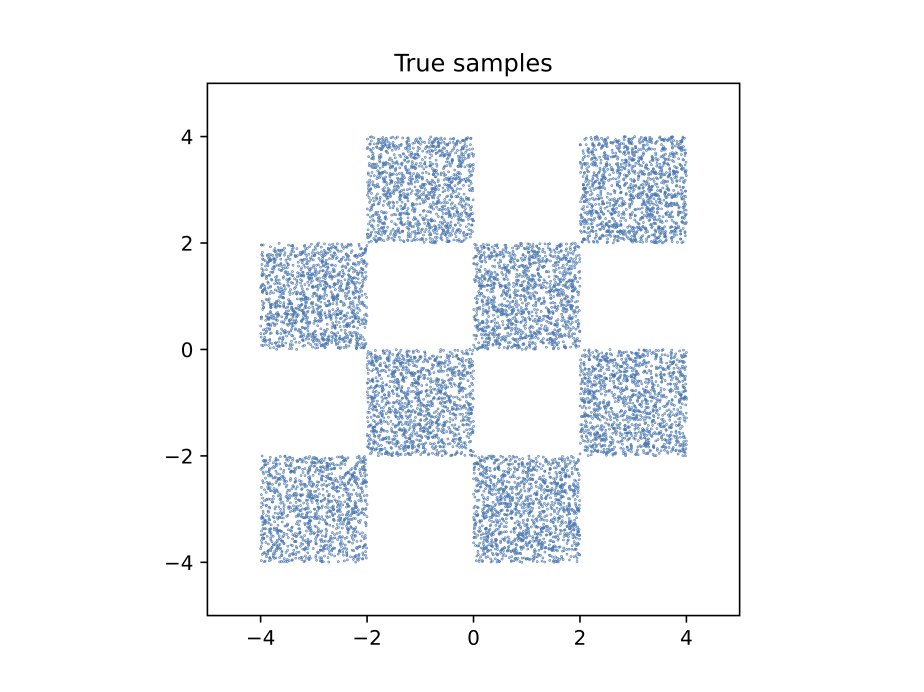} % 
\caption{\scriptsize True samples \\ \textcolor{white}{filler space filler space}}\label{fig:a}
    \end{subfigure}
    \begin{subfigure}{0.24\textwidth}
    \centering
    \includegraphics[width = \textwidth,trim = 180 30 180 30,clip]{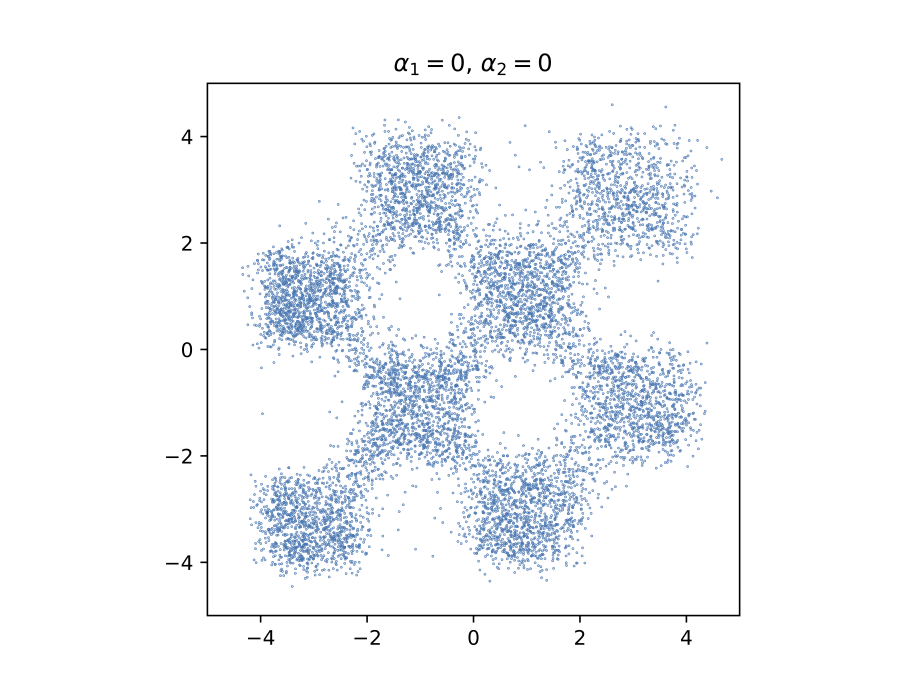} %
    \caption{\scriptsize Without any HJB regularization}\label{fig:b}
    \end{subfigure}
   \begin{subfigure}{0.24\textwidth}
    \centering
    \includegraphics[width = \textwidth,trim =  180 30 180 30,clip]{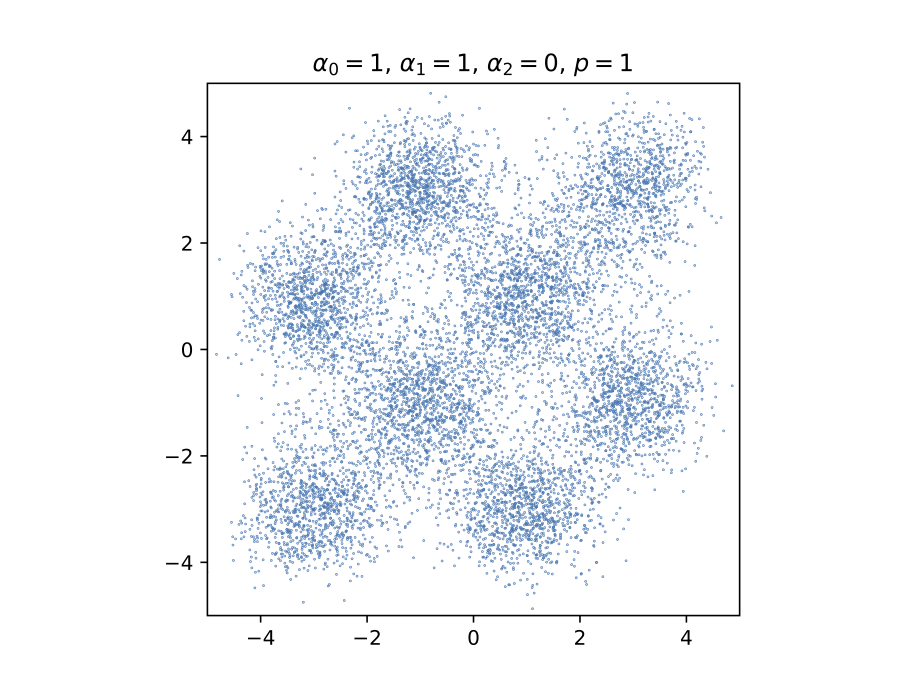}%
    \caption{\scriptsize HJB eq., no terminal. $p = 1$}\label{fig:newb}
    \end{subfigure}
    \begin{subfigure}{0.24\textwidth}
    \centering
    \includegraphics[width = \textwidth,trim = 180 30 180 30,clip]{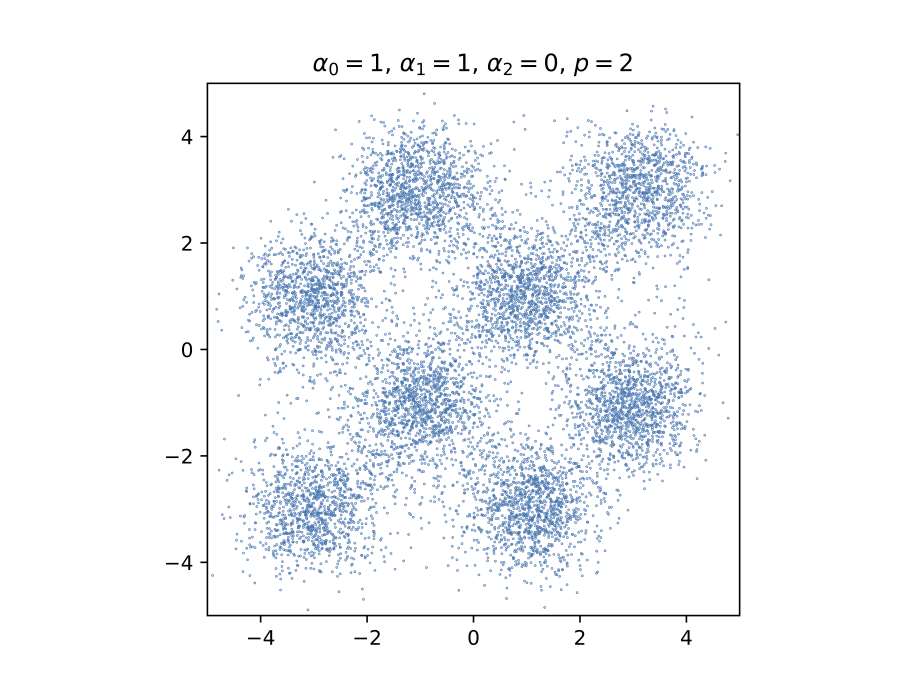}%
    \caption{\scriptsize HJB eq., no terminal. $p = 2$}\label{fig:newc}
    \end{subfigure}

    \begin{subfigure}{0.24\textwidth}
    \centering
    \includegraphics[width = \textwidth,trim = 180 30 180 30,clip]{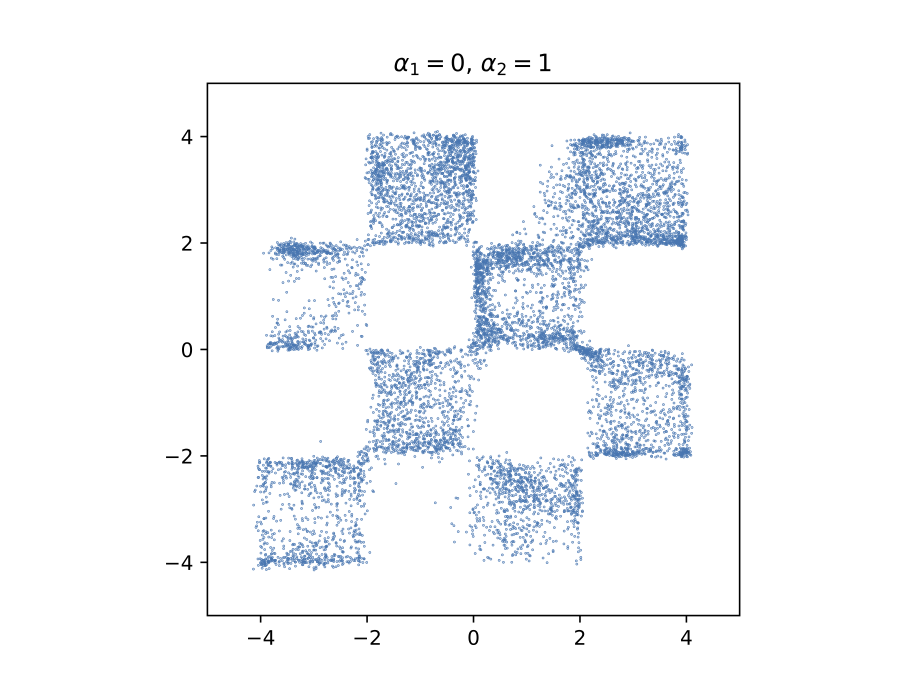}%
    \caption{\scriptsize Terminal constraint only}\label{fig:c}
    \end{subfigure}
    \begin{subfigure}{0.24\textwidth}
    \centering
    \includegraphics[width = \textwidth,trim =180 30 180 30,clip]{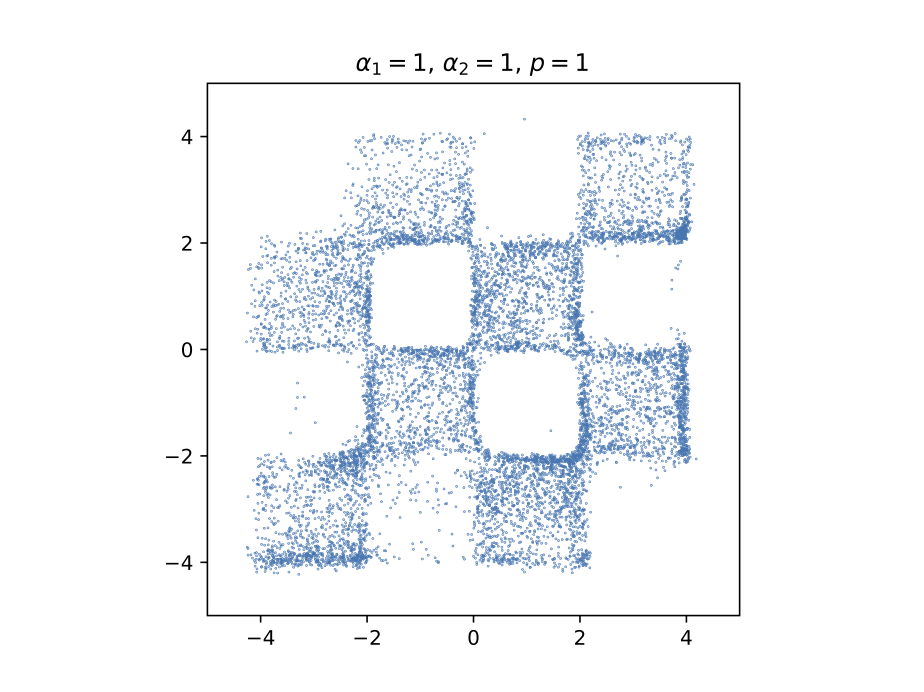}%
    \caption{\scriptsize HJB eq. $+$ terminal, $p = 1$}\label{fig:d}
    \end{subfigure}
    \begin{subfigure}{0.24\textwidth}
    \centering
\includegraphics[width = \textwidth,trim = 180 30 180 30,clip]{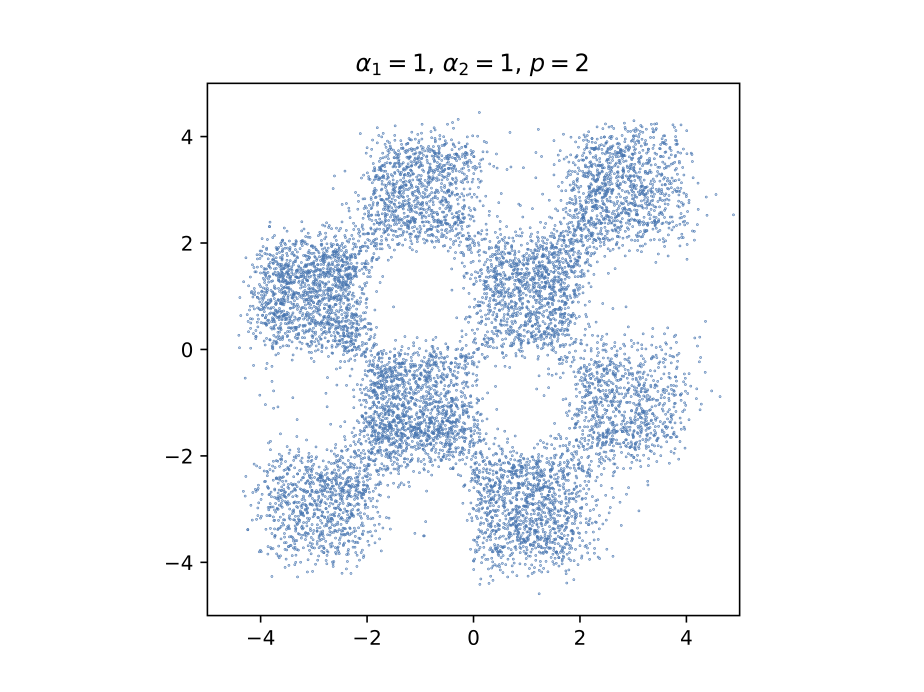}
\caption{\scriptsize HJB  eq. $+$ terminal, $p = 2$}\label{fig:e}
    \end{subfigure}
    \begin{subfigure}{0.24\textwidth}
    \centering
    \includegraphics[width = \textwidth,trim = 180 30 180 30,clip]{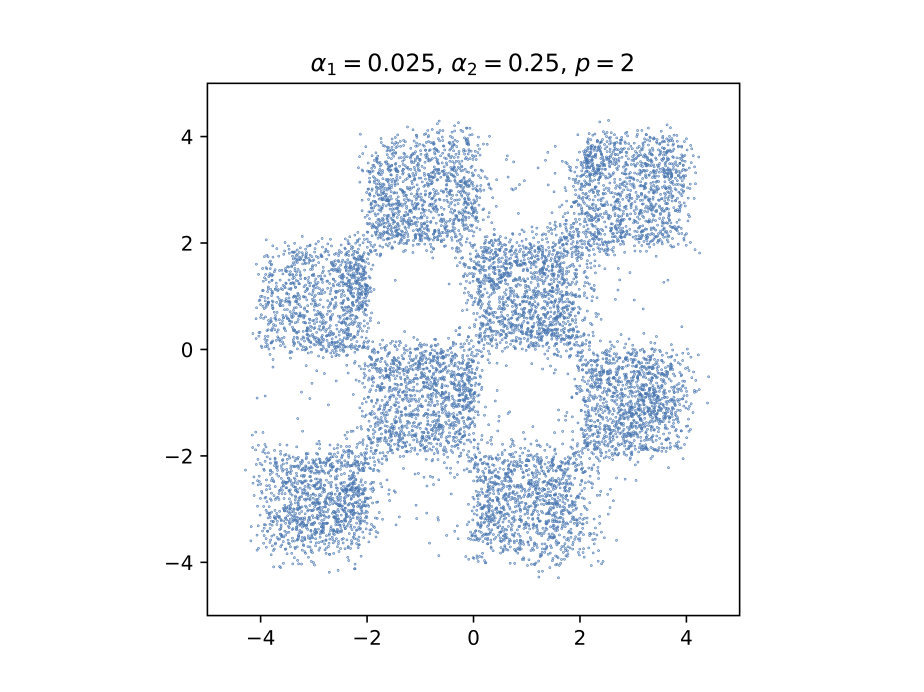} 
    \caption{\scriptsize HJB eq. $+$ terminal, $p = 2$, smaller $\alpha_i$}\label{fig:f}
    \end{subfigure}
    \caption{Effects of different choices of HJB regularization on score-based generative modeling. In these examples, $\alpha_0 = 1$ in \eqref{eq:regscorematching}.}
    \label{fig:checker}
\end{figure}

We can clearly see that the HJB regularizer has a nontrivial effect on the generative properties of the score-based generative modeling. Cases where the terminal constraint is enforced such as Figures \ref{fig:c} and \ref{fig:d}, yields sharper boundaries. However, they also seem to exhibit Gibbs phenomenon \cite{strauss2007partial}, in which more samples lie on boundaries of the checkerboard density. In Figure \ref{fig:f}, we see that reducing the constant $\alpha_2$ for the terminal constraint can help make the boundaries sharper, without exhibiting as much Gibbs phenomenon. Comparing Figures \ref{fig:d} and \ref{fig:e}, we find different behavior when choosing different values of $p$. Figure \ref{fig:f} shows that the parameters can be tuned to produce an improved generative model. On the other hand, if only the HJB equation is enforced as presented in \cite{lai2023fp}, then the regularizers may produce worse results than if no regularizer was present. This is likely due the fact that without the terminal constraint, the function $\mathcal{s}(x,t) = 0$ minimizes the regularizer. Therefore, the use of the HJB regularizer alone skews the learned score function towards zero. 

These simple examples clearly show that HJB regularization has the potential to provide improvements to score-based methods. They also demonstrate that a more extensive numerical and algorithmic investigation on the effects of HJB regularization strategies in MFG-based generative modeling are warranted for future research.

\begin{remark}
    We note that  we are not restricted to use the distribution $\eta(y,s)$ when evaluating the HJB regularizer \eqref{eq:hjbregularizer2}. Using $\eta(y,s)$ as the measure the HJB equations are integrated with respect to is a natural choice since it is already used in the implicit score-matching objective. The choice of measure here is, however, arbitrary; for example, \cite{lu2021deepxde} chooses a uniform distribution over space and time. The effects of the choice of measure can also be another direction for future research. 
\end{remark}

\subsection{Score-matching via PINNs}
The HJB regularizer provides an alternate way for implementing score-based generative modeling. Since the terminal conditions \eqref{eq:sgmhjb} can be imposed directly, we can recover results similar to standard score-based generative modeling by \emph{only} minimizing the HJB regularizer, i.e., $\alpha_0 = 0$ in \eqref{eq:regscorematching}. In Section~\ref{sec:SGM}, we established that score-based generative modeling is based on solving a pair of partial differential equations. We can therefore obtain approximate solutions of the optimality conditions by minimizing the regularizers directly, which we can then use with the reverse SDE \eqref{eq:SGM:bwd_sde} to generate new samples. Minimizing the HJB regularizer directly is the same task as training a physics-informed neural network (PINN) \cite{raissi2019physics,lu2021deepxde}.

In Figure~\ref{fig:sgmpinns} we show results of score-based generative modeling by only minimizer the HJB regularizer. We see different performance depending on the choice of parameter values. Comparing Figure \ref{fig:2a} with \ref{fig:2b}, we see that when the terminal constraint has a larger constant than the HJB equation, the generated distribution has sharper edges. When $p = 1$, we see many particles end up on the edges, which is indicative of a Gibbs phenomenon. We hope the results in this section may spawn further research regarding the use of HJB regularizers for generative modeling. 
\begin{figure}[H]
\centering
\captionsetup[subfigure]{justification=centering}
    \begin{subfigure}{0.25\textwidth}
    \centering
    \includegraphics[width = \textwidth,trim = 180 30 180 30,clip]{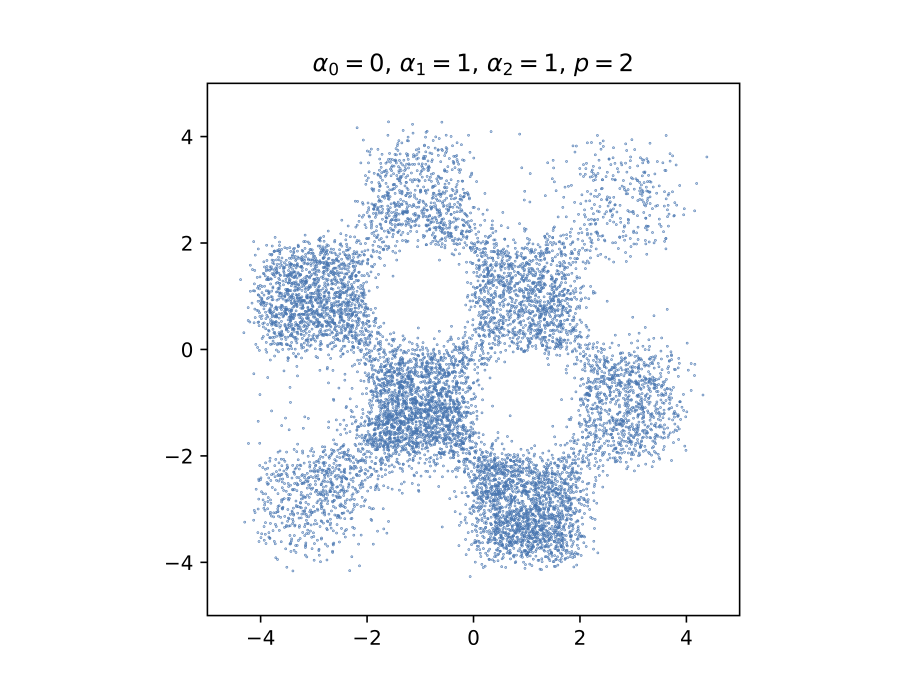}
    \caption{\scriptsize HJB regularizer only with \\   $p = 2,\,\alpha_1 = 1$ }\label{fig:2a}
    \end{subfigure}
        \begin{subfigure}{0.25\textwidth}
    \centering
    \includegraphics[width = \textwidth,trim = 180 30 180 30,clip]{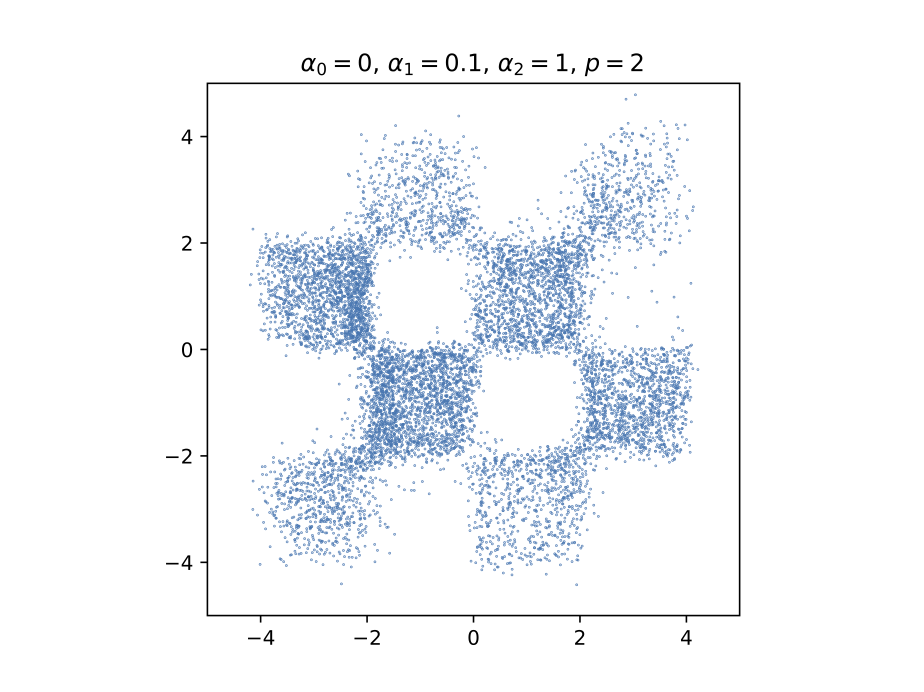}
    \caption{ \scriptsize HJB regularizer only with \\ $p = 2,\, \alpha_1 = 0.1$} \label{fig:2b}
    \end{subfigure}
    \begin{subfigure}{0.25\textwidth}
    \centering
    \includegraphics[width = \textwidth,trim = 180 30 180 30,clip]{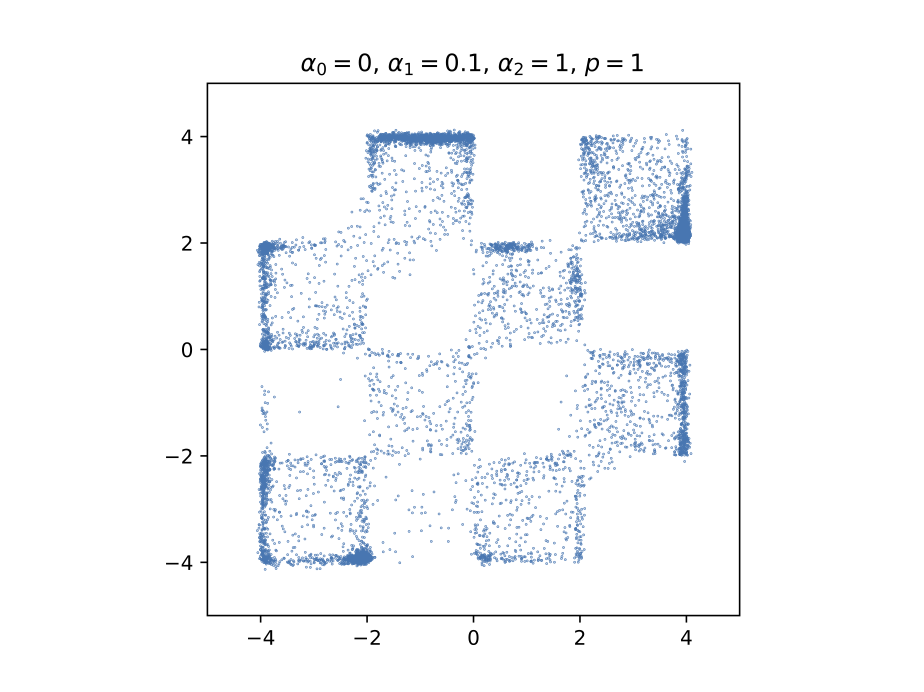}
    \caption{\scriptsize HJB regularizer only with \\   $p = 1$}\label{fig:2c}
    \end{subfigure}
    \caption{Score-based generative modeling via PINNs. Here $\alpha_0 = 0$ in \eqref{eq:regscorematching}. Compare these figures with standard SGM in Figure~\ref{fig:a}.}
    \label{fig:sgmpinns}
\end{figure}

\section{Inventing generative models using MFGs}
\label{sec:designing}
We provide some examples of how the MFG framework can easily lead to new variants of generative models, as well as  enhancements of existing ones in a coherent and motivated fashion. The richness of potential generative modeling formulations via the MFG framework is partly captured in Table~\ref{tab:big2} but it is primarily a topic for future research. The  examples we discuss here are by no means exhaustive. We once again emphasize that while one can conjure the following generative models with Table~\ref{tab:big2}, further numerical and algorithmic investigation is needed to evaluate their practical utility.

\subsection{Optimal transport BG: a corrected Boltzmann generator}
\label{sec:OTBG}
As we discussed in Section \ref{subsec:CNF:Boltzmann}, Boltzmann generators build generative models using normalizing flows that help when sampling from probability distributions that arise in computational chemistry \cite{noe2019boltzmann}. The normalizing flow formulation in \eqref{eq:boltzmanngen} is, however, ill-posed without additional constraints or structure as described in Section \ref{sec:wellposed}. One (of many) possible fixes is to introduce a quadratic running cost like in the OT-flow in Section \ref{subsec:CNF:OT}. We introduce the optimal transport Boltzmann generator MFG,
\begin{align}
    &\inf_{v,\rho} \left\{ \lambda\mathcal{D}_{KL}(\pi\| \rho(\cdot,T)) +(1-\lambda)\mathcal{D}_{KL}(\rho(\cdot,T)\|\pi)+ \int_0^T \int_{\R^d} \frac{1}{2}|v(x,t)|^2 \rho(x,t)\de x \de t \right\} \\
    &\frac{\partial \rho}{\partial t} + \nabla \cdot(v\rho) = 0,\, \rho(x,0) = \rho_0(x) \nonumber
    \end{align}
for some $\lambda \in [0,1]$. In this setting, the Hamiltonian will be quadratic, so the optimal velocity field will be the gradient of a scalar function. Therefore, like in the OT-flow, we may optimize for a single scalar function, rather than for each component of a vector-valued velocity function. Moreover, in \cite{kohler2020equivariant}, the normalizing flow velocity field was parametrized as the gradient of a scalar function, even though the optimization problem did not include a quadratic running cost. Adding such a running cost will mathematically justify the approach taken in \cite{kohler2020equivariant}, and as a result, likely facilitate the algorithm to learn  a normalizing  flow with \emph{gradient} structure. Furthermore, like in Section \ref{sec:SGMHJBreg}, we may introduce an HJB regularizer to enhance training of the Boltzmann generator. 

\subsection{Stochastic normalizing flows via MFGs}
\label{sec:snfmfg}
Since the introduction of continuous normalizing flows based on deterministic dynamics, there have been several attempts to extend it to stochastic differential equations \cite{wu2020stochastic,hodgkinson2020stochastic,zhang2021diffusion}. While each devises their own computational approaches to and implementation of a stochastic normalizing flow, it is unclear if their mathematical formulations are well-posed. With the MFG formulation, it is simple to devise well-posed stochastic normalizing flows. For example, consider the MFG
\begin{align}\label{eq:snfmfg}
&\inf_{v,\rho} \left\{ \mathcal{D}_{KL}(\pi\| \rho(\cdot,T)) + \int_0^T \int_{\R^d} \frac{1}{2}|v(x,t)|^2 \rho(x,t)\de x \de t \right\} \\
&\frac{\partial \rho}{\partial t} + \nabla \cdot(v\rho) = \frac{\sigma^2}{2}\Delta \rho,\, \rho(x,0) = \rho_0(x). \nonumber
\end{align}
We have already studied the well-posedness of this mean-field game: it has the same optimality conditions as score-based generative models \eqref{eq:SGM:MFG_optimality}, except that the terminal condition is based on the KL divergence instead of cross-entropy, i.e.,
\begin{align}
    U(x,T) = -\frac{\pi(x)}{\rho(x,T)}.
\end{align}
One can then devise computational schemes for solving this MFG and be assured that there exist consistent discretizations since the optimization problem is well-defined. To further belabor the point about the HJB regularizer, we may, again consider using the HJB regularizer for this newly introduced generative model as well.

Moreover, as we have discussed in Remark \ref{remark:sgmnf}, score-based generative models can be considered to be a stochastic normalizing flow. In particular, if the terminal cost $\mathcal{D}_{KL}(\pi\|\rho)$ is replaced with the cross-entropy $-\Ex_\rho[\log\pi]$,  \eqref{eq:snfmfg} is identical to the SGM MFG \eqref{eq:sgmmfg} for $f\coloneqq 0$, and $\rho_0(x) = \eta(x,T)$ defined by \eqref{eq:hjbfp}.

\subsection{SGM and the cross-entropy regularization of optimal transport}
\label{sec:sgmOT}
The Schr\"odinger bridge problem in Section \ref{sec:Schrodinger} is often studied as an entropically regularized version of the optimal transport problem \cite{chen2021stochastic}. In fact, if we replace the stochastic dynamics in \eqref{eq:sbpopt} by the Fokker-Planck equation, we obtain a potential formulation \eqref{eq:potentialform} of the SBP. Let $\Omega \subset \R^d$ be compact, and consider the MFG
\begin{align}
    \inf_{v,\rho} \left\{ \int_0^T \int_\Omega |v(x,t)|^2 \rho(x,t) \de x \de t: \frac{\partial \rho}{\partial t} + \nabla \cdot(v\rho) = \frac{\sigma^2}{2}\Delta \rho, \, \rho(x,0) = \rho_0(x), \, \rho(x,T) = \pi(x) \right\}.
\end{align}
When $\sigma = 0$, this MFG is the Benamou-Brenier formula, which is the dynamic formulation of the optimal transport problem \cite{villani2009optimal,santambrogio2015optimal}. This is the problem of transporting a distribution $\rho_0$ to $\pi$ with minimal cost \cite{santambrogio2015optimal}. When $\sigma>0$, it is a regularization of the optimal transport problem \cite{santambrogio2015optimal,chen2021stochastic}. 

Since the terminal constraint is difficult to enforce pointwise, it is common to relax it by introducing a penalty term \cite{genevay2017gan}. If we replace the terminal constraint with a cross-entropy loss function $-\Ex_\rho[\log \pi]$, we can interpret the resulting MFG as a \emph{cross-entropy regularized} entropic dynamic optimal transport problem:
\begin{align}
        \inf_{v,\rho} \left\{-\int_\Omega \rho(x,T) \log \pi(x) \de x +   \int_0^T \int_\Omega |v(x,t)|^2 \de x \de t: \frac{\partial \rho}{\partial t} + \nabla \cdot(v\rho) = \frac{\sigma^2}{2}\Delta \rho, \, \rho(x,0) = \rho_0(x) \right\}.
\end{align}
Notice, however, this is exactly equivalent to the SGM MFG \eqref{eq:sgmmfg} in Theorem~\ref{thm:scoremfg} for $f\coloneqq 0$. With this observation, we may approximate solutions to a certain class of optimal transport problems with computational tools of score-based generative modeling, which have proved to be extremely efficient in high-dimensions.

\subsection{Interpolating between MFGs: from Wasserstein gradient flows to normalizing flows and Wasserstein geodesics}
\label{subsec:geograd}
    We refer back to the connection between Wasserstein gradient flows and mean-field games described by Theorem~\ref{thm:MFG:Wasserstein_grad_flow}. In Section~\ref{sec:wassersteingradient}, we studied the MFG \eqref{eq:wassersteinmfg} in the asymptotic regime where $\epsilon \to 0$. Notice, however, there is another interesting regime when $\epsilon \to \infty$, in which the MFG \eqref{eq:wassersteinmfg} converges to a regularized optimal transport problem: 
    \begin{align}\label{eq:epsinf}
        \min_{v,\rho} \left\{\mathcal{F}(\rho(\cdot, T)) + \int_0^T\int_\Omega \frac{1}{2}|v(x,t)|^2 \rho(x,t) \de x \de t \right\}\quad  
        s.t. \quad  \frac{\partial \rho}{\partial t} + \nabla \cdot(v\rho) = 0, \, \rho(x,0) = \rho_0(x).
    \end{align}
    Since the functional $\mathcal{F}(\rho)$ is typically some divergence from $\rho$ to target $\pi$, for example, $\mathcal{F}(\rho) = \mathcal{D}_{KL}(\rho\|\pi)$, this problem is a generalization of the OT normalizing flow in Section~\ref{subsec:CNF:OT}. Recall, however, that \eqref{eq:epsinf} is the limiting case of \eqref{eq:wassersteinmfg} where $\epsilon \to \infty$, which implies that for $\epsilon \in (0,\infty)$ we have MFGs that \emph{interpolate} the Wasserstein gradient flow, and the OT-flow.  Following the same steps as in Section \ref{subsec:CNF:OT}, the optimality conditions of the MFG \eqref{eq:epsinf} are 
\begin{align}
    \begin{dcases}
        &-\frac{\partial U}{\partial t} + \frac{1}{2}| \nabla U|^2 = 0 \\
    &\frac{\partial \rho}{\partial t} - \nabla \cdot \left(\rho \nabla U \right) = 0 \\
    &U(x,T) = \frac{\delta \mathcal{F}}{\delta \rho}(\rho(\cdot,T)), \, \rho(x,0) = \rho_0(x). 
        \end{dcases} 
        \label{eq:geograd}
\end{align}

    Moreover, recall from Section~\ref{sec:sgmOT} that $\mathcal{F}(\rho)$ in the terminal cost can be interpreted as a relaxation of a terminal constraint $\rho(x,T) = \pi(x)$. Therefore, \eqref{eq:epsinf} is a relaxation of the Benamou-Brenier dynamic formulation of optimal transport. Since solutions of the optimal transport problem are geodesics connecting $\rho_0$ and $\pi$ in Wasserstein space \cite{villani2009optimal,santambrogio2015optimal}, solutions of the MFG \eqref{eq:wassersteinmfg} for $\epsilon \in (0,\infty)$ can be interpreted as interpolating between Wasserstein gradient flows ($\epsilon \to 0$) and Wasserstein geodesics ($\epsilon \to \infty$).

\begin{landscape}
\section{A modular mean-field games framework for generative modeling } \vspace{-15pt}
\label{sec:table}
\begin{table}[H]
\ssmall
\def\arraystretch{2.2}
\begin{tabular}{|c||c|c|c|c||c|c|c||c|}
\hline 
 & \multicolumn{4}{c||}{\textbf{Mean-field game} \eqref{eq:potentialform}} & & & & \\
\hline
\textbf{Model} & $\mathcal{M}(\rho)$ & $\mathcal{I}(\rho)$ & $L(x,v)$ & \textbf{Dynamics} & $H(x,p)$ & \textbf{Optimal} $v^*$& \makecell{\textbf{HJB}\\ \textbf{Reg?}} &\textbf{ Refs} \\ \hline \hline 
\makecell{Continuous \\ normalizing flow} &$ \mathcal{D}_{KL}(\rho\|\rho_{ref})$ & 0 & 0& $\de x = v \de t$& $\sup_{v\in K} - p^\top v $& $-\nabla_p H(x,\nabla U) $& $\Large{\newcrossmark}$ & 1
\\ \hline

\makecell{Score-based \\generative modeling} & $-\Ex_\rho\left[\log \pi \right]$ & 0 & $\frac{|v|^2}{2} - \nabla \cdot f$& \makecell{$\de x = (f + \sigma v )\de t$\\$ + \sigma \de W_t$}& \makecell{ $-f \cdot p + \frac{\sigma^2}{2}|p|^2  $\\$+\nabla \cdot f$} & $\sigma \nabla \log \eta_{T-t}$ & $\Large{\newcrossmark}$& 2 %\makecell{\cite{song2021score,de2021diffusion} \&\\ Sec.~\ref{sec:SGM} }
\\ \hline

\makecell{Score-based \\ probability flow} & $-\frac{1}{2}\Ex_\rho\left[\log \pi \right]$ & $\Ex_\rho\left[ \frac{|\sigma \nabla \log \rho|^2 }{8} \right]$ & $\frac{|v|^2}{2} - \frac{\nabla \cdot f}{2}$ & $dx = (f + \sigma v )\, dt$& \makecell{ $-f \cdot p + \frac{\sigma^2}{2}|p|^2  $\\$+\nabla \cdot f$}  & $\frac{\sigma}{2} \nabla \log \eta_{T-t}$ & $\Large{\newcrossmark}$ &3 % \makecell{\cite{maoutsa2020interacting,song2021score} \\ Sec.~\ref{sec:SGM:prob_flow}} 
\\ \hline

\makecell{Wasserstein gradient \\ flow (WGF) $(\epsilon \to 0)$}& $\mathcal{F}(\rho) e^{-T/\epsilon}$& $\frac{e^{-t/\epsilon}}{\epsilon}\mathcal{F}(\rho)$ & $\frac{e^{-t/\epsilon}}{2}|v|^2$& $\de x = v \de t$ & $\frac{1}{2}e^{t/\epsilon}|p|^2$ &$ - \nabla \frac{\delta \mathcal{F}}{\delta \rho}$ &$\Large{\newcrossmark}$ & 4 %\makecell{\cite{santambrogio2015optimal,rossi2011variational,arbel2019maximum,gu2022lipschitz} \& \\Sec.~\ref{sec:wassersteinmfg}}
\\ \hline

OT-Flow & $\mathcal{D}_{KL}(\rho\|\rho_{ref})$ & 0 &$\frac{1}{2}|v|^2$ & $\de x = v \de t$ &$\frac{1}{2}|p|^2$ & $-\nabla U$ & $\Large{\newcheckmark}$& 5 %\makecell{\cite{onken2021ot,huang2022bridging} \& \\ Sec.~\ref{subsec:CNF:OT} }
\\ \hline

 Boltzmann generator & \makecell{$\lambda\mathcal{D}_{KL}(\pi\|\rho)$\\$ + (1-\lambda)\mathcal{D}_{KL}(\rho\|\pi)$ }&0 & 0& $\de x = v \de t$& $\sup_{v\in K} - p^\top v $ & $-\nabla_p H(x,\nabla U)$&$\Large{\newcrossmark}$ & 6%\makecell{\cite{noe2019boltzmann,kohler2020equivariant} \& \\Sec.~\ref{subsec:CNF:Boltzmann} }
 \\ \hline
 
Schr\"odinger bridge & $\rho = \pi$ &0 &$\frac{1}{2}|v|^2$ & $\de x = \sigma v \de t + \sigma \de W_t$ & $\frac{1}{2}|p|^2$ &$-\sigma \nabla U$  & $\Large{\newcrossmark}$& 7%\makecell{\cite{pavon2021data,chen2021stochastic} \& \\ Sec.~\ref{sec:Schrodinger} }
\\ \hline

\makecell{Generalized \\Schr\"odinger  bridge} & $\rho = \pi$ & $\mathcal{I}(x,\rho)$ & $\frac{1}{2}|v|^2$ & \makecell{$\de x = \sigma v \de t + \sigma \de W_t$} & $\frac{1}{2}|p|^2$ & $-\sigma \nabla U$& $\Large{\newcrossmark}$&  8 % \makecell{\cite{liudeep} \& \\Sec.~\ref{sec:Schrodinger}}
\\ \hline 
\hline
\makecell{HJB-regularized \\SGM} & $-\Ex_\rho[\log\pi]$ & 0& $\frac{|v|^2}{2} - \nabla \cdot f$& \makecell{$\de x = (f + \sigma v )\de t$\\$ + \sigma \de W_t$} &\makecell{ $-f \cdot p + \frac{\sigma^2}{2}|p|^2$\\$ + \nabla \cdot f$} & $\sigma\nabla \log \eta_{T-t}$ &$\Large{\newcheckmark}$ & 9% \makecell{New \& \\Sec.~\ref{sec:SGMHJBreg}, Sec.~\ref{sec:hjbsgmnum} }
\\ \hline

\makecell{Stochastic OT\\normalizing flow} & $\mathcal{D}_{KL}(\pi\|\rho)$& 0 &$\frac{1}{2}|v|^2$ & \makecell{$\de x = \sigma v \de t + \sigma \de W_t$} & $\frac{1}{2}|p|^2$ &$-\sigma\nabla U$ & $ \Large{\left(\newcheckmark\right)}$ & 10% \makecell{ New \& \\ Sec.~\ref{sec:snfmfg} }
\\ \hline

\makecell{OT-Boltzmann \\generator}  & \makecell{$\lambda\mathcal{D}_{KL}(\pi\|\rho) $\\$+ (1-\lambda)\mathcal{D}_{KL}(\rho\|\pi)$} & 0 &$\frac{1}{2}|v|^2$ & {$\de x =  v\de t$} & $\frac{1}{2}|p|^2$ & $-\nabla U$ &$\Large{\left(\newcheckmark\right)}$ & 11% \makecell{New \&\\ Sec.~\ref{sec:OTBG}}
\\ \hline  

\makecell{Generalized OT-Flow \\ (Relaxed WGF $\epsilon>0$)} & $\mathcal{F}(\rho) e^{-T/\epsilon}$& $\frac{e^{-t/\epsilon}}{\epsilon}\mathcal{F}(\rho)$ & $\frac{e^{-t/\epsilon}}{2}|v|^2$& $\de x = v \de t$ & $\frac{1}{2}e^{t/\epsilon}|p|^2$ &$ - \nabla U$ &$\Large{\left(\newcheckmark\right)}$ &  12%\makecell{New \&\\ Sec.~\ref{subsec:geograd}} 
\\ \hline  

\makecell{Build your own \\ generative model} & \makecell{Choose\\your} & \makecell{own\\ cost} & \makecell{functions \\and} & \makecell{dynamics \\ here} & ? & ? & ? & ? %\makecell{Your \\contribution}
\\ \hline
\end{tabular}
\centering
\caption{\scriptsize Summary of %\MK{both established and new} 
generative modeling and related topics as mean-field games.  \\ 1. \cite{chen2018neural,grathwohl2018ffjord} \&  Sec.~\ref{sec:wellposed} \\ 2. \cite{song2021score,de2021diffusion} \& Sec.~\ref{sec:SGM} \\ 3. \cite{maoutsa2020interacting,song2021score} Sec.~\ref{sec:SGM:prob_flow} \\ 4. \cite{santambrogio2015optimal,rossi2011variational,arbel2019maximum,gu2022lipschitz} \& Sec.~\ref{sec:wassersteinmfg} \\ 5. \cite{onken2021ot,huang2022bridging} \& Sec.~\ref{subsec:CNF:OT} \\ 6. \cite{noe2019boltzmann,kohler2020equivariant} \& Sec.~\ref{subsec:CNF:Boltzmann}  \\ 7. \cite{pavon2021data,chen2021stochastic} \&  Sec.~\ref{sec:Schrodinger}\\ 8. \cite{liudeep} \& Sec.~\ref{sec:Schrodinger} \\ 9. New \& Sec.~\ref{sec:SGMHJBreg}, Sec.~\ref{sec:hjbsgmnum}, 10. New \&  Sec.~\ref{sec:snfmfg} \\ 11. New \& Sec.~\ref{sec:OTBG},  12. New \& Sec.~\ref{subsec:geograd}}
\label{tab:big2} 
\end{table} 
\end{landscape}

\section{Discussion and outlook} %\todo{Rewrite with reviews in mind}

In this paper, we showed that mean-field games provide a natural framework for studying continuous-time generative models. Given samples from a target distribution, continuous-time generative models aim to approximate the target by building a dynamical system that evolves samples from a reference distribution to those of the target. Depending on the choice of the continuous-time model, the dynamics of each particle may depend on the evolution of other particles or on the final state of the density of particles. The mean-field games framework can be readily used as an organizational principle for understanding the generative modeling task. 

We showed how normalizing flows, score-based generative models, and Wasserstein gradient flows can be derived from different MFGs. Moreover, we demonstrated how MFGs can explain and enhance generative models by revealing and exploiting their mathematical structure, yielding new computational tools, and provide a system for inventing new generative models. We summarize our main contributions here: 

\begin{itemize}
    \item \textbf{Explaining and categorizing generative models:} We showed that the potential formulation of MFG in Section \ref{sec:potentialmfg} describes major classes of generative models. We showed how with different choices cost function, mean-field games correspond to different generative models. Each generative model can be defined by four ingredients: the terminal, interaction, and running cost functions, and the form of the dynamical system. \edits{The MFG optimality conditions uncover the common mathematical structure of all flow and diffusion-based generative models. While the Fokker-Planck equation is clearly the generative portion of the model, the Hamilton-Jacobi-Bellman equation can be interpreted as the `discriminative' part, which determines the optimal velocity field of the generative model. The optimality conditions describe the mathematical structure that identifies the inductive bias of generative flows.} Overall, the MFG perspective yields a simple way of organizing the vast collection of flow and diffusion-based generative models. We exhibit some of these methods in Table~\ref{tab:big2}. The table also provides a way to understand generative models in relation to each other. 
    
    %In particular, in Section \ref{sec:SGM} we provided a self-contained derivation of score-based generative modeling in which the forward-reverse SDE system appears naturally as a result of the MFG optimality conditions. Moreover, we demonstrated that the score-matching objective can be derived from the MFG objective function. 

    Furthermore, the MFG formulation provides insight into the \emph{well-posedness} of a generative model. In particular, well-posedness of a generative model can be studied in terms of the well-posedness of the optimality conditions, which are a set of partial differential equations. In Section \ref{sec:wellposed} we used the optimality conditions of MFGs to show that the canonical formulation of normalizing flows is ill-posed because the Hamiltonian, and therefore the associated HJB equation, is not well-defined without additional constraints to the problem. The MFG formulation then provides various options to make the problem well-posed. In Section \ref{sec:wassersteinmfg}, we showed that the finite-time relaxations of the Wasserstein gradient flow require the terminal and interaction costs to be related so that the flow avoids boundary layers at terminal time.

    \item \textbf{Enhancing generative models}: Studying generative models within the MFG framework provides access to optimality conditions which describe the solution to the generative modeling problem. These HJB regularizers, while first developed for optimal transport flows, %(see Section~\ref{subsec:CNF:OT})
    can be applied to any generative model described by the MFG framework presented here.  In Sections %\ref{sec:hjbregularizer}
    \ref{sec:SGMHJBreg} and \ref{sec:hjbsgmnum} we demonstrated how these HJB regularizers may be applied to improve score-based generative modeling. The HJB regularizers provides one approach for enforcing the inductive bias of generative models, which may accelerate their training. 

    Moreover, the MFG perspective provides ideas for better relaxations of problems in optimal transport. For example, in Section~\ref{sec:sgmOT} we remarked that the entropy-regularized Benamou-Brenier dynamic formulation of optimal transport, where the terminal constraint is substituted with a terminal cross-entropy cost function, is equivalent to score-based generative modeling. Thus, the MFG perspective provides insight into how to choose relaxations that are more easily solved via machine learning techniques. 

    \item \textbf{Inventing generative models}: The MFG framework provides a systematic way for inventing new generative models. In Section~\ref{sec:designing} we proposed three new generative models that either extend existing approaches (Sections~\ref{sec:SGMHJBreg} and \ref{sec:snfmfg}) or amend ill-posed approaches (Section \ref{sec:OTBG}).      
\end{itemize}

Our contributions here have only discussed a fraction of the potential insights of each of the above three categories. The main goal of this paper has been to illustrate the richness of studying generative models through the mean-field game perspective. For future work, we provide a decidedly non-exhaustive list of possible future research directions. {Table~\ref{tab:big2}, however, beckons for further future investigations and contributions to generative modeling and mean field games, well-beyond what is discussed in this paper.}

\begin{itemize}
    % \item \textbf{Further analysis of generative models}:
    
     \item \edits{\textbf{The role of stochasticity:} We can study further the well-posedness of generative models by using the theory of MFGs and partial differential equations. These aspects are important for machine learning applications as they may lead to consistent discretizations of generative models, and  more stable algorithms for training such models. MFG optimality conditions may provide insight into the role of stochasticity in generative models. For  stochastic dynamics  (e.g.,  diffusion-based methods), gradients of the HJB  are always smooth. Therefore learning the optimal velocity is always well-defined. This  is not  the case for deterministic dynamics such as normalizing flows where gradients of the   corresponding first-order HJB and hence optimal velocities may have jumps, creating ambiguities in the optimal velocities that the generative algorithm will need to somehow resolve. 
    In this sense, \textit{stochastic} generative algorithms  have theoretical and possibly practical advantages, that need to be explored further using PDE and MFG regularity theory.}

    \item \edits{\textbf{Finite sample analysis}: Mean-field games are asymptotic dynamics of a large population of interacting agents. Studying MFGs in the nonasymptotic regime may provide insight into improving the performance of generative models. In particular, similar questions have been studied in seminal papers in mean-field games by \cite{huang2006large,lasry2007mean}. Since then, there has been extensive literature in MFGs for discrete state spaces and approximations of MFGs.} %All these questions and insights are relevant in the finite sample regime for generative modeling which the referee mentions. }
    %Discrete-time normalizing flows and denoising diffusion models \cite{ho2020denoising} are often related to continuous counterparts by arguing the former is a discretization of the latter \cite{song2021score}. Studying and deriving discrete-time generative models in the context of discrete-time mean-field games may provide alternative ways of understanding well-posedness of learning discrete-time flows \cite{saldi2020approximate}. 

    \item \textbf{Finding and exploiting structure}: As we have seen for normalizing flows and its variants in Sections~\ref{sec:wellposed}, the resulting structure of the MFG and its associated PDEs may inform neutral network parametrizations. In this paper we highlighted generative models that are the gradient of a scalar function. Furthermore,  \cite{ruthotto2020machine} encodes linear and quadratic functions into the approximating class of neural networks  which arises from the knowledge of the underlying MFG's structure. 

    \item \textbf{Creating even more generative models}: Table~\ref{tab:big2} provides a systematic way of creating new generative models by simply playing with different cost functions and dynamics. It also shows us capabilities of MFGs that are under-explored in generative modeling. For example, we see that the interaction term $\mathcal{I}(\rho)$ is not often used in generative modeling. We can also choose dynamical systems beyond ODEs and SDEs. For example, if we formulate generative modeling via discrete-time mean-field games, the dynamics can be jump processes and nonlocal diffusions \cite{liu2021computational}.

    One particularly interesting class of models that merit further computational investigation are relaxations of Wasserstein gradient flows we introduced in Section \ref{sec:wassersteinrelax}. As we mentioned in Section~\ref{subsec:geograd}, solutions to the Wasserstein MFG for $\epsilon\in(0,\infty)$ seem to be interpolations between Wasserstein geodesics and Wasserstein gradient flows. The properties and value of computing these curves could be a rich area of investigation.  

    \item \textbf{ML tools for MFGs}: While adapting ML tools, such as normalizing flows, to solve MFGs has been considered previously \cite{ruthotto2020machine,huang2022bridging}, our MFGs for generative modeling framework provides multiple bridges between the two fields. That is, any ML method devised to perform generative modeling based on the MFG framework can be adapted to solve mean-field games instead. For example, one idea is to adapt score-based generative modeling for solving some class of MFGs and HJB equations. Training generative models may result in new classes of computationally tractable mean-field games. 

\end{itemize}
Finally, we mention, once again, that while the MFG laboratory may lead to many new generative models, it is not a formula that yields useful algorithms on its own. The MFG formulation establishes generative models with well-posed optimization properties and is a first,  but necessary step,  towards a systematic approach for building novel generative models. However, numerical and algorithmic aspects must be subsequently developed for the practical implementation of new generative models.

% \vskip 0.5in
\newpage

\acks{ The research of B.Z. and  M.K.  was partially supported by the Air Force Office of Scientific Research (AFOSR) under the grant FA9550-21-1-0354. The research of M.K. was partially supported by the National Science Foundation (NSF) under the grant TRIPODS CISE-1934846. }


% \newpage

% \appendix
% \section{}
% \label{app:theorem}

% % Note: in this sample, the section number is hard-coded in. Following
% % proper LaTeX conventions, it should properly be coded as a reference:

% %In this appendix we prove the following theorem from
% %Section~\ref{sec:textree-generalization}:

% In this appendix we prove the following theorem from
% Section~6.2:

% \noindent
% {\bf Theorem} {\it Let $u,v,w$ be discrete variables such that $v, w$ do
% not co-occur with $u$ (i.e., $u\neq0\;\Rightarrow \;v=w=0$ in a given
% dataset $\dataset$). Let $N_{v0},N_{w0}$ be the number of data points for
% which $v=0, w=0$ respectively, and let $I_{uv},I_{uw}$ be the
% respective empirical mutual information values based on the sample
% $\dataset$. Then
% \[
% 	N_{v0} \;>\; N_{w0}\;\;\Rightarrow\;\;I_{uv} \;\leq\;I_{uw}
% \]
% with equality only if $u$ is identically 0.} \hfill\BlackBox

% \section{}

% \noindent
% {\bf Proof}. We use the notation:
% \[
% P_v(i) \;=\;\frac{N_v^i}{N},\;\;\;i \neq 0;\;\;\;
% P_{v0}\;\equiv\;P_v(0)\; = \;1 - \sum_{i\neq 0}P_v(i).
% \]
% These values represent the (empirical) probabilities of $v$
% taking value $i\neq 0$ and 0 respectively.  Entropies will be denoted
% by $H$. We aim to show that $\fracpartial{I_{uv}}{P_{v0}} < 0$....\\

% {\noindent \em Remainder omitted in this sample. See http://www.jmlr.org/papers/ for full paper.}

\vskip 0.2in
\bibliography{sample}

\begin{thebibliography}{80}
\providecommand{\natexlab}[1]{#1}
\providecommand{\url}[1]{\texttt{#1}}
\expandafter\ifx\csname urlstyle\endcsname\relax
  \providecommand{\doi}[1]{doi: #1}\else
  \providecommand{\doi}{doi: \begingroup \urlstyle{rm}\Url}\fi

\bibitem[Albergo and Vanden-Eijnden(2022)]{albergo2022building}
Michael~S Albergo and Eric Vanden-Eijnden.
\newblock Building normalizing flows with stochastic interpolants.
\newblock \emph{arXiv preprint arXiv:2209.15571}, 2022.

\bibitem[Albergo et~al.(2023)Albergo, Boffi, and
  Vanden-Eijnden]{albergo2023stochastic}
Michael~S Albergo, Nicholas~M Boffi, and Eric Vanden-Eijnden.
\newblock Stochastic interpolants: A unifying framework for flows and
  diffusions.
\newblock \emph{arXiv preprint arXiv:2303.08797}, 2023.

\bibitem[Anderson(1982)]{anderson1982reverse}
Brian~DO Anderson.
\newblock Reverse-time diffusion equation models.
\newblock \emph{Stochastic Processes and their Applications}, 12\penalty0
  (3):\penalty0 313--326, 1982.

\bibitem[Ansari et~al.(2020)Ansari, Ang, and Soh]{ansari2020refining}
Abdul~Fatir Ansari, Ming~Liang Ang, and Harold Soh.
\newblock Refining deep generative models via discriminator gradient flow.
\newblock \emph{arXiv preprint arXiv:2012.00780}, 2020.

\bibitem[Arbel et~al.(2019)Arbel, Korba, Salim, and Gretton]{arbel2019maximum}
Michael Arbel, Anna Korba, Adil Salim, and Arthur Gretton.
\newblock Maximum mean discrepancy gradient flow.
\newblock \emph{Advances in Neural Information Processing Systems}, 32, 2019.

\bibitem[Bensoussan et~al.(2013)Bensoussan, Frehse, Yam,
  et~al.]{bensoussan2013mean}
Alain Bensoussan, Jens Frehse, Phillip Yam, et~al.
\newblock \emph{Mean field games and mean field type control theory}, volume
  101.
\newblock Springer, 2013.

\bibitem[Berner et~al.(2022)Berner, Richter, and Ullrich]{berner2022optimal}
Julius Berner, Lorenz Richter, and Karen Ullrich.
\newblock An optimal control perspective on diffusion-based generative
  modeling.
\newblock In \emph{NeurIPS 2022 Workshop on Score-Based Methods}, 2022.

\bibitem[Boyd et~al.(2004)Boyd, Boyd, and Vandenberghe]{boyd2004convex}
Stephen Boyd, Stephen~P Boyd, and Lieven Vandenberghe.
\newblock \emph{Convex optimization}.
\newblock Cambridge university press, 2004.

\bibitem[Chen et~al.(2018)Chen, Rubanova, Bettencourt, and
  Duvenaud]{chen2018neural}
Ricky T.~Q. Chen, Yulia Rubanova, Jesse Bettencourt, and David~K Duvenaud.
\newblock Neural ordinary differential equations.
\newblock In S.~Bengio, H.~Wallach, H.~Larochelle, K.~Grauman, N.~Cesa-Bianchi,
  and R.~Garnett, editors, \emph{Advances in Neural Information Processing
  Systems}, volume~31. Curran Associates, Inc., 2018.
\newblock URL
  \url{https://proceedings.neurips.cc/paper/2018/file/69386f6bb1dfed68692a24c8686939b9-Paper.pdf}.

\bibitem[Chen et~al.(2022)Chen, Liu, and Theodorou]{chenlikelihood}
Tianrong Chen, Guan-Horng Liu, and Evangelos Theodorou.
\newblock Likelihood training of {S}chr{\"o}dinger bridge using
  forward-backward sdes theory.
\newblock In \emph{International Conference on Learning Representations}, 2022.

\bibitem[Chen et~al.(2021)Chen, Georgiou, and Pavon]{chen2021stochastic}
Yongxin Chen, Tryphon~T Georgiou, and Michele Pavon.
\newblock Stochastic control liaisons: {R}ichard {S}inkhorn meets {G}aspard
  {M}onge on a {S}chr\"odinger bridge.
\newblock \emph{Siam Review}, 63\penalty0 (2):\penalty0 249--313, 2021.

\bibitem[De~Bortoli et~al.(2021)De~Bortoli, Thornton, Heng, and
  Doucet]{de2021diffusion}
Valentin De~Bortoli, James Thornton, Jeremy Heng, and Arnaud Doucet.
\newblock Diffusion {S}chr{\"o}dinger bridge with applications to score-based
  generative modeling.
\newblock \emph{Advances in Neural Information Processing Systems},
  34:\penalty0 17695--17709, 2021.

\bibitem[Evans(2005)]{evans2005introduction}
Lawrence~C Evans.
\newblock An introduction to mathematical optimal control theory.
\newblock \emph{Lecture Notes, University of California, Department of
  Mathematics, Berkeley}, 3:\penalty0 15--40, 2005.

\bibitem[Evans(1998)]{evans_PDE}
L.C. Evans.
\newblock \emph{{{P}artial {D}ifferential {E}quations}}.
\newblock American Mathematical Society, Providence, RI, 1998.

\bibitem[Fleming and Soner(2006)]{fleming2006controlled}
Wendell~H Fleming and Halil~Mete Soner.
\newblock \emph{Controlled Markov processes and viscosity solutions},
  volume~25.
\newblock Springer Science \& Business Media, 2006.

\bibitem[Garbuno-Inigo et~al.(2020)Garbuno-Inigo, Hoffmann, Li, and
  Stuart]{garbuno2020interacting}
Alfredo Garbuno-Inigo, Franca Hoffmann, Wuchen Li, and Andrew~M Stuart.
\newblock Interacting {L}angevin diffusions: Gradient structure and ensemble
  kalman sampler.
\newblock \emph{SIAM Journal on Applied Dynamical Systems}, 19\penalty0
  (1):\penalty0 412--441, 2020.

\bibitem[Genevay et~al.(2017)Genevay, Peyré, and Cuturi]{genevay2017gan}
Aude Genevay, Gabriel Peyré, and Marco Cuturi.
\newblock {GAN} and {VAE} from an optimal transport point of view, 2017.

\bibitem[Glaser et~al.(2021)Glaser, Arbel, and Gretton]{glaser2021kale}
Pierre Glaser, Michael Arbel, and Arthur Gretton.
\newblock {KALE} flow: A relaxed {KL} gradient flow for probabilities with
  disjoint support.
\newblock \emph{Advances in Neural Information Processing Systems},
  34:\penalty0 8018--8031, 2021.

\bibitem[Goodfellow et~al.(2020)Goodfellow, Pouget-Abadie, Mirza, Xu,
  Warde-Farley, Ozair, Courville, and Bengio]{goodfellow2020generative}
Ian Goodfellow, Jean Pouget-Abadie, Mehdi Mirza, Bing Xu, David Warde-Farley,
  Sherjil Ozair, Aaron Courville, and Yoshua Bengio.
\newblock Generative adversarial networks.
\newblock \emph{Communications of the ACM}, 63\penalty0 (11):\penalty0
  139--144, 2020.

\bibitem[Grathwohl et~al.(2018)Grathwohl, Chen, Bettencourt, Sutskever, and
  Duvenaud]{grathwohl2018ffjord}
Will Grathwohl, Ricky~TQ Chen, Jesse Bettencourt, Ilya Sutskever, and David
  Duvenaud.
\newblock {FFJORD}: Free-form continuous dynamics for scalable reversible
  generative models.
\newblock \emph{arXiv preprint arXiv:1810.01367}, 2018.

\bibitem[Grover et~al.(2018)Grover, Dhar, and Ermon]{grover2018flow}
Aditya Grover, Manik Dhar, and Stefano Ermon.
\newblock Flow-{GAN}: Combining maximum likelihood and adversarial learning in
  generative models.
\newblock In \emph{Proceedings of the AAAI conference on artificial
  intelligence}, volume~32, 2018.

\bibitem[Gu et~al.(2022)Gu, Birmpa, Pantazis, Rey-Bellet, and
  Katsoulakis]{gu2022lipschitz}
Hyemin Gu, Panagiota Birmpa, Yiannis Pantazis, Luc Rey-Bellet, and Markos~A
  Katsoulakis.
\newblock Lipschitz regularized gradient flows and latent generative particles.
\newblock \emph{arXiv preprint arXiv:2210.17230}, 2022.

\bibitem[Gu{\'e}ant(2012)]{gueant2012mean}
Olivier Gu{\'e}ant.
\newblock Mean field games equations with quadratic {H}amiltonian: a specific
  approach.
\newblock \emph{Mathematical Models and Methods in Applied Sciences},
  22\penalty0 (09):\penalty0 1250022, 2012.

\bibitem[Guo et~al.(2022)Guo, Xu, and Zariphopoulou]{guo2022entropy}
Xin Guo, Renyuan Xu, and Thaleia Zariphopoulou.
\newblock Entropy regularization for mean field games with learning.
\newblock \emph{Mathematics of Operations Research}, 47\penalty0 (4):\penalty0
  3239--3260, 2022.

\bibitem[Ho et~al.(2020)Ho, Jain, and Abbeel]{ho2020denoising}
Jonathan Ho, Ajay Jain, and Pieter Abbeel.
\newblock Denoising diffusion probabilistic models.
\newblock \emph{Advances in Neural Information Processing Systems},
  33:\penalty0 6840--6851, 2020.

\bibitem[Hodgkinson et~al.(2020)Hodgkinson, van~der Heide, Roosta, and
  Mahoney]{hodgkinson2020stochastic}
Liam Hodgkinson, Chris van~der Heide, Fred Roosta, and Michael~W Mahoney.
\newblock Stochastic normalizing flows.
\newblock \emph{arXiv preprint arXiv:2002.09547}, 2020.

\bibitem[Huang et~al.(2021)Huang, Lim, and Courville]{huang2021variational}
Chin-Wei Huang, Jae~Hyun Lim, and Aaron~C Courville.
\newblock A variational perspective on diffusion-based generative models and
  score matching.
\newblock \emph{Advances in Neural Information Processing Systems},
  34:\penalty0 22863--22876, 2021.

\bibitem[Huang et~al.(2022)Huang, Yu, Chen, and Lai]{huang2022bridging}
Han Huang, Jiajia Yu, Jie Chen, and Rongjie Lai.
\newblock Bridging mean-field games and normalizing flows with trajectory
  regularization.
\newblock \emph{arXiv preprint arXiv:2206.14990}, 2022.

\bibitem[Huang et~al.(2006)Huang, Malham{\'e}, and Caines]{huang2006large}
Minyi Huang, Roland~P Malham{\'e}, and Peter~E Caines.
\newblock Large population stochastic dynamic games: closed-loop mckean-vlasov
  systems and the nash certainty equivalence principle.
\newblock 2006.

\bibitem[Jin and Xin(1995)]{jin1995relaxation}
Shi Jin and Zhouping Xin.
\newblock The relaxation schemes for systems of conservation laws in arbitrary
  space dimensions.
\newblock \emph{Communications on Pure and Applied Mathematics}, 48\penalty0
  (3):\penalty0 235--276, 1995.

\bibitem[Jordan et~al.(1998)Jordan, Kinderlehrer, and
  Otto]{jordan1998variational}
Richard Jordan, David Kinderlehrer, and Felix Otto.
\newblock The variational formulation of the {F}okker--{P}lanck equation.
\newblock \emph{SIAM Journal on mathematical analysis}, 29\penalty0
  (1):\penalty0 1--17, 1998.

\bibitem[Kobyzev et~al.(2020)Kobyzev, Prince, and
  Brubaker]{kobyzev2020normalizing}
Ivan Kobyzev, Simon~JD Prince, and Marcus~A Brubaker.
\newblock Normalizing flows: An introduction and review of current methods.
\newblock \emph{IEEE transactions on pattern analysis and machine
  intelligence}, 43\penalty0 (11):\penalty0 3964--3979, 2020.

\bibitem[K{\"o}hler et~al.(2020)K{\"o}hler, Klein, and
  No{\'e}]{kohler2020equivariant}
Jonas K{\"o}hler, Leon Klein, and Frank No{\'e}.
\newblock Equivariant flows: exact likelihood generative learning for symmetric
  densities.
\newblock In \emph{International conference on machine learning}, pages
  5361--5370. PMLR, 2020.

\bibitem[Lai et~al.(2023)Lai, Takida, Murata, Uesaka, Mitsufuji, and
  Ermon]{lai2023fp}
Chieh-Hsin Lai, Yuhta Takida, Naoki Murata, Toshimitsu Uesaka, Yuki Mitsufuji,
  and Stefano Ermon.
\newblock Fp-diffusion: Improving score-based diffusion models by enforcing the
  underlying score {F}okker-{P}lanck equation.
\newblock 2023.

\bibitem[Lasry and Lions(2007)]{lasry2007mean}
Jean-Michel Lasry and Pierre-Louis Lions.
\newblock Mean field games.
\newblock \emph{Japanese Journal of Mathematics}, 2\penalty0 (1):\penalty0
  229--260, 2007.

\bibitem[L{\'e}onard(2014)]{leonard2014survey}
Christian L{\'e}onard.
\newblock A survey of the {S}chr{\"o}dinger problem and some of its connections
  with optimal transport.
\newblock \emph{Discrete \& Continuous Dynamical Systems}, 34\penalty0
  (4):\penalty0 1533, 2014.

\bibitem[Liu et~al.(2022)Liu, Chen, So, and Theodorou]{liudeep}
Guan-Horng Liu, Tianrong Chen, Oswin So, and Evangelos Theodorou.
\newblock Deep generalized {S}chr{\"o}dinger bridge.
\newblock In \emph{Advances in Neural Information Processing Systems}, 2022.

\bibitem[Liu and Wang(2016)]{liu2016stein}
Qiang Liu and Dilin Wang.
\newblock Stein variational gradient descent: A general purpose {B}ayesian
  inference algorithm.
\newblock \emph{Advances in neural information processing systems}, 29, 2016.

\bibitem[Liu et~al.(2021)Liu, Jacobs, Li, Nurbekyan, and
  Osher]{liu2021computational}
Siting Liu, Matthew Jacobs, Wuchen Li, Levon Nurbekyan, and Stanley~J Osher.
\newblock Computational methods for first-order nonlocal mean field games with
  applications.
\newblock \emph{SIAM Journal on Numerical Analysis}, 59\penalty0 (5):\penalty0
  2639--2668, 2021.

\bibitem[Liutkus et~al.(2019)Liutkus, Simsekli, Majewski, Durmus, and
  St{\"o}ter]{liutkus2019sliced}
Antoine Liutkus, Umut Simsekli, Szymon Majewski, Alain Durmus, and
  Fabian-Robert St{\"o}ter.
\newblock Sliced-{W}asserstein flows: Nonparametric generative modeling via
  optimal transport and diffusions.
\newblock In \emph{International Conference on Machine Learning}, pages
  4104--4113. PMLR, 2019.

\bibitem[Lu et~al.(2021)Lu, Meng, Mao, and Karniadakis]{lu2021deepxde}
Lu~Lu, Xuhui Meng, Zhiping Mao, and George~Em Karniadakis.
\newblock {DeepXDE}: A deep learning library for solving differential
  equations.
\newblock \emph{SIAM Review}, 63\penalty0 (1):\penalty0 208--228, 2021.
\newblock \doi{10.1137/19M1274067}.

\bibitem[Maoutsa et~al.(2020)Maoutsa, Reich, and Opper]{maoutsa2020interacting}
Dimitra Maoutsa, Sebastian Reich, and Manfred Opper.
\newblock Interacting particle solutions of {F}okker--{P}lanck equations
  through gradient--log--density estimation.
\newblock \emph{Entropy}, 22\penalty0 (8):\penalty0 802, 2020.

\bibitem[Mielke and Stefanelli(2011)]{mielke2011weighted}
Alexander Mielke and Ulisse Stefanelli.
\newblock Weighted energy-dissipation functionals for gradient flows.
\newblock \emph{ESAIM: Control, Optimisation and Calculus of Variations},
  17\penalty0 (1):\penalty0 52--85, 2011.

\bibitem[Mroueh et~al.(2019)Mroueh, Sercu, and Raj]{mroueh2019sobolev}
Youssef Mroueh, Tom Sercu, and Anant Raj.
\newblock Sobolev descent.
\newblock In \emph{The 22nd International Conference on Artificial Intelligence
  and Statistics}, pages 2976--2985. PMLR, 2019.

\bibitem[Munson et~al.(2013)Munson, Okiishi, Huebsch, and
  Rothmayer]{munson2013fluid}
Bruce~Roy Munson, Theodore~Hisao Okiishi, Wade~W Huebsch, and Alric~P
  Rothmayer.
\newblock \emph{Fluid mechanics}.
\newblock Wiley Singapore, 2013.

\bibitem[No{\'e} et~al.(2019)No{\'e}, Olsson, K{\"o}hler, and
  Wu]{noe2019boltzmann}
Frank No{\'e}, Simon Olsson, Jonas K{\"o}hler, and Hao Wu.
\newblock Boltzmann generators: Sampling equilibrium states of many-body
  systems with deep learning.
\newblock \emph{Science}, 365\penalty0 (6457):\penalty0 eaaw1147, 2019.

\bibitem[N{\"u}sken(2023)]{nusken2023geometry}
N~N{\"u}sken.
\newblock On the geometry of {S}tein variational gradient descent.
\newblock \emph{Journal of Machine Learning Research}, 24:\penalty0 1--39,
  2023.

\bibitem[Onken et~al.(2021)Onken, Fung, Li, and Ruthotto]{onken2021ot}
Derek Onken, Samy~Wu Fung, Xingjian Li, and Lars Ruthotto.
\newblock {OT}-flow: Fast and accurate continuous normalizing flows via optimal
  transport.
\newblock In \emph{Proceedings of the AAAI Conference on Artificial
  Intelligence}, volume~35, pages 9223--9232, 2021.

\bibitem[Osher and Fedkiw(2001)]{osher2001level}
Stanley Osher and Ronald~P Fedkiw.
\newblock Level set methods: an overview and some recent results.
\newblock \emph{Journal of Computational physics}, 169\penalty0 (2):\penalty0
  463--502, 2001.

\bibitem[Ott et~al.(2020)Ott, Katiyar, Hennig, and Tiemann]{ott2020neural}
Katharina Ott, Prateek Katiyar, Philipp Hennig, and Michael Tiemann.
\newblock When are neural {ODE} solutions proper {ODEs}?
\newblock \emph{arXiv preprint arXiv:2007.15386}, 2020.

\bibitem[Papamakarios et~al.(2021)Papamakarios, Nalisnick, Rezende, Mohamed,
  and Lakshminarayanan]{papamakarios2021normalizing}
George Papamakarios, Eric Nalisnick, Danilo~Jimenez Rezende, Shakir Mohamed,
  and Balaji Lakshminarayanan.
\newblock Normalizing flows for probabilistic modeling and inference.
\newblock \emph{The Journal of Machine Learning Research}, 22\penalty0
  (1):\penalty0 2617--2680, 2021.

\bibitem[Pavon et~al.(2021)Pavon, Trigila, and Tabak]{pavon2021data}
Michele Pavon, Giulio Trigila, and Esteban~G Tabak.
\newblock The data-driven {S}chr{\"o}dinger bridge.
\newblock \emph{Communications on Pure and Applied Mathematics}, 74\penalty0
  (7):\penalty0 1545--1573, 2021.

\bibitem[Raissi et~al.(2019)Raissi, Perdikaris, and
  Karniadakis]{raissi2019physics}
Maziar Raissi, Paris Perdikaris, and George~E Karniadakis.
\newblock Physics-informed neural networks: A deep learning framework for
  solving forward and inverse problems involving nonlinear partial differential
  equations.
\newblock \emph{Journal of Computational physics}, 378:\penalty0 686--707,
  2019.

\bibitem[Rezende and Mohamed(2015)]{rezende2015variational}
Danilo Rezende and Shakir Mohamed.
\newblock Variational inference with normalizing flows.
\newblock In \emph{International conference on machine learning}, pages
  1530--1538. PMLR, 2015.

\bibitem[Rossi et~al.(2011)Rossi, Savar{\'e}, Segatti, and
  Stefanelli]{rossi2011variational}
Riccarda Rossi, Giuseppe Savar{\'e}, Antonio Segatti, and Ulisse Stefanelli.
\newblock A variational principle for gradient flows in metric spaces.
\newblock \emph{Comptes Rendus Mathematique}, 349\penalty0 (23-24):\penalty0
  1225--1228, 2011.

\bibitem[Rossi et~al.(2019)Rossi, Savar{\'e}, Segatti, and
  Stefanelli]{rossi2019weighted}
Riccarda Rossi, Giuseppe Savar{\'e}, Antonio Segatti, and Ulisse Stefanelli.
\newblock Weighted energy-dissipation principle for gradient flows in metric
  spaces.
\newblock \emph{Journal de Math{\'e}matiques Pures et Appliqu{\'e}es},
  127:\penalty0 1--66, 2019.

\bibitem[Ruthotto et~al.(2020)Ruthotto, Osher, Li, Nurbekyan, and
  Fung]{ruthotto2020machine}
Lars Ruthotto, Stanley~J Osher, Wuchen Li, Levon Nurbekyan, and Samy~Wu Fung.
\newblock A machine learning framework for solving high-dimensional mean field
  game and mean field control problems.
\newblock \emph{Proceedings of the National Academy of Sciences}, 117\penalty0
  (17):\penalty0 9183--9193, 2020.

\bibitem[Santambrogio(2015)]{santambrogio2015optimal}
Filippo Santambrogio.
\newblock \emph{Optimal transport for applied mathematicians}.
\newblock Springer, 2015.

\bibitem[Santambrogio(2017)]{santambrogio2017euclidean}
Filippo Santambrogio.
\newblock $\{$Euclidean, metric, and Wasserstein$\}$ gradient flows: an
  overview.
\newblock \emph{Bulletin of Mathematical Sciences}, 7:\penalty0 87--154, 2017.

\bibitem[Segatti(2013)]{segatti2013variational}
Antonio Segatti.
\newblock A variational approach to gradient flows in metric spaces.
\newblock \emph{Bollettino dell'Unione Matematica Italiana}, 6\penalty0
  (3):\penalty0 765--780, 2013.

\bibitem[Sethian(1999)]{sethian1999level}
James~Albert Sethian.
\newblock \emph{Level set methods and fast marching methods: evolving
  interfaces in computational geometry, fluid mechanics, computer vision, and
  materials science}, volume~3.
\newblock Cambridge university press, 1999.

\bibitem[Shi et~al.(2023)Shi, De~Bortoli, Campbell, and
  Doucet]{shi2023diffusion}
Yuyang Shi, Valentin De~Bortoli, Andrew Campbell, and Arnaud Doucet.
\newblock Diffusion {S}chr\"odinger bridge matching.
\newblock \emph{arXiv preprint arXiv:2303.16852}, 2023.

\bibitem[Song and Ermon(2019)]{song2019generative}
Yang Song and Stefano Ermon.
\newblock Generative modeling by estimating gradients of the data distribution.
\newblock \emph{Advances in neural information processing systems}, 32, 2019.

\bibitem[Song and Kingma(2021)]{song2021train}
Yang Song and Diederik~P Kingma.
\newblock How to train your energy-based models.
\newblock \emph{arXiv preprint arXiv:2101.03288}, 2021.

\bibitem[Song et~al.(2020)Song, Garg, Shi, and Ermon]{song2020sliced}
Yang Song, Sahaj Garg, Jiaxin Shi, and Stefano Ermon.
\newblock Sliced score matching: A scalable approach to density and score
  estimation.
\newblock In \emph{Uncertainty in Artificial Intelligence}, pages 574--584.
  PMLR, 2020.

\bibitem[Song et~al.(2021{\natexlab{a}})Song, Durkan, Murray, and
  Ermon]{song2021maximum}
Yang Song, Conor Durkan, Iain Murray, and Stefano Ermon.
\newblock Maximum likelihood training of score-based diffusion models.
\newblock \emph{Advances in Neural Information Processing Systems},
  34:\penalty0 1415--1428, 2021{\natexlab{a}}.

\bibitem[Song et~al.(2021{\natexlab{b}})Song, Sohl-Dickstein, Kingma, Kumar,
  Ermon, and Poole]{song2021score}
Yang Song, Jascha Sohl-Dickstein, Diederik~P Kingma, Abhishek Kumar, Stefano
  Ermon, and Ben Poole.
\newblock Score-based generative modeling through stochastic differential
  equations.
\newblock In \emph{International Conference on Learning Representations},
  2021{\natexlab{b}}.

\bibitem[Strauss(2007)]{strauss2007partial}
Walter~A Strauss.
\newblock \emph{Partial differential equations: An introduction}.
\newblock John Wiley \& Sons, 2007.

\bibitem[Tabak and Turner(2013)]{tabak2013family}
Esteban~G Tabak and Cristina~V Turner.
\newblock A family of nonparametric density estimation algorithms.
\newblock \emph{Communications on Pure and Applied Mathematics}, 66\penalty0
  (2):\penalty0 145--164, 2013.

\bibitem[Tabak and Vanden-Eijnden(2010)]{tabak2010density}
Esteban~G Tabak and Eric Vanden-Eijnden.
\newblock Density estimation by dual ascent of the log-likelihood.
\newblock \emph{Communications in Mathematical Sciences}, 8\penalty0
  (1):\penalty0 217--233, 2010.

\bibitem[Vahdat et~al.(2021)Vahdat, Kreis, and Kautz]{vahdat2021score}
Arash Vahdat, Karsten Kreis, and Jan Kautz.
\newblock Score-based generative modeling in latent space.
\newblock \emph{Advances in Neural Information Processing Systems},
  34:\penalty0 11287--11302, 2021.

\bibitem[Verine et~al.(2021)Verine, Chevaleyre, Rossi,
  et~al.]{verineexpressivity}
Alexandre Verine, Yann Chevaleyre, Fabrice Rossi, et~al.
\newblock On the expressivity of bi-{L}ipschitz normalizing flows.
\newblock In \emph{ICML Workshop on Invertible Neural Networks, Normalizing
  Flows, and Explicit Likelihood Models}, 2021.

\bibitem[Villani et~al.(2009)]{villani2009optimal}
C{\'e}dric Villani et~al.
\newblock \emph{Optimal transport: old and new}, volume 338.
\newblock Springer, 2009.

\bibitem[Vincent(2011)]{vincent2011connection}
Pascal Vincent.
\newblock A connection between score matching and denoising autoencoders.
\newblock \emph{Neural computation}, 23\penalty0 (7):\penalty0 1661--1674,
  2011.

\bibitem[Wang et~al.(2021{\natexlab{a}})Wang, Jiao, Xu, Wang, and
  Yang]{wang2021deep}
Gefei Wang, Yuling Jiao, Qian Xu, Yang Wang, and Can Yang.
\newblock Deep generative learning via {S}chr{\"o}dinger bridge.
\newblock In \emph{International Conference on Machine Learning}, pages
  10794--10804. PMLR, 2021{\natexlab{a}}.

\bibitem[Wang and Li(2022)]{wang2022accelerated}
Yifei Wang and Wuchen Li.
\newblock Accelerated information gradient flow.
\newblock \emph{Journal of Scientific Computing}, 90:\penalty0 1--47, 2022.

\bibitem[Wang et~al.(2021{\natexlab{b}})Wang, Chen, Liu, and
  Kang]{wang2021particle}
Yiwei Wang, Jiuhai Chen, Chun Liu, and Lulu Kang.
\newblock Particle-based energetic variational inference.
\newblock \emph{Statistics and Computing}, 31:\penalty0 1--17,
  2021{\natexlab{b}}.

\bibitem[Wu et~al.(2020)Wu, K{\"o}hler, and No{\'e}]{wu2020stochastic}
Hao Wu, Jonas K{\"o}hler, and Frank No{\'e}.
\newblock Stochastic normalizing flows.
\newblock \emph{Advances in Neural Information Processing Systems},
  33:\penalty0 5933--5944, 2020.

\bibitem[Yang and Karniadakis(2020)]{yang2020potential}
Liu Yang and George~Em Karniadakis.
\newblock Potential flow generator with ${L_2}$ optimal transport regularity
  for generative models.
\newblock \emph{IEEE Transactions on Neural Networks and Learning Systems},
  33\penalty0 (2):\penalty0 528--538, 2020.

\bibitem[Zhang and Chen(2021)]{zhang2021diffusion}
Qinsheng Zhang and Yongxin Chen.
\newblock Diffusion normalizing flow.
\newblock \emph{Advances in Neural Information Processing Systems},
  34:\penalty0 16280--16291, 2021.

\end{thebibliography}

\end{document}